\documentclass{article}
\usepackage{graphicx} 

\usepackage{authblk}
\usepackage[left=1in, right=1in, top=1in, bottom=1in]{geometry}

\usepackage{amsthm}
\usepackage{amssymb}
\usepackage{amsmath}
\usepackage{mathtools}
\usepackage{csquotes}
\usepackage{enumitem}
\usepackage{hyperref}
\usepackage{xcolor}
\usepackage{fancyhdr}
\usepackage{mdframed}
\usepackage{framed}

\setlist[itemize]{leftmargin=*}
\setlist{nosep}
\setlist{nolistsep}

\hypersetup{
    colorlinks,
    linkcolor={red!50!black},
    citecolor={blue!100!black},
    urlcolor={black!80!black}
}

\renewcommand{\paragraph}[1]{\vspace{.5 cm} \noindent \textbf{#1} }

\usepackage{bbm}

\usepackage{algorithm}
\usepackage[noend]{algpseudocode} 

\usepackage[capitalise,nameinlink,noabbrev]{cleveref}
\crefname{equation}{}{} 
\crefname{enumi}{Step}{} 

\makeatletter
\renewcommand{\theHALG@line}{\thealgorithm.\arabic{ALG@line}}
\makeatother

\theoremstyle{definition}
\newtheorem{theorem}{Theorem}[section]
\newtheorem{lemma}[theorem]{Lemma}
\newtheorem{proposition}[theorem]{Proposition}
\newtheorem{corollary}[theorem]{Corollary}

\usepackage{thmtools}
\usepackage{thm-restate}
\usepackage[round]{natbib}
\usepackage{booktabs}
\usepackage{multirow}
\usepackage{subcaption}
\setcitestyle{maxnames=10}

\crefname{protocol}{Protocol}{Protocols}
\Crefname{protocol}{Protocol}{Protocols}

\DeclareMathOperator{\E}{\mathbb{E}}

\newcommand{\interior}[1]{ {\kern0pt#1}^{\mathrm{o}} }

\newcommand{\eps}{\varepsilon}

\newcommand{\R}{\mathbb{R}}

\newcommand{\Z}{\mathbb{Z}}

\newcommand{\Cbiso}{\mathbb{C}_{\mathtt{BISO}}}
\newcommand{\Ciso}{\mathbb{C}_{\mathtt{ISO}}}
\newcommand{\Cint}{\mathbb{C}_{\mathtt{INT}}}
\newcommand{\Fiso}{\mathcal{F}_{\mathtt{ISO}}}

\newcommand{\Cgen}{\mathbb{C}_{\mathtt{GENLIN}}}
\newcommand{\Cgenint}{\mathbb{C}_{\mathtt{GENINT}}}

\newcommand{\Ex}{B}
\newcommand{\tensor}{T}
\newcommand{\numcriteria}{d}
\newcommand{\numsamp}{N}
\newcommand{\genlinf}{F}
\newcommand{\Tst}{T^*}
\newcommand{\Mst}{M}

\newcommand{\mat}{\widehat{M}}
\newcommand{\ktdist}{K}
\newcommand{\ktcomp}{\overline{K}}
\newcommand{\perm}{\pi}
\newcommand{\numvals}{m}
\newcommand{\subindex}{\ell}
\newcommand{\const}{C}
\newcommand{\indexset}{I}
\newcommand{\sqloss}{R_{\mathtt{OLS}}}
\newcommand{\numindex}{L}
\newcommand{\hatlin}{\widehat{g}_{\mathtt{LIN}}}
\newcommand{\scoref}{f}

\newcommand{\argmin}{\mathrm{argmin}}
\newcommand{\hatf}{\widehat{f}}

\newcommand{\fcv}{f_{\mathtt{CV}}}
\newcommand{\flin}{f_{\mathtt{LIN}}}

\newcommand{\Flin}{\mathcal{F}_{\mathtt{LIN}}}
\newcommand{\yval}{y}
\newcommand{\funspace}{\mathcal{F}}
\newcommand{\Fgen}{\mathcal{F}_{\mathtt{GAM}}}
\newcommand{\Fint}{\mathcal{F}_{\mathtt{GENINT}}}
\newcommand{\calX}{\mathcal{X}}
\newcommand{\calY}{\mathcal{Y}}
\newcommand{\flam}{\widehat{f}_{\lambda}}
\newcommand{\hatlam}{\widehat{f}_{\lambda}}
\newcommand{\minval}{p_{\mathtt{min}}}

\usepackage{todonotes}
\newcommand{\xval}{x}
\newcommand{\ccalX}{\widetilde{\cal{X}}}
\newcommand{\cFiso}{\widetilde{\Fiso}}
\newcommand{\cP}{\widetilde{p}}
\newcommand{\cscoref}{\widetilde{f}}
\newcommand{\cxval}{\widetilde{x}}
\newcommand{\risklam}{R_{\mathtt{VAL}}(\lambda)}
\newcommand{\RLS}{\mathtt{RLS}}
\newcommand{\zval}{z}

\usepackage{float}
\usepackage{caption}

\newfloat{protocol}{thp}{lop}[section]
\floatname{protocol}{Protocol}

\newcommand{\testset}{S_2}

\newcommand{\dataset}{S'}

\hypersetup{
    linkcolor={red!50!black},
    citecolor={blue!100!black},
    urlcolor={blue!80!black}
}
\title{Learning What Evaluators Value:~\\ A Reliable Approach to Modeling Evaluator Preferences}
\author{Madeline Kitch}
\author{Nihar B. Shah}
\affil{Carnegie Mellon University\\
\texttt{\{mckitch,nihars\}@cs.cmu.edu}}

\date{}

\begin{document}
\maketitle
\begin{abstract}
  In many applications, human and LLM evaluators use assessments of relevant criteria to create an overall evaluation for an item or individual. For example, in admissions, committees assess candidates on attributes such as test scores, GPA, and research experience to evaluate their overall fit for the program. Another example arises in medical care where clinicians use patient reports of symptoms to consider preliminary diagnoses and assess risks. Each setting involves mapping multiple criteria to an overall evaluation---a process that reflects the evaluator's underlying preferences. We focus on the fundamental question of learning these preferences. 
  
  Many applications of this problem make specific modeling assumptions on evaluator preferences that may be substantially violated in the real world. We make the minimal assumption that the preference function is coordinate-wise non-decreasing, which is reasonable in a large number of evaluation settings. We theoretically characterize the severity of model mismatch for many common assumptions and show that it can lead to significant issues for learning evaluator preferences and other important downstream tasks. We then present an algorithm for learning evaluators' preferences that is robust to model mismatch. We prove theoretically that our algorithm can learn any preference function without sacrificing performance when the linearity assumption holds. Evaluations of our algorithm with synthetic simulations and real-world data confirm its ability to learn preferences robustly and illustrate key aspects of LLM and human preferences. To conclude, we present a case study of how to use our method to provide insights into reviewer preferences and increase fairness in the peer review process. 
\end{abstract}

\section{Introduction}\label{sec:intro}

In many areas of life, we use evaluations to help guide decision-making. Doctors evaluate patients in order to recommend the best course of treatment; (LLM) reviewers evaluate academic papers to determine areas of improvement and conference acceptance; everyday individuals evaluate products and hotels that help inform others on what to buy and where to stay.

Often, to make their recommendations or assessments more clear and methodical, evaluators provide additional evaluations of relevant criteria or aspects alongside their overall rating or recommendation. 
In the medical setting, doctors typically record patients' clinical signs and symptoms. This information then guides medical diagnoses and the clinician's recommended course of treatment. In academic peer review, reviewers may be asked to provide ratings for aspects of the paper, such as soundness and contribution, prior to providing their overall score. Product ratings (e.g., of hotels) may include how the individual felt about specific aspects (e.g., rooms or location) in addition to a rating of their overall experience. \Cref{fig:tripadvisor_iclr} provides real-world examples of the criteria and overall assessment setting we study for the latter two applications: product ratings and peer review.  

\begin{figure}[tbp]\label{fig:Tripadvisor}
    \centering
    \subcaptionbox{Example Tripadvisor reviews\label{fig:tripadvisor}}{%
        \includegraphics[width=\linewidth]{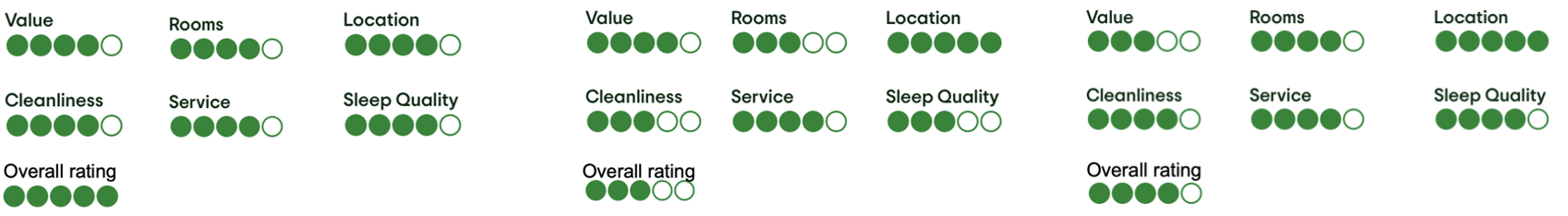}
        }%
    \vfill
    \subcaptionbox{Example ICLR 2024 reviews\label{fig:iclr}}{%
        \includegraphics[width=.6\linewidth]{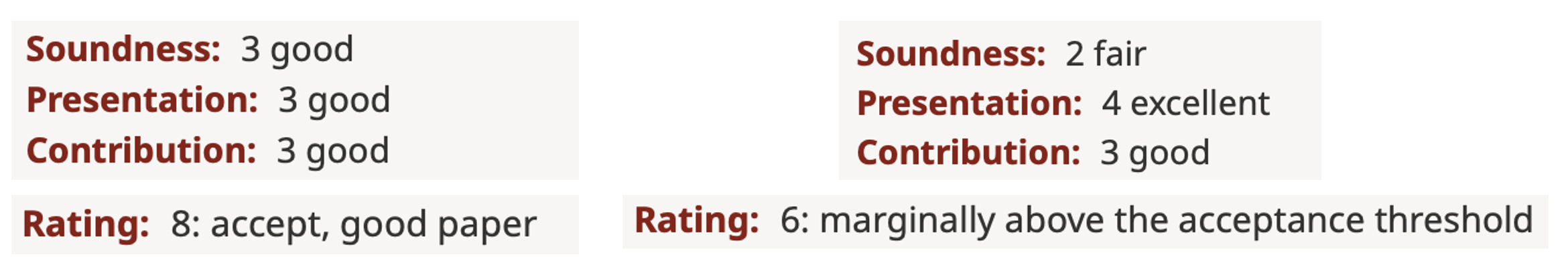}
        }
    \caption{An illustration of the evaluation settings we consider. Evaluators provide criteria evaluations alongside an overall score or recommendation. Our focus is on learning how evaluators map assessments of criteria (e.g., cleanliness or soundness) to their overall scores (e.g., a stay rated 4 out of 5 stars, or a review says 8: accept). 
    }
    \label{fig:tripadvisor_iclr}
\end{figure}

Such criteria or aspect evaluations serve many purposes. They illustrate what factors went into the evaluator's decision-making: why a doctor recommended a certain medication, or why a paper achieved a high score. Hence, they can help to increase the transparency of the evaluation process. Moreover, criteria evaluations ensure that the assessors consider multiple dimensions of the item or individual prior to reaching their conclusion. Consequently, it is harder for an evaluator to simply go with their gut and miss a potentially relevant factor. 

\Cref{fig:setting} illustrates the problem and general evaluation setting we consider. In this paper, we are interested in the fundamental question: \textit{How do evaluators map criteria evaluations to their overall scores or recommendations?}

We call this mapping the evaluators' \textit{preference function} (or collective/community preference function, to emphasize that evaluators often bring distinct individual preferences). \textbf{Learning this mapping is useful in several ways}:~\\

\begin{figure}[t]
    \centering
        \includegraphics[width=0.6\linewidth]{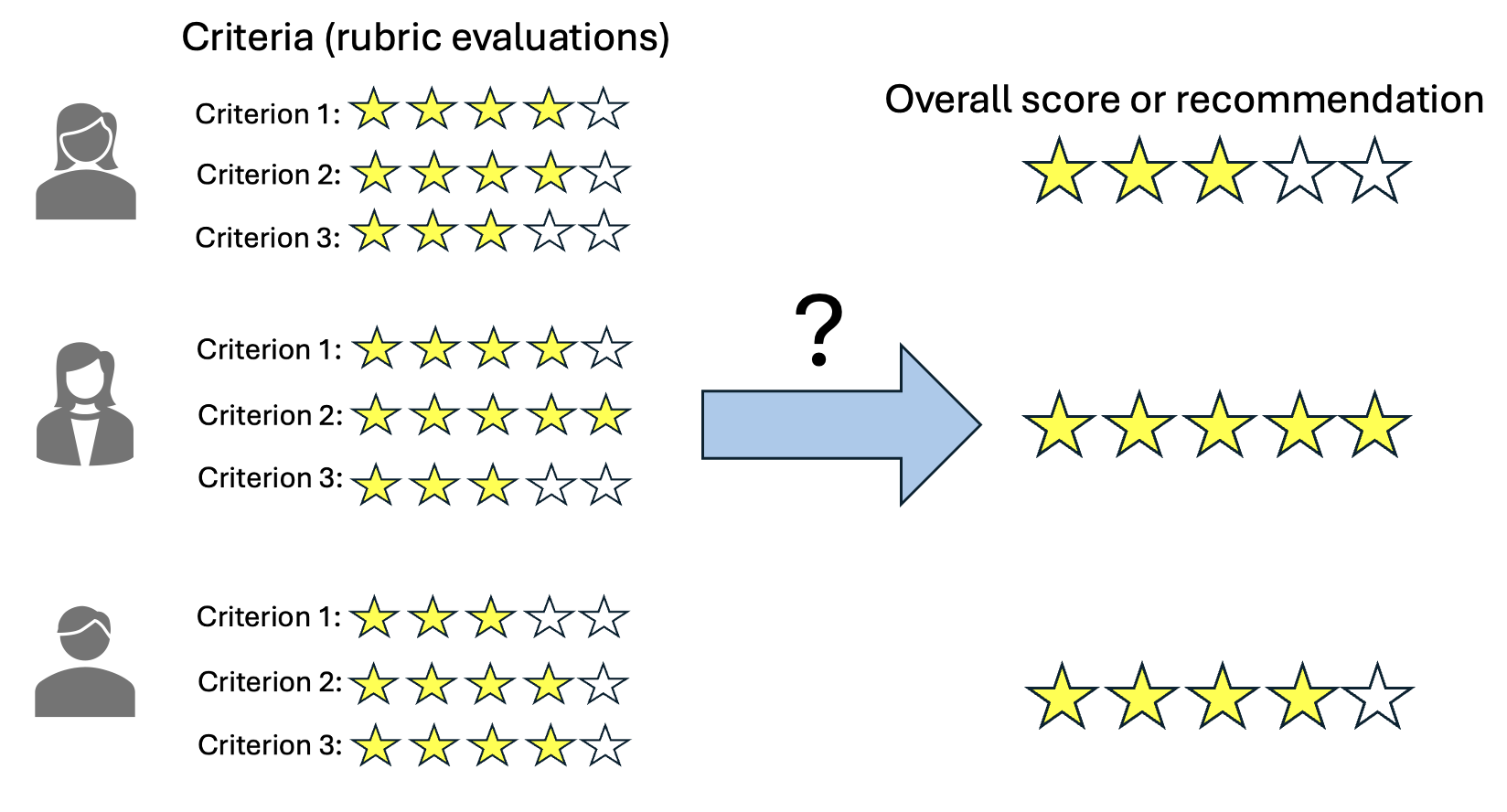}
    \caption{An illustration of the problem we address in this paper. There are many evaluators, each of whom provides assessments of multiple criteria or aspects (on the left) and an overall evaluation (on the right). The question we are interested in is denoted by the blue arrow: how evaluators map multiple criteria assessments to their overall scores or recommendations.
    }
    \label{fig:setting}
\end{figure}

\begin{itemize}[itemsep=0.5em]
    \item 
    
    Practitioners and researchers can use the preference function to understand what evaluators care about. Say, for instance, someone wants to understand the relative value of a contribution in a paper (rated from 1-4). A learned preference function can illustrate how improving the rating of contribution (all else constant) alters the overall recommendation of the paper, and how this depends on the other criteria evaluations. If evaluators believe novel ideas are most useful when they are well-motivated and presented, an increase in contribution may be more beneficial if the paper also has a high score for soundness. Further, if evaluators place special emphasis on particularly novel ideas, an increase from a rating of 1-2 may be less impactful than one from 3-4. 

    \item The preference function can reduce discrepancies arising from idiosyncratic evaluation preferences. Often, there are multiple evaluators with different preferences. Since every evaluator applies their own unique preferences, overall assessments depend on the specific preferences of the evaluator to whom an item or individual is assigned. To alleviate such discrepancies, we can learn the preference function representing the collective preferences of evaluators. This learned function can flag assessments where the mapping of criteria to overall scores deviates significantly from the norm. Then, practitioners can assess the flagged evaluation and determine whether the subsequent recommendation or assessment should be altered to align more with collective practice. 

    \item We can use evaluator preferences to increase alignment in rule-based evaluations. There are many situations where rule-based evaluations or algorithms replace human judgment. This can be for multiple reasons, including the fact that evaluations can be costly in terms of time and financial resources. One example is preliminary diagnostic tests of anxiety and depression (PHQ-9 and GAD-7). In these settings, if we apply the learned evaluation function instead of the current algorithm (e.g., the predicted psychiatric diagnosis given patient responses), this could allow for the evaluation to more closely align with evaluator beliefs and/or expertise.~\\
\end{itemize}

For these and other reasons, many researchers and practitioners have tried to learn evaluator preferences. However, they typically assume the preference function is linear, or some extension thereof (e.g., log-linear or linear with interactions). Such simplifying assumptions have benefits. Commonly used models are often more interpretable. If the preference function is captured by a basic linear model, the overall score is a weighted combination of individual criteria. Hence, we can directly look at the weights to infer the relative importance of each criterion. Additionally, stronger assumptions on preferences help to reduce the amount of information needed to learn them. Since we typically use a dataset of many criteria evaluations (inputs) alongside overall scores (outputs), this implies that we need fewer total evaluations (criteria and overall score pairs) to reliably learn evaluator preferences. Yet these benefits come with a significant downside: any parametric model implies certain assumptions about people's preferences that may or may not make sense. What's more, people often fail to check the necessary model assumptions in practice. 
\citet{jones_common_2025}, a recent study in a biomedical journal underscores this point. They surveyed 83 papers where researchers assumed a linear model and needed to check the linearity assumption. Only 15 of those papers (less than 20\%) did so directly. These and other findings led them to conclude:
\begin{center}
\textit{\enquote{[T]he reporting of linear regression assumptions is alarmingly low.}}
\end{center}
The implication of \citet{jones_common_2025} is that researchers are potentially using incorrect models to capture evaluator preferences without even realizing it. Furthermore, others have begun to critique the use of linear or other simple parametric functions on philosophical grounds. In a qualitative study on learning human preferences regarding kidney allocation prioritization, \citet{keswani_can_2025} found that half (10 of 20) participants employed thresholds and/or other non-linear transformations, illustrating the fact that linear models cannot capture the essential components of people's moral preferences. Their finding on model class misalignment is significant because if we then use these incorrect models for decision-making (e.g., for allocating kidneys or conducting peer review), this can lead to real-world harms.

Our paper is motivated by the importance of learning evaluator preferences and by the potential issues arising from assuming an incorrect functional form. Thus, in this paper, we make a minimal non-parametric assumption on preferences, and use it to study the issues of and solutions to model mismatch. 

We detail our approach next, followed by a summary of our contributions.~\\

\noindent\textbf{Our approach.} In general, any parametric model contains implicit assumptions about people's preferences that may or may not make sense. In order to avoid imposing such assumptions on evaluators, we adopt a non-parametric model, assuming only that the preference function is isotonic, that is, non-decreasing in each criterion. Isotonic functions are those that are coordinate-wise non-decreasing. The choice of isotonic functions is particularly well-suited for evaluation problems. In many settings, while it may not be clear the magnitude of a change in criteria scores, it is reasonable to assume the direction. For instance, all else constant, a paper with a greater degree of novelty should not be less likely to be accepted, and a psychiatric patient presenting with more severe symptoms should not be deemed lower risk. While the definition of an isotonic function is quite succinct, it is important to realize how expressive this class is. Isotonic functions are able to capture any form of interactions between variables, thresholds, and nonlinear transformations of the output. 

Given this setup, we investigate three types of questions: First, we ask whether many of the common assumptions are problematic when people use them to learn more complex monotonic functions. We ask if there are realistic preference functions that linear models, generalized additive models, or generalized additive models with interactions fundamentally cannot capture. Additionally, we investigate if this leads to issues in other common downstream applications: determining the relative importance of different evaluation criteria, and ranking items or individuals based on their criteria evaluations. Answering these questions helps us to understand the extent to which model mismatch could be an issue. Second, given our understanding of the model mismatch problem, we then ask if there is anything we can do to help address it. In particular, we investigate learning algorithms that could allow people to learn what any set of evaluators cares about but do so without sacrificing performance when the linear assumption holds. Since the two previous components were theoretical, we then empirically analyzing whether our findings hold on real-world preference data. To conclude, we illustrate how our method can be used in practice to compare LLM and human evaluation preferences and analyze the academic peer review process.~\\

\noindent\textbf{Our contributions:} Here are the key contributions of this paper.~\\

\underline{\textit{Problems with common model assumptions.}} We establish strong negative results for commonly used models when preferences are more complex. Our findings underscore the significant problems that can occur when incorrectly assuming one of these models. 

\begin{itemize}[itemsep=0.5em]
        \item \textit{Inability to learn evaluator preferences}. Our main goal is to learn how the evaluators map criteria to overall scores. We consider the case when a researcher incorrectly assumes a linear model, a generalized additive model, or a generalized additive model with criteria interactions. Given this model mismatch, we investigate how well any estimated preference function can estimate the evaluators' overall scores given the criteria scores. The estimation error is the expected squared difference between the output of the estimated function and the true preference function, given a random set of criteria scores. We prove that there exist monotonic preference functions such that when fit by any linear model with extensions, the resulting estimation error is as bad as randomly guessing outputs up to a constant factor.  
        \item \textit{Issues ranking items based on criteria evaluations}. Evaluation settings frequently call for the ranking of items or individuals. For instance, admissions officers rank potential students, and conferences (loosely) rank papers to determine acceptance. When a researcher mistakenly assumes a generalized additive model, we consider how often the ranking provided by the estimated preference function agrees with the ranking of the true preference function. We prove that there exist monotonic functions for which the ranking error (proportion of comparisons that the estimated ranking is correct) can be on the same order as randomly guessing a ranking between items.  
        \item \textit{(Mis)-understanding the relative importance of criteria}. An additional use case for preference functions is to understand the relative importance evaluators give to different criteria. When using a simple linear model, one can answer this question by comparing the values of the regression coefficients. The criterion with the largest corresponding coefficient is assumed to carry the greatest ``weight'' in decision-making. We prove that when preferences are more complex, and there is an imbalance of samples, researchers can be led to the opposite conclusion than what is actually true (e.g., inferring that criterion 1 is more important than criterion 2 when the reverse is true).~\\  
    \end{itemize}

\underline{\textit{Algorithmic solution}}. Given the aforementioned issues with commonly used models, we develop and theoretically analyze an algorithm that can learn evaluator preferences more reliably.  

\begin{itemize}[itemsep=0.5em]
    \item \textit{Our Algorithm}. Our goal in developing an algorithm is to ensure that the evaluators' preferences are learned, whether or not the researcher's assumption holds. The key techniques we use to accomplish these aims are isotonic regression (regression for isotonic functions), regularization, and cross-validation. Regularization creates a penalty for the degree of model mismatch (i.e., distance from a linear model) in addition to a penalty for error on the training set. Our algorithm does regularized isotonic regression---weighing both the fit of the model to preference data and the degree of model mismatch. Finally, we carefully pick the relative weight of both error components (model mismatch and fit) through cross-validation. 
    \item \textit{Theoretical results}. The most natural baseline is non-negative least squares (which we often refer to as \textit{linear regression} for brevity). Non-negative least squares imposes non-negativity on coefficients to ensure monotonicity and performs weakly better than unconstrained regression in practice for isotonic linear models. We prove theoretically that our algorithm satisfies both goals we care about: (1) it is able to learn any evaluator preference function and (2) it is nearly as good as linear regression. For (1), we prove an upper bound on the estimation error of our algorithm for any monotonic preference function that vanishes in the limit. Regarding (2), we prove that the estimation error of our algorithm is upper-bounded by the error of linear regression. Importantly, while our results are specific to linear regression, the key theoretical techniques we rely on extend to proving robustness for any other model assumption.~\\
\end{itemize}
    
\underline{\textit{Empirical findings}}. Lastly, we use real-world preference data to measure the performance of our algorithm and illustrate the insights it can provide for evaluation systems, especially in peer review.

\begin{itemize}[itemsep=0.5em]
\item \textit{Synthetic simulations}. We compare our algorithm to linear regression on synthetic data from common utility models used in economics---linear, Leontief, and Cobb-Douglas. When preferences are linear, both algorithms have comparable estimation error. And, for non-linear preferences (Leontief and Cobb-Douglas), our algorithm substantially outperforms linear regression. Furthermore, we compare our algorithm with two monotone-constrained machine learning models (gradient-boosted tree ensembles and neural networks) and find that it has similar or lower estimation error across the different sample sizes and preference types.

\item \textit{Real-world human preferences}. To see how our algorithm fares given real preference data, we apply our algorithm to preference data from Tripadvisor hotel ratings. We then measure the prediction error and the reducible error (a proxy for estimation error) of the fitted preference function. Our algorithm has lower prediction and reducible error than linear regression, a result that is statistically significant across all the datasets we test. The difference in reducible error is most pronounced. There is as much as a 69\% decrease when compared with linear regression. That said, this benefit is modest when compared to the level of noise in the dataset. We use this to conclude that the relative benefit to practitioners is likely to be larger when estimating the underlying preference function than predicting criteria scores if there is a significant amount of noise or preference variability. To address this potential issue, we provide methods for testing the degree of noise in the data.
\item \textit{Real-world preferences of LLMs}. We provide a case study demonstrating how our method can be used in practice to compare preferences of human and LLM evaluators. To do so, we use a dataset of human and LLM  (GPT-4o and Llama 3.1 70b) generated scientific reviews from \citet{yu_is_2026}. We test both preference alignment (i.e., how close the estimated mapping from criteria to overall scores of an LLM is to that of human evaluators) and how consistent the preferences of a given LLM are across reviews. On both of these dimensions, GPT significantly outperforms Llama. While our focus is not on this specific result, it illustrates how our methods can be used to measure the alignment of LLM preferences with those of human evaluators, and the internal consistency of LLM preferences.
\item \textit{Case study on peer review}. We conclude by analyzing the way peer reviewers aggregate multiple criteria into overall scores, using ICLR 2026 review data. 
We find that of the three review criteria (`contribution,' `soundness,' and `presentation'), contribution has by far the largest predictive value. That said, none are close to zero suggests that each criterion plays an important role in assessing a paper. We then use our method to measure the commensuration bias of each review, which we estimate as the difference between the provided rating and the rating implied by the reviewer's own criteria scores under the community preference function. Under a simplified model of the peer review process, we estimate that nearly 16\% of paper decisions would be altered in the absence of commensuration bias, indicating that reviewer preference heterogeneity likely has a significant impact on peer review outcomes. The ICLR 2026 reviews ranked in descending order of commensuration bias is available at  \url{https://github.com/madelineceli13/evaluator-preference-algorithm/blob/main/data/case_study/debiased_reviews_2026.csv}. 
\end{itemize}

\section{Related work}\label{sec:related_work}

There are various approaches people currently use when uncertain about their model assumptions. One such method is the Ramsey RESET test, used predominantly in econometrics \citep{ramsey_tests_1969} for testing linearity. The idea is that if the model is truly linear, then higher-order terms of the fitted y-values should have no explanatory power when added to the regression. If one finds that they do, this implies that the model is mis-specified. While useful, one of the drawbacks of this test is that it does not inform the researcher of a more suitable functional form, merely that the null hypothesis of a linear model is insufficient. Another method that people use is to plot potential interaction or higher order terms against regression residuals to see if there is a meaningful relationship for both linear models \citep{cook_exploring_1993} and generalized additive models \citep{cook_partial_1998}. These methods are extended to more complex regression settings in \citet{fox_visualizing_2018}. However, they require the researcher to conduct an ad-hoc iterative approach, using the results of these tests to update their models and then test model fit again. Assuming only monotonicity of the regression function, we propose an alternative approach that avoids model mis-specification altogether. By using a regression method that accommodates more complex models directly, we completely avoid the problem of iterative model revision. 

As we are interested in learning isotonic (preference) functions, our work relates to the literature on isotonic regression. When using a least squares estimator, a series of works have proved theoretical upper bounds on isotonic regression in the case of one variable \citep{chatterjee_risk_2014}, two variables \citep{chatterjee_matrix_2015}, or in general dimensions \citep{fang_multivariate_2020, han_isotonic_2019}. We complement this by proving bias results, demonstrating that a broad class of common model assumptions can lead to worst-case minimax error. Our results in \Cref{sec:problems_with_linearity} for estimating monotonic functions $\numcriteria \ge 2$ dimensions, assuming a linear, generalized additive model, or generalized additive model with interactions, can result in error as bad as randomly guessing outputs (up to constant factors). Additionally, one of the common threads in the literature is the natural adaptivity of the least squares estimator when the monotonic function has a simpler structure. When the number of rectangles or hyper-rectangles on which the function is constant is sufficiently small, \citet{chatterjee_matrix_2015} shows that the estimation error converges at a parametric rate in terms of the number of samples up to logarithmic multiplicative factors for $\numcriteria = 2$, and when $\numcriteria\ge 1$, \citet{fang_multivariate_2020} prove the same result for entirely monotonic functions, a strict subset of monotonic functions. Similarly, when the function is constant with respect to one input, \citet{chatterjee_matrix_2015} shows that the risk of a 2-dimensional isotonic regression converges (up to logarithmic factors) at the same rate as the 1-dimensional function. While these results are promising, \citet{han_isotonic_2019} proves isotonic regression adapts at a strictly nonparametric rate in higher dimensions. Instead of proving the adaptivity of isotonic regression to linear models (the parametric model of interest), we combine isotonic regression with cross-validation. By applying a result on oracle inequalities from \citet{vaart_oracle_2006}, we are able to show that our estimator achieves the same rate as linear regression (with no additional logarithmic dependence on the number of dimensions) when the true model is linear, while still being able to learn any isotonic function. Our work, therefore, provides an example of the benefit of combining parametric and nonparametric estimators through cross-validation when model assumptions are untestable or unlikely to hold.  

As we develop methods for learning people's preferences, our work complements previous research on learning from human feedback. One of the primary forms of feedback is pairwise comparisons. Given a set of pairwise comparisons among one or many evaluators, there are multiple questions of interest: constructing a consensus ranking given noisy comparisons from many individuals \citep{shah_simple_2018, heckel_active_2016}, inferring the latent qualities of the items or options being compared \citep{wang_stretching_2019, bradley_rank_1952}, or estimating the probabilities of subsequent comparisons \cite{shah_feeling_2016, shah_stochastically_2016}. As in our setting, researchers frequently assume parametric models such as the Bradley-Terry-Luce or Thurstone model. \citet{shah_stochastically_2016} establishes bias results for when these assumptions are incorrectly applied. Additionally, the authors introduce a natural non-parametric assumption---Stochastically Transitive Models---and prove that when assuming this larger model class, one can still recover minimax optimal parametric rates. Each of these results---the issues with current model assumptions and the possibility for more reliable algorithms that do not sacrifice performance---mirrors ours. However, this area of work focuses on using ordinal (comparison) feedback to learn overall qualities; in our setting, we observe cardinal feedback on a structured input space. Additionally, instead of learning about how existing objects or items compare, our goal is to learn a \textit{function}. 

Given the fact that our data often comes from individuals with diverse opinions, another set of relevant work is on preference aggregation in social choice theory. A social choice theoretic perspective ensures certain axiomatic properties (e.g., anonymity or unanimity) are satisfied when combining individual preference data to create an output (i.e., a single ranking or preference function). In a typical voting setting, individuals provide rankings over a set of options or candidates, and some algorithm takes these rankings as input and outputs a consensus ranking. This consensus ranking can denote, for instance, the relative preferences over candidates or policy options. There is a large amount of existing work on developing aggregating rules for voting (see \citep{brandt_handbook_2016} for an overview), alongside impossibility results such as May's theorem \citep{may_set_1952} and Arrow's Impossibility theorem \citep{arrow_social_1951} for the simultaneous satisfying of a set of axioms. 

Since our focus is on learning functions, not rankings, nearest to our setting is the work of \cite{noothigattu_loss_2021} which focuses on learning evaluator preferences from an axiomatic perspective. Given the problem of empirical risk minimization over the set of isotonic functions, they consider the set of $L(p,q)$ loss functions (the matrix extension of $L_p$ loss for vectors) and prove that only one choice of parameters, $p=q=1$, satisfies the three natural axioms they consider. The algorithm they present has been used broadly in several ML/AI conferences as well as admissions processes, thus underscoring the applicability and importance of our setting. In contrast to \citet{noothigattu_loss_2021}, we take a statistical approach to learning the preference function; our aim is to find a function that best fits the evaluation data. In addition to offering a different perspective on this problem, there is a practical reason as well---such axioms may not be necessary in settings where practitioners care most about reliably estimating the value or outcome (i.e., clinicians assessing patient risks, companies predicting product ratings). Purely minimizing the loss function over all isotonic functions can lead to a larger degree of overfitting given a small sample size, causing issues for accurate prediction. To combat this, we apply a different loss function, minimizing a regularized error (accounting for the complexity of the learned model) rather than simply the error on the training data. This fix enables us to provide guarantees even on a small sample size.

Lastly, our work relates to applied literature on commensuration bias, noise in decision-making, and participatory design. Sunstein and Tversky, a legal scholar and psychologist, respectively, demonstrate that, given the exact same information, evaluation decisions can vary significantly due to the idiosyncrasies of unique evaluator preferences (i.e., every reviewer has unique beliefs) and the specific evaluation context (e.g., the time of day it is reviewed). Using judicial decision-making and hiring as examples, they argue that the result of ``noisy'' decision-making is widespread inefficiencies and unfairness \citep{kahneman_noise_2021}. Philosopher Carol Lee writes about this same problem in the context of the scientific review process. She argues that commensuration, the process of ``converting distinct qualities into a single metric,'' can be significantly inconsistent from one evaluation to the next \citep{lee_commensuration_2015}. This is due in large part to evaluator preference inconsistency, a phenomenon she calls \textit{commensuration bias}. Whether under the label of noise or commensuration bias, the result is that the assessment and resulting outcome or recommendation for an item or individual becomes contingent on the particular evaluator they are assigned, violating the philosophical imperative of equal treatment. Our approach can help to alleviate such unfairness: Practitioners can use the community preference function to flag assessments where the criteria-to-overall score mapping deviates from the norm. Taking a further look at such assessments can help ensure that the outcome or recommendations ultimately provided are truly reflective of what the community of evaluators cares about. Finally, within computer science, there has been increasing interest in understanding (mathematically) human preferences so as to incorporate their beliefs in algorithmic design \citep{lee_webuildai_2019, keswani_can_2025, freedman_adapting_2020}. At a high level, the aim is to learn individual preference functions, and then aggregate them via an axiomatic approach \citep{lee_webuildai_2019, noothigattu_voting-based_2018}; most work assumes a linear model \citep{keswani_can_2025, lee_webuildai_2019} or another simple parametric form such as weighted p-means \citep{pardeshi_learning_2024}. This is despite the apparent issues with doing so as described in \citet{keswani_can_2025}. When applied to this literature, we provide a method for learning what individuals care about without missing morally relevant components of their beliefs. 

\section{Model overview}\label{sec:model_overview}
In this section, we will provide a formal overview of the evaluation setting, our learning problem, and the common model assumptions we consider. Recall that evaluators take as input assessments of individual criteria or aspects and output an overall score or assessment (see \Cref{fig:tripadvisor_iclr}). In our ICLR example, the criteria correspond to aspects of the paper (soundness, presentation, and contribution); the overall score corresponds to the rating (e.g., 8 out of 10). Each criterion is given scores corresponding to integer values from 1 (poor) to 4 (excellent). We use $\numcriteria \ge 2$ to denote the number of relevant criteria, and $\numvals\ge 2$ to denote the number of values a score can have (e.g., for ICLR, there are 3 criteria and 4 possible scores per criterion). Let $[\numvals]=\{1,2,\dots, \numvals\}$ be the set of integers between $1$ and $\numvals$. Then the vector of criteria scores, $\xval$, is contained in $\calX = [\numvals]^\numcriteria$. Further, because overall scores commonly take on values in a bounded range (e.g., papers are scored from 1-10), we assume that the preference function maps to some bounded interval. For technical convenience, we rescale the outputs to $[0,1]$. Therefore, the evaluator preference function $\scoref$ takes as input a vector of criteria scores $\xval \in \calX$ and outputs a scalar value $\scoref(\xval) \in [0,1]$.

Recall that we assume only that the preference function is isotonic, or non-decreasing in each criterion. Formally, we define the set of isotonic preferences with range $\calY$, to be
\begin{equation}\label{defn:iso}
    \Fiso(\calY) := \{\scoref: \calX \to \calY \mid \scoref(\xval) 
    \ge \scoref(\xval') \text{ whenever } \xval_i \ge \xval'_i \text{ } 
    \forall i \in [\numcriteria]\},
\end{equation}
and we let $\Fiso := \Fiso([0,1])$. As discussed in \Cref{sec:intro}, in many evaluation settings, any realistic preference function is coordinate-wise non-decreasing; hence, it is reasonable to assume that $\Fiso$ contains the set of \textit{all} possible evaluator preferences. 

\subsection{Learning preference functions}
\textit{Our primary goal is to learn evaluator preferences}. To accomplish this, we use evaluation data consisting of criteria and overall scores, which serve as the function inputs and outputs, respectively. Let there be $\numsamp\ge 2$ total assessments, indexed by $j$, and let $\yval^{(j)}$ be the overall score or recommendation for item $j$. The set of all $\numsamp$ evaluations is represented by a sequence of tuples $\{\xval^{(j)},\yval^{(j)}\}_{j\in [\numsamp]}$ satisfying 
\begin{equation}\label{eq:model}
    \yval^{(j)} =\scoref(\xval^{(j)}) + w^{(j)}
\end{equation}
where $w^{(j)}$ is independent noise, capturing the fact that the preference applied to each sample may be different due to context or differences in evaluator opinions. The $\{\xval^{(j)}\}_{j=1}^{\numsamp}$ are independent and identically distributed from some distribution $P$ with support on $\calX$ and inducing a set of probabilities $P(\xval) \in [0,1]$. Our goal is to learn $\scoref$ as well as possible over the set of typical criteria scores (captured by the distribution $P$). We measure this by how well our estimated function, $\hatf$, predicts the output of $\scoref$, when given a new set of criteria scores $\xval \sim P$. Let $\|g\|^2_{L_2(P)}=\E_{\xval \sim P}\|g(\xval)\|_2^2$ be the expected $L_2$ norm of a function with inputs randomly sampled from $P$. Then we can define our measure, estimation error, as
\begin{equation}\label{eq:expected_risk}
    R(\hatf, \scoref) = \mathbb{E}\|\scoref(\xval) - \hatf\|^2_{L_2(P)}
\end{equation}
where the expectation is taken over the randomness of the sample data. Since $\scoref$ takes values in $[0,1]$, any estimator $\hatf$ constrained to the same range satisfies $R(\hatf,\scoref)\leq 1$.

\subsection{Common model assumptions}
We consider a set of common model assumptions, the simplest of which is the linear model. The set of non-decreasing linear models with range $\calY$ is defined as 
\begin{equation}\label{defn:lin}
    \Flin(\calY) := \left\{\scoref: \calX \to \calY  \;\middle|\; \scoref(\xval) = a^\top \xval + b, \; a \in \R_{\ge 0}^\numcriteria, \; b \in \R \right\}.
\end{equation}
For brevity, we set $\Flin([0,1]) = \Flin$. Researchers often extend linear models to capture specific nonlinearities. Sometimes the marginal benefit of a change in criteria score depends on the current level of that score (i.e., contribution matters only if it exceeds a $2$). To capture these effects, one can apply a non-decreasing transformation, such as a threshold function, to each criterion input score. Additionally, people often apply monotonic transformations to the outputs of the linear function, such as taking the logarithm or squaring the sum. Generalized additive models (GAMs) allow for transformations of the output and transformations of inputs. Let $\Fgen$ be the set of GAMs that are isotonic. Then $\scoref \in \Fgen$ if there exists a non-decreasing transformation of outputs $\genlinf: \R \to \R$, set of non-decreasing transformations of inputs $h_i: [\numvals] \to \R$ for each $i \le \numcriteria$, and coefficients and coefficients $a\in \R^\numcriteria, b \in \R$ such that for all $\xval \in \calX$:
\begin{equation}\label{eq:genlin}
    \scoref(\xval) = \genlinf\left(\sum_{i \in [\numcriteria]}a_ih_i(\xval_i) + b\right) \text{ for all }\xval \in \calX. 
\end{equation}
By construction, $\Fgen$ contains $\Flin$. And, $\Fgen$ contains no constraints on the range of the model, thus making this class even more expressive. 

Our main results apply to a larger class, $\Fint$, which allows for products of criteria; people often include these interaction terms when one criterion impacts the effect of another (i.e., a higher soundness score may increase the value of a paper's contributions). The name \texttt{GENINT} comes from the fact that this class includes all generalized additive models with interactions between criteria. Let $\mathcal{P}([\numcriteria])$ denote the power set of $[\numcriteria]$. Then $\scoref \in \Fint$ if there exists a non-decreasing transformation of outputs $\genlinf: \R \to \R$, set of non-decreasing transformations of inputs $h_i: [\numvals] \to \R$ for each $i$, and coefficients $\{\gamma_{\indexset}\}_{\indexset \in \mathcal{P}([\numcriteria])}$ such that for all $\xval \in \calX$:
\begin{equation}\label{eq:genint_class}
    \scoref(\xval) = \genlinf\!\bigg(\sum_{\indexset \in \mathcal{P}([\numcriteria])} \gamma_\indexset \prod_{i \in I} h_i(\xval_i)\bigg).
\end{equation}
We use the convention $\prod_{i \in \emptyset}(\cdot) = 1$ so that $\gamma_\emptyset$ serves as an intercept. Notably, $\Fint$ contains $\Flin$ and $\Fgen$, weighted $p$-means functions used in social choice \citep{pardeshi_learning_2024}, as well as generalized linear models, and (log-)linear models with interactions.

\section{Problems with assuming linear preferences}\label{sec:problems_with_linearity}
While we may not believe preferences to be linear on philosophical grounds (the inability to capture criteria interactions or thresholds), the more important question is whether this assumption meaningfully impacts outcomes that practitioners care about. If an incorrect assumption has little impact on real-world outcomes, then it is not a problem. However, we find the opposite to be true. In the worst case, incorrectly assuming a linear, generalized additive model, or a generalized additive model with interactions, can have significant consequences. 

Three common uses of preference functions are (1) to understand evaluators by learning the underlying model that reflects what they care about, (2) to rank items/individuals based on their predicted overall assessment, and (3) to understand the relative importance of each criterion. We will look at each of these uses--understanding evaluators, ranking items/individuals, and criteria importance--in turn, explaining common applications where they arise as well as negative theoretical results on the use of linear models. 

\subsection{Estimating preferences}\label{sec:estimation_functions}
One of the main goals of learning evaluator preferences is to understand what evaluators care about. This function can tell us what criteria matter more or less, how they combine, and when changes in criterion scores are more or less significant. Further, given preference functions for two groups of evaluators, we can also learn about how the opinions of each group differ. Additionally, practitioners can use this learned preference function to predict evaluations given a set of criterion scores. This is useful in applications where criterion scores are easy to acquire or deterministic, whereas overall assessments can be costly. For instance, PHQ-9 and GAD-7 are used for preliminary risk assessments in place of a psychiatric diagnosis to save resources and time of medical personnel. In this section, we ask: \textit{How well can commonly used models capture complex evaluator preferences?}

Consider the case when the researcher or practitioner assumes the true preference function is in $\Fint$ (defined in Equation \Cref{eq:genint_class}). Recall this includes the set of linear models, generalized additive models, and generalized additive models with criteria interactions. We care about minimizing the error $\|\hatf -\scoref \|_{L_2(P)}^2$ for an estimator $\hatf$ of the true preference function. Given a set of ``worst-case'' true preferences, if the researcher outputs an estimated preference function in $\Fint$, we ask what the minimum error they could possibly achieve is. That is, if we had access to an oracle that could tell us the best estimation of a given preference function in $\Fint$, what is the worst error this oracle could ever incur? 

\begin{proposition}\label{prop:bias_genint_functions}
    Recall that $\|\hatf -\scoref \|_{L_2(P)}^2\le 1$ for any $\hatf$ with outputs in $[0,1]$. There exists some universal constant $\const$ such that for all $\numvals \ge 4$ and $\numcriteria \ge 2$,
    \[\sup_{\scoref \in \Fiso}\inf_{\hatf \in \Fint}\|\hatf -\scoref \|_{L_2(P)}^2 \ge \const\]
    where $P$ is the uniform distribution over $\calX$. 
\end{proposition}
We will now discuss what this result means for researchers and practitioners interested in learning preference functions. For any (reasonable) choice of estimator, $\hatf$, the risk is upper-bounded by $1$. This means that the worst-case error (e.g., that incurred by a random function over $[0,1]$) is a constant factor. Our result holds not just for a simple linear model, but also for those allowing for unique transformations of each criterion input, transformations of the output, and as many as $2^{\numcriteria}$ terms corresponding to interactions between criteria. The implication is that one can have a worst-case estimation error even if they employ many of the extensions of linear models used in practice. Hence, even if a researcher or practitioner assumes a very general class of models, they may still end up learning preferences quite poorly. The full proof of \Cref{prop:bias_genint_functions} is in \Cref{app:estimation_functions}; we include a proof sketch below. 

\textit{Proof sketch}: We begin by showing that there are ordinal relations over inputs that can be satisfied by functions in $\Fiso$ but not $\Fint$. To illustrate how this can arise, consider the case of a linear model. If a model is linear, $\scoref(1,2) > \scoref(2,1)$ implies that coefficient for $\xval_1$ is greater than the coefficient for $\xval_2$. Therefore, $\scoref(2,3) > \scoref(3,2)$ and in general for any $i>j$, $\scoref(j,i) > \scoref(i,j)$. This example illustrates that the relative value of the coefficients can impose a constraint on the relative value of the outputs. We prove that this same type of issue can arise even for model classes as expressive as $\Fint$. If there are such ordinal results that any $\hatf \in \Fint$ cannot satisfy, and the true preference function, $\scoref$ does, then $\hatf$ must induce a different ranking over the pairs of inputs than $\scoref$. If the outputs of $\scoref$ on those input values are non-equal, this leads to some amount of error. When there are enough instances of such sets of inputs, this leads to a very large error. We prove that the number of such subsets of inputs is $\Omega(\numvals^\numcriteria)$. This gives rise to our constant error result.$\qed$
\subsection{Ranking items}
In many evaluation settings, it is of interest to create a ranking or prioritization among a set of items, individuals, or options. In resource allocation problems, we may care about identifying individuals or communities who would benefit the most from investment. In admissions and hiring processes, there is typically a fixed number of applicants who can receive offers (i.e., only the top $k$ are selected). Concretely, \citet{lee_webuildai_2019} applies learned individual preferences to rank prioritization of food bank recipients. 

In this section, we consider the issues that could arise when using an estimated preference function to rank items when the researcher incorrectly assumes a model class. We consider a broad class of models. In particular, we consider a linear model where there is a set of non-decreasing transformations applied to each of the criteria inputs, and a transformation applied to the output of the weighted combination of the (transformed) inputs, $\Fgen$ (see Equation \ref{eq:genlin} for a definition). 

To measure the ability of the estimated preference function to rank items, we consider the rankings over all possible sets of criteria scores $\xval \in \calX$ induced by the estimated function $\hatf$, denoted by $\perm_{\hatf}$. We compare $\perm_{\hatf}$ to the ranking over sets of criteria scores induced by the true preference function $\scoref$, denoted by $\perm_{\scoref}$. The distance between the two rankings is measured by a normalized Kendall tau distance. Normalized Kendall tau distance sums the ranking error over all unique pairs of inputs and then divides by the total number of input pairs. We consider a variant of normalized Kendall tau distance that accounts for ties only when they occur in $\perm_{\hatf}$ and the pairs are ranked in $\perm_{\scoref}$. Otherwise, any constant function $\hatf$ would have a distance of zero. The error for inputs $\xval \neq \xval' \in \calX$ is 
\begin{align*}
    \ktcomp(\xval, \xval';\perm_{\hatf},\perm_{\scoref})
    & =\begin{cases}1 & \text{relative order is different under $\hatf$ and $\scoref$}\\
\frac{1}{2} & \text{tied in $\hatf$, ranked in $\scoref$}\\
0 & \text{otherwise}.
\end{cases}
\end{align*}
And, the total distance is 
\[\ktdist(\perm_{\hatf},\perm_\scoref)=\frac{1}{\binom{\numvals^{\numcriteria}}{2}}\sum_{\text{ unique }(\xval, \xval') \in \calX^2}\ktcomp(\xval, \xval';\perm_{\widehat{f}},\perm_\scoref).\]
The worst-case error of $\hatf$ is $1$ when all pairs are ranked incorrectly. We prove in \Cref{prop:bias_ranking_genlin} that there exist isotonic preference functions such that for any $\hatf \in \Fgen$, this worst-case error is incurred (up to constant factors). 
  
\begin{proposition}\label{prop:bias_ranking_genlin}
    There exists some universal constant $\const > 0$ such that for all $\numvals \ge 3$ and $\numcriteria \ge 2$, there exists a $\scoref \in \Fiso$ such that 
    \[\inf_{\hatf \in \Fgen}\ktdist (\perm_{\hatf},\perm_{\scoref}) \ge \const.\]
\end{proposition}

Our result implies that when there is a model mismatch, a constant fraction of all pairs of criteria scores could be ranked incorrectly. In hiring and peer review, an incorrect ranking means that a better paper or applicant would be ranked \textit{below} a worse one. Combined with our worst-case result on learning evaluator preferences \Cref{prop:bias_genint_functions}, our findings underscore theoretically substantial issues that could result from many common model assumptions. 

The full proof is in \Cref{app:ranking}. We will present a brief proof sketch below. 

\textit{Proof sketch}: We begin by showing that there are ordinal relations over inputs that can be satisfied by functions in $\Fiso$ but not $\Fgen$. This is the same approach we use in the proof of \Cref{prop:bias_genint_functions}. Our main result relies on an example of a ranking over 2 criteria and 3 possible criterion scores. One approach would be to divide $\calX$ into disjoint index subsets where each subset has the ranking that no $\hatf \in \Fgen$ can satisfy. Therefore, $\hatf$ would incur non-zero error over each subset. Summing over all subsets, we would get an error of order $\numvals^\numcriteria$ (the size of $\calX$). However, after normalization ($\frac{1}{\binom{\numvals^{\numcriteria}}{2}}$ is order $\numvals^{-2\numcriteria}$), the error would be only $O(\numvals^{-\numcriteria})$. To address this issue, we define a perturbation function with very small but unique entries and add it to our existing function. This perturbation function, loosely speaking, ensures that any $\hatf \in \Fgen$ incurs error when comparing entries between any two index subsets.\qed

\subsection{Measuring criteria importance}  
Lastly, researchers often use regression coefficients on the criteria output by linear models to understand the relative importance of each criterion. This is useful information that can convey what aspects of an item or individual matter more, i.e., the contribution or presentation of a paper, or the level of service or location for a hotel. Thus, we ask: \textit{how can assuming a linear model lead to issues when measuring relative criteria importance?}

To illustrate potential issues, consider a (simplified) setting of \citet{lane_architectural_2023}. The authors are interested in understanding how evaluators ``weigh'' two different evaluation criteria ---feasibility and novelty---and how these weights are dependent upon the type of expertise they have. For simplicity, we assume that each criterion is binary and that the researchers assume a linear model of the form $\scoref(\xval_1, \xval_2) = a_1\xval_1 + a_2\xval_2$. Before discussing potential issues arising from model mismatch, we will introduce the key theoretical result we rely on in our analysis. 

\begin{proposition}\label{prop:criteria_importance}
    Let $\numcriteria=\numvals =2$. Assume that we have access to data $\{\xval^{(j)}, \yval^{(j)}\}_{j \in [\numsamp]}$ where the $\yval^{(j)}$ are noiseless observations of the true preference function $\scoref \in \Fiso$. That is, for all $j \in [\numsamp]$, $\yval^{(j)}=\scoref(\xval^{(j)})$. Let $a_1\xval_1 + a_2 \xval_2$ be the OLS estimator fit on the data. There exist $\scoref \in \Fiso$ and sets of sample inputs $\{\xval^{(j)}\}_{j \in [\numsamp]}$ such that 
    \begin{enumerate}[label = (\alph*)]
        \item $a_1 < a_2$ when $\scoref(\xval_1, \xval_2) > \scoref(\xval_2, \xval_1)$ for all $\xval_1 > \xval_2$. 
        \item $a_1 < a_2$ when $\scoref(\xval_1, \xval_2) = \scoref(\xval_2, \xval_1)$ for all $(\xval_1, \xval_2) \in \calX$.
    \end{enumerate}
\end{proposition}
At a high level, part (a) implies that one can incorrectly assume criterion $1$ is weighted more than criterion $2$ when the reverse is true. Part (b) implies that one can find that one criterion is weighted above the other even when the underlying data-generating process is fully symmetric. 

We prove \Cref{prop:criteria_importance} in \Cref{app:criteria_importance} and present a brief intuition for this result below. 

\textit{Proof sketch}: The key issue arises from an imbalance in samples of different possible sets of criteria evaluations. For instance, consider the case when the true function is symmetric and equal to $0$ except when the inputs are both $1$. Intuitively, the more inputs of the form $(\xval_1, \xval_2) = (1,0)$, the more we hypothesize that the true value of $a_1$ is small. Conversely, the more inputs of the form $(\xval_1, \xval_2) = (0,1)$, the more it appears that $a_2$ is small. In general, we will assume that $a_1 > a_2$ whenever the number of inputs equal to $(1,0)$ is less than that equal to $(0,1)$. We prove that such issues can occur for symmetric and non-symmetric functions by proving that $a_1-a_2$ has a certain sign when $a_1$ and $a_2$ are the minimizers of the OLS problem.\qed

In \citet{lane_architectural_2023}, as is common in many applications, the authors use the values of the regression coefficients, $a_1, a_2$, to infer the relative worth or weight of both criteria. They consider two groups of evaluators: multivalent evaluators, those who have expertise in both relevant technical domains, and univalent evaluators, those with expertise in only one of them. The authors find that univalent evaluators overweigh novelty more so than multivalent ones. By part (a) of \Cref{prop:criteria_importance}, it is possible that the researchers found that novelty (criterion 2) matters more to univalent evaluators, when in fact it matters less. By part (b), they may equally have concluded that feasibility (criterion 1) matters more to multivalent evaluators when they are fully indifferent. Consequently, the true story may actually be that multivalent evaluators care \textit{more} about novelty than univalent evaluators, not less.  

\section{Our algorithm and theoretical guarantees}\label{sec:our_solution}

In this section, we will describe our algorithmic solution and its theoretical guarantees. Though we focus specifically on the linear assumption, the key techniques we use---cross-validation and regularization---as well as much of our theoretical analysis are applicable to developing robust algorithms for any parametric model choice. 

Our goal in developing an algorithm is two-fold: (1) we want to be able to learn what evaluators care about, even if the researcher's assumption of linearity does not hold; (2) we do not want to make the researcher worse off if their assumption does hold. The most natural baseline is non-negative least squares (hereafter \textit{linear regression}), which imposes non-negativity on coefficients to ensure monotonicity and performs similarly to unconstrained regression in practice. Our algorithm should ideally match its performance when linearity holds. By \Cref{prop:bias_genint_functions}, even generalized additive models with interactions can poorly capture more complex evaluator preferences, leading to a constant level of error. Thus, we cannot rely on linear regression alone. Yet if the model truly is linear (as the researcher assumes), it will be easier to learn through linear regression (when we assume linearity) than isotonic regression (when we assume it is merely isotonic). As neither isotonic nor linear regression alone suffices, this motivates the development of a new approach, which we detail in the following section. 

\subsection{Overview of our algorithm}

At a high level, our solution is to solve a regularized least squares problem over isotonic functions. We weigh both how far the observed preferences are from the estimated function (i.e., the fit on preference data) and how far that function is from any linear one (i.e., the amount of model mismatch). Our algorithm is presented as \Cref{alg:CV}. We will now interpret the algorithm and the associated sub-routines: regularized isotonic regression, post-processing, and cross-validation.  

\begin{algorithm}[tbp]
\caption{Our Algorithm}\label{alg:CV}
\begin{algorithmic}[1]
\Require Dataset of size $\numsamp$, set of regularization parameters $\Lambda$ containing $0$ and $\infty$
\State Randomly partition $S=[\numsamp]$ into a training set $S_0$ of size $3\numsamp/4$ and validation set $S_1 = S \setminus S_0$
\For{each $\lambda \in \Lambda$}
    \State Compute $(\hatlam, \hatlin) = \RLS(S_0, \lambda)$ by solving Equation ~\eqref{eq:regularized_isotonic} \Comment{solve regularized least squares}
    \State $\flam \gets$ \textsc{PostProcess}$(\hatlam, \hatlin, S_0, \lambda)$
    \Comment{truncate and interpolate, Protocol~\ref{prot:postprocess}}
    
    \State $\risklam = \sum_{j \in S_1}(\flam(\xval^{(j)})-\yval^{(j)}
    )^2$ \Comment{compute validation risk}
\EndFor
\State $\lambda^\star \gets \argmin_{\lambda \in \Lambda}\, \risklam$
\State Repeat Steps 3--5 with dataset $S$ and parameter $\lambda^\star$ to obtain $\fcv$.
\Comment{refit on full dataset}
\State \Return $\fcv$
\end{algorithmic}
\end{algorithm}

\paragraph{Regularization.} We test a set of regularization weights, denoted by $\lambda$, that allows us to interpolate between linear regression ($\lambda = \infty$) and isotonic regression ($\lambda = 0$). Recall that $\Fiso(\R)$ and $\Flin(\R)$ are the set of isotonic and linear functions with outputs in $\R$ (see Definitions \ref{defn:iso} and \ref{defn:lin}, respectively). Given $\dataset \subseteq [\numsamp]$ with unique inputs $\calX_{\dataset}$, the core optimization problem is:

\begin{equation}\label{eq:regularized_isotonic}
    \RLS(\dataset,\lambda) \!= \min_{\hatf \in \Fiso(\R)}\min_{g_{\mathtt{LIN}} \in \Flin(\R)} 
    \frac{1}{|\dataset|}\sum_{j \in \dataset}\!\left(\hatf(\xval^{(j)}) - \yval^{(j)}\!\right)^2 \!+\! \frac{\lambda}{|\calX_{\dataset}|}\!\sum_{\xval \in \calX_{\dataset}}\!\!\!\left(\hatf(\xval) - g_{\mathtt{LIN}}(\xval)\right)^2\!\!.
\end{equation}

We write $(\hatf, \hatlin)=\RLS(\dataset,\lambda)$ to denote the outer and inner minimizers of problem~\eqref{eq:regularized_isotonic} respectively. 

\paragraph{Post-processing: truncation and interpolation.} Ensuring a well-defined preference function in $\Fiso$ as output requires a few additional steps. Because of noise and the regularization component, the outputted function $\hatf$ of Equation \Cref{eq:regularized_isotonic} may map to values outside of $[0,1]$. Hence, we truncate each regression function $\hatlam$ to have outputs in $[0,1]$. To estimate values or criteria scores not seen in training, we use interpolation. For $\lambda =\infty$, we can extrapolate our (truncated) linear model. Otherwise, we set the output to the average of the smallest and largest values it can take while preserving the monotonicity and range constraints. These steps are illustrated fully in \Cref{prot:postprocess}. 
\begin{protocol}[tbp]
\caption{Post-Processing}\label{prot:postprocess}
\begin{algorithmic}[1]
\Require Fitted functions $\hatf$ and $\hatlin$, training set $\dataset$, parameter $\lambda$
\State $\hatf(\xval) \gets \max\left(0, \min(\hatf(\xval), 1)\right)$ \quad for all $\xval \in \calX_{\dataset}$
\Comment{truncate to $[0,1]$}
\For{all $\xval \notin \calX_{\dataset}$}
    \State $A(\xval) \gets \{\hatf(\zval) : \zval \in \dataset,\, \zval \succsim \xval\} \cup \{\max_{\zval \in \dataset} \hatf(\zval)\}$
    \State $B(\xval) \gets \{\hatf(\zval) : \zval \in \dataset,\, \zval \precsim \xval\} \cup \{\min_{\zval \in \dataset} \hatf(\zval)\}$
    \If{$\lambda \neq \infty$}
        \State $\hatf(\xval) \gets \tfrac{1}{2}\bigl(\min A(\xval) + \max B(\xval)\bigr)$
        \Comment{interpolation}
    \Else
        \State $\hatf(\xval) \gets \max\bigl(0,\, \min(\hatlin(\xval),\, 1)\bigr)$
        \Comment{linear extrapolation}
    \EndIf
\EndFor
\State \Return $\hatf$
\end{algorithmic}
\end{protocol}

\paragraph{Cross-validation.} The final regularization weight is carefully chosen from a set of weights, $\Lambda$. Cross-validation addresses the fact that, a priori, we do not know a suitable choice of $\lambda$. For nonlinear data, it is likely to be smaller. For linear data or data with fewer samples, it is likely to be larger. Cross-validation allows us to carefully pick $\lambda$ based on how well a set of regression functions $\{\hatlam\}_{\lambda \in \Lambda}$ fit the true function on an independent set of data. For simplicity, we use hold-out cross-validation and divide the dataset into training and validation subsets (denoted by $S_0$ and $S_1$) with a 75-25 split. However, our results hold when the split assigns any constant fraction of the data to both subsets, and one could also employ K-fold cross-validation.

\subsection{Theoretical guarantees}\label{sec:alg_gaurentees}
In this section, we provide theoretical results to support both goals (1) and (2) discussed in \Cref{sec:our_solution}. Our first goal is to learn what any evaluator cares about. Our second goal is to have our algorithm perform nearly as well as linear regression. Since the re-use of data in steps 8 and 9 is challenging to analyze, we will consider the estimator $\fcv=\widehat{\scoref}_{\lambda^\star}$ before retraining. In practice, however, retraining tends only to improve performance. 

\begin{theorem}\label{thm:CV}
    Let $\flin$ be the non-negative least squares estimator solving Equation \Cref{eq:regularized_isotonic} with $\lambda = \infty$ for the training set $S_0$, and let $\fcv=\widehat{\scoref}_{\lambda^\star}$. Let the additive noise terms, $w^{(j)}$, be independent samples from $\mathcal{N}(0,1)$. Then the following properties hold: 
    \begin{enumerate}[label = (\alph*)]
        \item There exists a constant $\const >0$ such that for any distribution $P$ over $\calX$,
        \[\sup_{\scoref \in \Flin}R(\fcv, \scoref)\le \const \log(\Lambda) \sup_{\scoref \in \Flin}R(\flin, \scoref).\]
        \item Assume that for some $\minval > 0$, $\min_{\xval \in \calX}P(\xval)\ge \frac{\minval}{|\calX|}.$ Let $\gamma_\numcriteria=(\numcriteria^2+\numcriteria+1)/2$ for $\numcriteria\ge 3$ and $\gamma_2=9/2$. Then there exists a constant $\const_{\numcriteria, \minval}>0$ depending only on $\numcriteria, \minval$ such that 
        \[\sup_{\scoref \in \Fiso}R(\fcv, \scoref) \le \const_{\numcriteria, \minval}\numsamp^{-1/\numcriteria}\log^{\gamma_\numcriteria}\numsamp\]
    \end{enumerate}
\end{theorem}
Part (a) of the theorem tells us that given a linear preference function, the risk of our algorithm is upper-bounded by the risk of $\flin$ up to a constant factor depending logarithmically on $|\Lambda|$. Hence, the asymptotic dependence on $\numsamp$, $\numvals$, and $\numcriteria$ is just as good for our algorithm as non-negative least squares. Part (b) of the theorem states that if all criteria scores have some positive probability of being observed, the risk of our algorithm decays with $\numsamp^{-1/\numcriteria}$ (plus constant and logarithmic factors). The assumption on $\minval$ is minimal (it can be arbitrarily small). As the number of samples increases, this bound becomes tighter with error vanishing completely as $\numsamp \to \infty$. Together, these results illustrate that if one assumes a linear model, they cannot be meaningfully worse off from using our algorithm. And, on the flip side, they may avoid disastrous consequences (as illustrated in \Cref{prop:bias_genint_functions}) if the true preference function is more complex. The full proof of \Cref{thm:CV} is in \Cref{app:our_solution}, and we will provide a brief proof sketch below. It is worth noting that while this proof only requires $\Lambda \supseteq \{0,\infty\}$, more values of $\lambda$ can be helpful in practice. Most data falls somewhere between having enough data for pure isotonic regression and having more than enough for a simple linear model, justifying an intermediate regularization parameter. Indeed, across our synthetic simulations, which vary dataset size and functional form, about half select an intermediate $\lambda$ as discussed in \Cref{app:reg_param}.

\textit{Proof sketch}: Proving this result has two general components. Firstly, we use an existing oracle inequality from \citet{vaart_oracle_2006} that upper bounds the risk of the cross-validation estimator in terms of the risk of any of the $\hatf_{\lambda}$ and an additive $\log |\Lambda|/\numsamp$ term (up to constant factors). To prove part (a), it suffices to show that $\sup_{\scoref \in \Flin}R(\flin, \scoref) = \Omega(1/\numsamp)$. This is a well-known lower bound on linear estimators. Since we include a truncated linear regression estimator for cross-validation, this bound implies that using cross-validation leads to the same asymptotic dependence on $\numcriteria, \numsamp$ and $\numvals$. To prove part (b), it suffices to prove a $O(\numsamp^{-1/\numcriteria}\log ^{\gamma_\numcriteria} \numsamp)$ upper-bound on the risk of isotonic regression. We do so by extending existing results on isotonic regression sample complexity from \citet{han_isotonic_2019} to our setting, which has a discrete instead of continuous space for random inputs. \qed 

\section{Empirical results}\label{sec:empirical_results}
In this section, we present synthetic simulations and empirical results on real data from human and LLM evaluations. Here, we provide a brief overview of the corresponding results. In \Cref{sec:synthetic_simulations}, we analyze how our algorithm performs relative to linear regression for learning linear and nonlinear preferences. These results illustrate the potential bias from assuming a linear model and show that our algorithm does nearly as well as linear regression for both linear and nonlinear preference models, aligning with our theoretical results in \Cref{sec:alg_gaurentees}. To test how our algorithm fares on real data, we compare our algorithm with linear regression on a dataset of Tripadvisor reviews in \Cref{sec:trip_advisor}. Our findings demonstrate the relative improvement from our algorithm in reducing prediction error as well as the distance from the best-performing oracle, which proxies for estimation error. Furthermore, in \Cref{app:ml_comparisons} we present a comparison with two monotone-constrained machine learning models and find our algorithm performs similarly or better on both synthetic and Tripadvisor data. In \Cref{appendix:additional_human_evaluator}, we repeat the comparative analysis with linear regression on two other human review datasets: ICLR peer reviews (2023-2025), and engineering solutions in aerospace and robotics from \citep{lane_architectural_2023}; additionally, we investigate how the (differing) expertise of evaluators influences the consistency of their preferences. In \Cref{sec:iclr_case_study}, we provide a practical case study in using our algorithm to analyze LLM evaluations. Using data of ICLR reviews from both humans and LLMs, we evaluate how consistent LLM preferences are across reviews, and how close those preferences are to those of human evaluations. Finally, we have a separate section detailing how our method can provide insights about and increase the fairness of the peer review process. For all results, we run our algorithm using the set of regularization parameters $\Lambda = \{2^k: k \in \Z, -9 \le k \le 8\} \cup\{0,\infty\}$ and we discuss the chosen parameter values in detail in \Cref{app:reg_param}. 

\subsection{Synthetic simulations}\label{sec:synthetic_simulations}
In this section, we conduct simulation studies to evaluate the performance of our algorithm relative to linear regression as the sample size increases for learning both linear and non-linear preferences. In particular, we compare our algorithm against non-negative least squares (linear regression that enforces the monotonicity constraint) for three classes of preference or utility functions common in the economics literature. Each of these preference functions is parameterized by a vector $a \in \R^\numcriteria_{\ge 0}$ and normalized such that they provide outputs in $[0,1]$. In addition, we include other comparisons with neural networks and gradient-boosted tree ensembles in \Cref{app:ml_comparisons}.~\\

\begin{itemize}[itemsep=0.5em]
    \item \textit{Linear utility.} Under linear preferences, criteria are perfect substitutes for one another. That means we can exchange an increase in criterion 1 for some increase in criterion 2, which doesn't depend on the level of the criterion scores:   
    \[\scoref(\xval) = \frac{1}{\numvals \|a\|_1}\sum_{i \in [\numcriteria]}a_i\xval_i.\] 
    \item \textit{Leontief utility.} Leontief utility assumes that criteria 1 and 2 perfectly complement one another. The output of the function is the lowest of the weighted criteria scores. Therefore, an increase in criterion 1 matters only if criterion 2 is sufficiently high: 
    \[\scoref(\xval) = \frac{1}{\numvals \min_{i \in [\numcriteria]}a_i}\min_{i \in [\numcriteria]} a_i\xval_i.\]
    \item \textit{Cobb-Douglas utility.} Cobb-Douglas utility assumes that the criteria scores (raised to some power) multiply one another. Hence, the relative benefit of a criterion score is dependent upon the other criteria scores. In Cobb-Douglas preferences, the complementarity of the criteria lies somewhere between linear and Leontief preferences: 
    \[\scoref(\xval) = \frac{1}{\numvals^{\|a\|_1}}\prod_{i \in [\numcriteria]}\xval_i^{a_i}.\]
\end{itemize}
To measure the ability of both algorithms to learn these three classes of preference functions, we run $50$ trials for a set of $10$ sample sizes between $50$ and $1000$ with $\numcriteria=2$ criteria and $\numvals = 5$ possible score values per criterion. For each trial, the preference parameter $a$ is drawn uniformly from $[1,2]^{\numcriteria}$ and we observe samples with independent $N(0,\sigma^2)$ noise with $\sigma = 0.2$. We set $P$ to be the uniform distribution over $\calX$. After we fit the models on the chosen data, we compute the risk or estimation error for both algorithms. Error bars represent a $95\%$ confidence interval for the expected estimation error and are computed using the sample standard error. 

\begin{figure}[tbp]
    \centering
    \includegraphics[width=\linewidth]{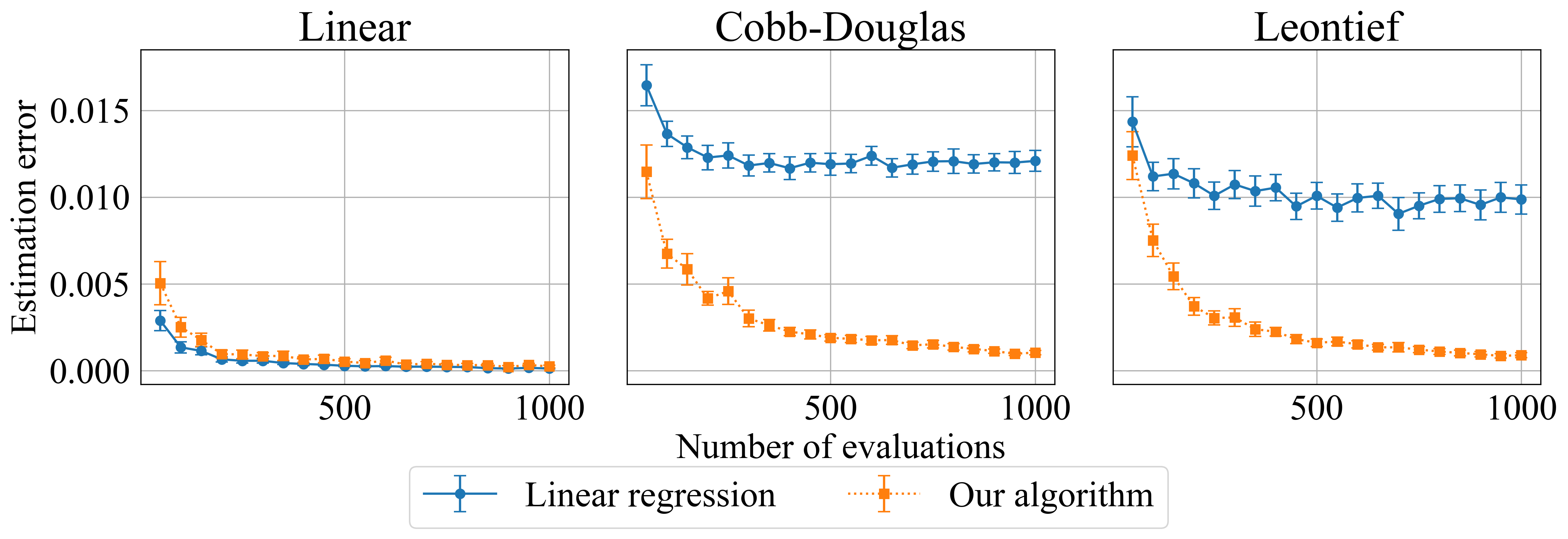}
    \caption{Simulation results for synthetic preferences. Error bars are computed using the standard error of the mean estimation error.}
    \label{fig:synthetic}
\end{figure}
Our theoretical results, namely \Cref{thm:CV} part (a), predict that our algorithm should converge at a similar rate compared with linear regression. As shown in \Cref{fig:synthetic}, they have very similar risks, even for the smallest sample size, $\numsamp = 50$. Additionally, note that for the non-linear functions (Cobb-Douglas and Leontief), linear regression is unable to decrease error at some point around $\numsamp = 400$ samples. This illustrates the bias when linear models are used to learn functions that their model class does not contain, aligning with our bias result, \Cref{prop:bias_genint_functions}. In contrast to linear regression, our algorithm is still able to learn the non-linear utility functions, albeit at a slower rate. The two algorithms are comparable at small sample sizes, but as the sample size increases, the value-added of using our algorithm increases. For Leontief preferences, at $\numsamp=50$ the average error of linear regression is about $1.2$ times that of our algorithm, at $\numsamp = 1000$ the average error is over $13$ times larger. 

\subsection{Human reviews on Tripadvisor}\label{sec:trip_advisor}
The goal of this section is to compare the performance of and insights from our algorithm relative to those of linear regression. Our data is of user ratings of hotels on Tripadvisor obtained from the HotelRec dataset \citep{antognini_hotelrec_2020}. For additional analysis on data of ICLR peer reviews (2023-2025) and engineering solutions in aerospace and robotics from \citep{lane_architectural_2023}, please see \Cref{appendix:additional_human_evaluator}. 

\subsubsection{Data information}
On Tripadvisor, people typically rate their overall experience and provide sub-ratings on criteria such as value, location, and sleep quality. Each of these criteria and the overall ratings is given between 1 and 5 stars. To ensure we have captured the relevant information in someone's review, we restrict our attention to the set of reviews containing only overall scores and the six most common criteria: value, rooms, location, cleanliness, service, and sleep quality (see \Cref{fig:Tripadvisor} for an example). For the same reason, we consider only reviews after 2014, after which the number of user ratings of other criteria (e.g., business service and check in/ front desk) is negligible.

We first consider all users in our dataset. Since preferences may vary based on the kind of trip, we additionally fit a preference function based on ``trip type.'' As the trip type is not part of the HotelRec dataset, we use a rule-based classification of review comments as a proxy. We only test categories for which the classification was reliable, meaning that out of a random sample, at least 8 out of 10 were correctly classified: these were business, family, and couples. For learning preferences of different trip types, we included data from 2014-2019. To avoid long runtimes, for all users, we used evaluations from 2019. 

\begin{table}[tbp]\label{table:data_summary}
    \centering
    \begin{tabular}{lcccc}
        \toprule
        Dataset & \# criteria ($\numcriteria$) & \# values ($\numvals$) & score range & \# samples ($\numsamp$) \\
        \midrule
        \multicolumn{5}{l}{\textit{Tripadvisor ratings}} \\
        \quad All (2019)             & \multirow{4}{*}{6} & \multirow{4}{*}{5} & \multirow{4}{*}{1--5} & 332{,}349 \\
        \quad Couples (2014--2019)   &                    &                    &                       & 302{,}894 \\
        \quad Family (2014--2019)    &                    &                    &                       & 371{,}185 \\
        \quad Business (2014--2019)  &                    &                    &                       & 197{,}031 \\
        \bottomrule
    \end{tabular}
    \caption{Summary of dataset attributes.}
\end{table}

\subsubsection{Methods}
For each dataset, we performed a random train-test split with 80\% of the data for training and 20\% of the data for testing. We used \Cref{alg:CV} on the training dataset as well as non-negative least squares. Now we will explain the questions we will ask and how we plan to answer them, given the data. 

\textit{\textbf{Question 1:} How good is our algorithm relative to linear regression at predicting overall scores given criteria scores?} We estimate the prediction error as the mean difference between the output of our estimated preference function $\hatf$ and the true function $\scoref$ on the test data. Let $\testset \subset [\numsamp]$ be the test data. We report
\begin{equation}\label{eq:prediction_error}
  \text{Prediction error}=\frac{1}{|\testset|}\sum_{j \in \testset}\left(\hatf(\xval^{(j)}) - \yval^{(j)}\right)^2.  
\end{equation}

\textit{\textbf{Question 2:} How good is our algorithm relative to linear regression at estimating the evaluators' preference function?} Recall that the expected prediction error is the sum of estimation error (the relevant quantity of interest) and the level of noise. If the level of noise is high, it can mislead us into thinking that the two algorithms being compared are relatively similar when significant differences in the estimation error exist. Thus, we need a way to disentangle the error coming from noise and the error in our estimated preference function. In synthetic simulations, we know $P$, $\scoref$, $\hatf$, and could use these to directly compute estimation error. In real data, we have access to neither $\scoref$ nor $P$. To circumvent this, we compute a proxy for the level of noise and estimation error. The expectation of the error of $\scoref$ is the level of noise $\|w\|_{L_2(P)}$. The empirical mean squared error of $\scoref$ on the test set would provide an estimate of noise. Without access to $\scoref$, we consider the lowest error \textit{any} preference function could achieve. Because no algorithm could achieve an error below this, we call it the irreducible error. Given a dataset $\dataset \subseteq [\numsamp]$:
\begin{equation}\label{eq:irreducible_error}
    \text{Irreducible error} = \frac{1}{|\dataset|}\min_{\widetilde{\scoref} \in \Fiso}\sum_{j \in \dataset}\left(\widetilde{\scoref}(\xval^{(j)}) - \yval^{(j)}\right)^2.
\end{equation}
Using irreducible error computed on the test set (i.e., $\dataset = \testset$) as our proxy for noise, we can compute an estimate of the estimation error. As this error can be fully eliminated given the proper choice of a preference function, we refer to it as \textit{reducible error}.  
\begin{equation}\label{eq:reducible_error}
  \text{Reducible error}=\underbrace{\frac{1}{|\testset|}\sum_{i \in \testset}\left(\hatf(\xval_i) - \yval_i\right)^2}_{\text{Prediction error}} - \frac{1}{|\testset|}\underbrace{\min_{\widetilde{\scoref} \in \Fiso}\sum_{j \in \testset}\left(\widetilde{\scoref}(\xval^{(j)}) - \yval^{(j)}\right)^2}_{\text{Irreducible error}}. 
\end{equation}
Since the true function has a weakly higher error than $\widetilde{\scoref}$, our estimate of reducible error has a negative bias. Thus, if anything, it understates the relative benefit of our algorithm. 

\textit{\textbf{Question 3:} What insights can our learned functions tell us about people's preferences?} To further illustrate the benefit of learning preferences as well as possible, we provide an example of how we can use learned functions to understand what people care about. One question of interest is the relative impact of an increase in the score of one criterion (all else constant) compared with another. For instance, does raising a service or rooms rating matter more, and how does this depend on other aspects of the stay? There is no natural analog of weights for multidimensional isotonic functions. Instead, we capture the relative effects of a marginal change in criterion score by graphing the learned function for different values of that criterion while holding the remaining criteria fixed. Error bars are constructed using the standard error of prediction error (as an upper-bound for that of reducible error). Please see \Cref{app:criteria_plots} for a description. 

\subsubsection{Results}
We begin by answering questions 1 and 2 regarding the prediction and reducible error measurements. \textit{In general, we find our algorithm to be slightly better than linear regression in terms of prediction error, but significantly better in terms of reducible error (at approximating the community-level function).} \Cref{fig:empirical_results} shows the prediction error for isotonic regression compared with linear regression. Error bars representing 95\% confidence intervals for prediction error are included based on the standard deviation of test error for both measures. The confidence intervals are computed by using the squared prediction error on the test set. Our algorithm has lower prediction error across all datasets, and, as indicated by the error bars, these differences are statistically significant. The benefit of our algorithm becomes more apparent when looking at reducible error in \Cref{fig:empirical_results}. \textit{For the group of all users, our algorithm results in over a 50\% reduction in reducible error compared with linear regression, and as much as a 69\% reduction.} In the context of the Tripadvisor dataset, this means our algorithm is far better at learning a collective function of what travelers care about (the estimation problem) and only marginally better at predicting overall scores. 

\begin{figure}[tbp]
    \centering
    \includegraphics[width=\linewidth]{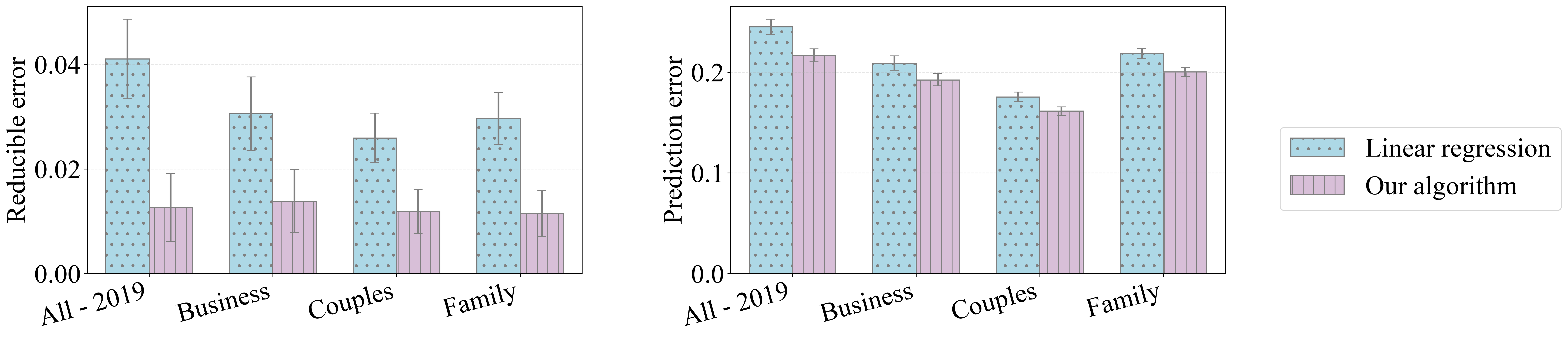}
    \caption{Results of our algorithm relative to linear regression on Tripadvisor rating data. Error bars represent 95\% confidence intervals and are computed via the standard deviation of prediction error.}\label{fig:empirical_results}
\end{figure}

\begin{table}[tbp]
\caption{Empirical results on Tripadvisor dataset. Standard errors are computed with the standard error of the mean for the prediction error. $\Delta$ (\%) denotes the relative error reduction of our algorithm over linear regression.}\label{tab:trip_advisor_results}

\centering
\begin{tabular}{lccrccrr}
\toprule
 & \multicolumn{3}{c}{Prediction error} & \multicolumn{3}{c}{Reducible error} \\
\cmidrule(lr){2-4} \cmidrule(lr){5-7}
 & Lin. Reg. & Our algo & $\Delta$ (\%) & Lin. Reg. & Our algo & $\Delta$ (\%) \\
\midrule
All - 2019 & 0.245 (0.004) & 0.217 (0.003) & 11.6\% & 0.041 (0.004) & 0.013 (0.003) & \textbf{69.1\%} \\
Business travel & 0.209 (0.004) & 0.192 (0.003) & 8.0\% & 0.031 (0.004) & 0.014 (0.003) & \textbf{54.6\%} \\
Couples travel & 0.176 (0.002) & 0.161 (0.002) & 8.0\% & 0.026 (0.002) & 0.012 (0.002) & \textbf{54.2\%} \\
Family travel & 0.219 (0.003) & 0.200 (0.002) & 8.3\% & 0.03 (0.003) & 0.011 (0.002) & \textbf{61.3\%} \\
\bottomrule
\end{tabular}
\end{table}

The implications of these findings for practitioners and researchers are important. As indicated by our synthetic simulations (see \Cref{fig:synthetic}) and the bound on the error for linear models in \Cref{thm:CV} (a), our algorithm never does much worse than linear regression. This is supported by our empirical findings as well. However, the cost of changing methods can be high and not worth it for only a modest improvement. On this note, we suspect that if the goal is to predict overall scores or evaluations, linear regression may suffice, especially given very noisy data. To test how noisy the data is, a researcher can compute the irreducible error. This error provides a lower bound on the performance of our algorithm. Depending on how the irreducible error and the error of linear regression compare, the researcher can then decide if our algorithm is worth implementing. On the flip side, if the goal is to learn evaluator preferences as well as possible and apply this mapping for some downstream application (i.e., ranking papers), we suspect there could be substantial benefit in using this more general approach, as indicated by the significant reduction in reducible error on the Tripadvisor dataset. 

To answer question 3, we plot the values of the learned preference function for the 2019 dataset, fixing all but one criterion at 3 stars. Specifically, we look at the effects of the first three criteria: service, location, and value. The result is shown in \Cref{fig:varying_criteria}. Interestingly, we find that the curves generally have a concave shape. This is consistent with the intuition that one truly bad aspect of a stay can have a significant impact on the overall experience, but even really good service or location can only make an otherwise average experience marginally better. Additionally, the relative benefit of an increase in criteria scores (service, location, or cleanliness) depends on the level of that score. When changing from 1-2 stars, the service criterion has the largest impact; when moving from 4-5 stars, they all essentially have no effect. \textbf{A linear model, which assumes constant marginal changes, cannot capture either of these two insights.} This example illustrates that linear regression will always be constrained in its ability to tell us about evaluator preferences. Hence, if the goal is to learn what evaluators care about, it may be worth running our algorithm even if the error reduction is modest. 

\begin{figure}
    \centering
    \includegraphics[width=0.5\linewidth]{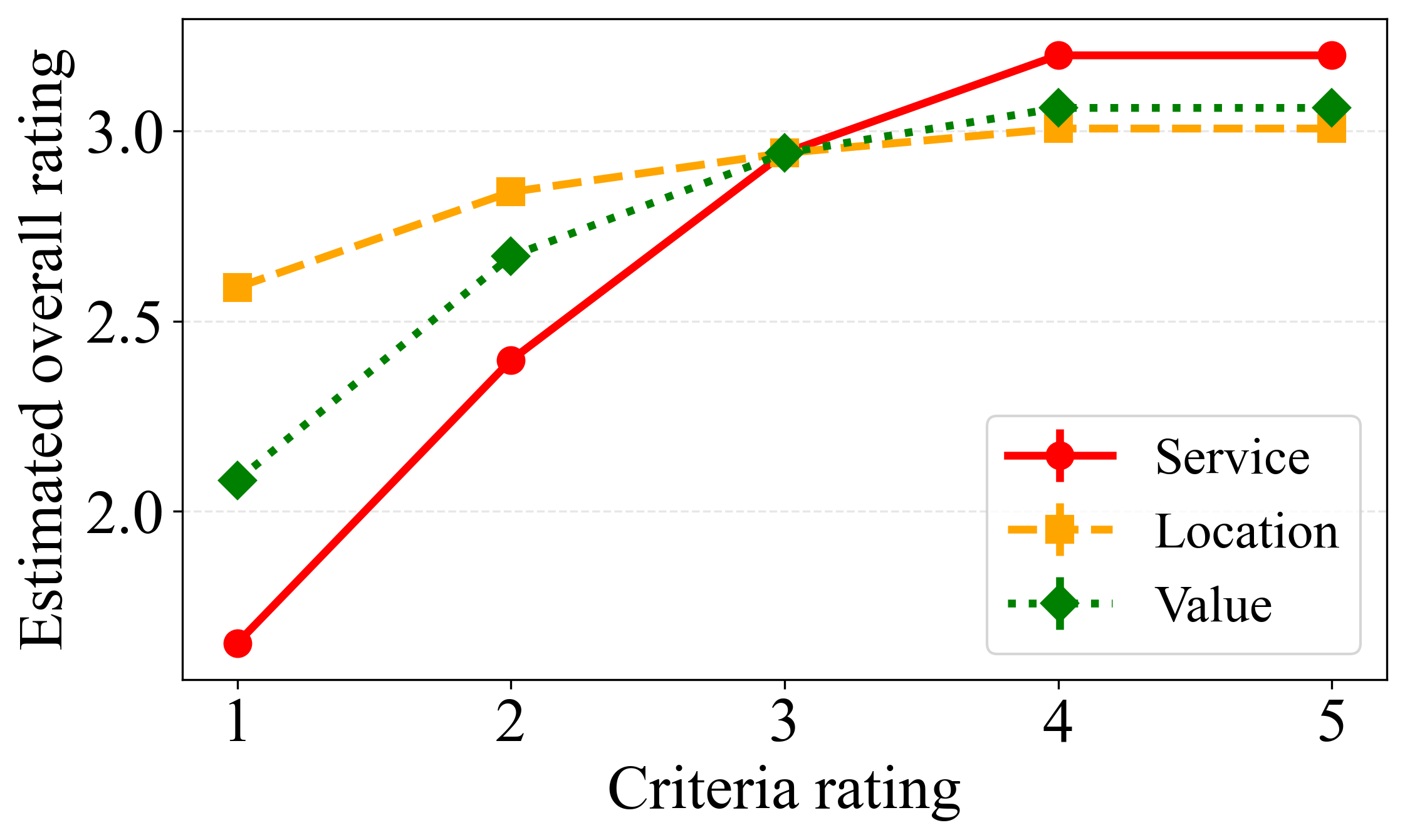}
    \caption{Impact of increasing different criteria scores on the overall rating of the learned preference function for the dataset of Tripadvisor reviews from 2019. Error bars are computed based on the standard error of prediction error (see \Cref{app:criteria_plots}). Other criteria scores are fixed at 3 out of 5 stars.}\label{fig:varying_criteria}
\end{figure}

\subsection{LLM and human peer review dataset}\label{sec:ICLR}
Increasingly, large language models are being considered in addition to or in place of human evaluators in areas such as medicine, hiring, and peer review. Given this fact, it is imperative that we understand how LLMs make evaluation decisions. 
To this end, we provide a case study of how one could use our methods to qualitatively analyze the internal consistency of LLM evaluation preferences and how they differ from human preferences. To our knowledge, we are the first to study both of these topics through the lens of learned preference functions. While there is existing work on using human preference functions to create more aligned LLM models \citep{liu_hd-eval_2024}, and modeling preference discrepancies between humans and LLMs \citep{polo_bridging_2025}, we are not aware of any current work that specifically studies the differences in LLM and human preference functions, nor how consistent the criteria-to-overall score mapping is between evaluations.~\\

\subsubsection{Data information}
We use data from \citet{yu_is_2026} containing human-written and AI-written peer reviews for papers from ICLR 2024. The AI reviews are generated using two widely used LLMs: GPT-4o and Llama 3.1 70b. We use three datasets: one for the data of human reviews, one of GPT reviews, and one of Llama reviews. Each dataset contains $~27$k reviews corresponding to the same set of papers (see \Cref{tbl:data_ICLR_summary} for the dataset descriptions). \Cref{fig:tripadvisor_iclr} (b) illustrates an example review. Reviewers provide criteria assessments for presentation, contribution, and soundness. These each take on integer values between $1$ (poor) and $4$ (excellent). Additionally, the reviewers provide an overall rating or recommendation, which can take on values between $1$ (strong reject) and $10$ (strong accept).  

\begin{table}[tbp]
    \centering
    \begin{tabular}{lcccc}
        \toprule
        Dataset & \# criteria ($\numcriteria$) & \# values ($\numvals$) & score range & \# samples ($\numsamp$) \\
        \midrule
        \multicolumn{5}{l}{\textit{ICLR peer reviews}} \\
        \quad Human evaluators             & \multirow{3}{*}{3} & \multirow{3}{*}{4} & \multirow{3}{*}{1--10} & 27,857 \\
        \quad GPT-4o   &                    &                    &                       & 27,857 \\
        \quad Llama 3.1 70b    &                    &                    &                       & 27,834 \\
        \bottomrule
    \end{tabular}
    \caption{Summary of ICLR dataset attributes.}\label{tbl:data_ICLR_summary}
\end{table}

\subsubsection{Methods}

\textit{\textbf{Question 1}: How consistent are the preferences of LLMs compared with human evaluators?} One of the key differences in using human versus LLM evaluators is that human evaluators tend to have their own individual preferences, while the same LLM makes all evaluations. Thus, one would expect LLM preferences to be more consistent. We investigate to what extent this is true, and how the consistency varies based on the choice of model (GPT or Llama). To do so, we measure the level of irreducible error given by equation \Cref{eq:irreducible_error} on the full dataset. 

In order to estimate uncertainty in the reported statistic, we use bootstrapping. This helps us avoid computing the error minimizing function and testing for significance on the same data, which would be ``double dipping.'' For each resample, we sample with replacement from the whole dataset, and then recompute the irreducible error. In total, we compute $1000$ resamples, and then report the $95\%$ confidence interval as the $2.5$th and $97.5$th percentiles of the bootstrap distribution.~\\ 

\textit{\textbf{Question 2}: How much do the preferences of humans and LLMs differ? What are some ways in which they are different?}
As LLMs are being considered (or used) as evaluators for applications in medical care, peer review, and hiring (among others), to complement or replace human judgment. Therefore, it is important to understand how much human and LLM judgments differ. We do this by first estimating the preference functions for GPT-4o, Llama 3.1 70b, and human evaluators on the full dataset. Then, we compute the expected error between the outputs of the two preference functions. We set the expectation, denoted by $P$, to be the empirical distribution over criteria scores for the test set of human evaluations to reflect our definition of risk in \Cref{eq:expected_risk}. As this difference reflects the disagreement of LLM and human preferences, we refer to it as preference misalignment. Let the human and LLM preference function $\hatf_{\mathtt{Human}}$ and $\hatf_{\mathtt{LLM}}$, respectively, we compute
\begin{equation}\label{eq:dist_criteria}
    \text{Preference misalignment} = \E\|\hatf_{\mathtt{Human}}-\hatf_{\mathtt{LLM}}\|_{L_2(P)}^2.
\end{equation} 
In questions of moral decision-making (e.g., kidney or food bank allocation), minimizing this deviation may be critical to ensure such decisions truly reflect human values. 
To understand ways in which they are different, we create plots that vary one of the three criterion scores, while fixing the other criterion scores at 2 or 3 (the most common scores). As with our Tripadvisor dataset, we use the standard error of prediction error to compute error bars for the plot. 

\subsubsection{Results}
We present the measurements of both qualities in \Cref{fig:iclr_results}. 
Our first quantity---irreducible error---captures the level of noise, and tells us how consistent the preferences of human evaluators or a given LLM are across reviews. The second quantity---preference misalignment---concerns the expected discrepancy between human and LLM evaluation preferences. Overall, we find GPT-4o to be significantly better than Llama 3.1 70b on both metrics. The irreducible error or noise level of GPT is nearly a quarter that of human evaluators, while the irreducible error of Llama and Human reviews is comparable. Additionally, the preferences of GPT are significantly closer to human ones as measured by our distance function. As illustrated by \Cref{fig:iclr_results}, the measures of preference alignment for GPT and Llama differ by a factor of more than five. Though this analysis is not comprehensive enough to draw any conclusions, our case study provides interesting insights, suggesting that preference alignment and consistency may vary widely depending on the specific LLM. 

\begin{figure}[tbp]
    \centering
    \includegraphics[width=0.8\linewidth]{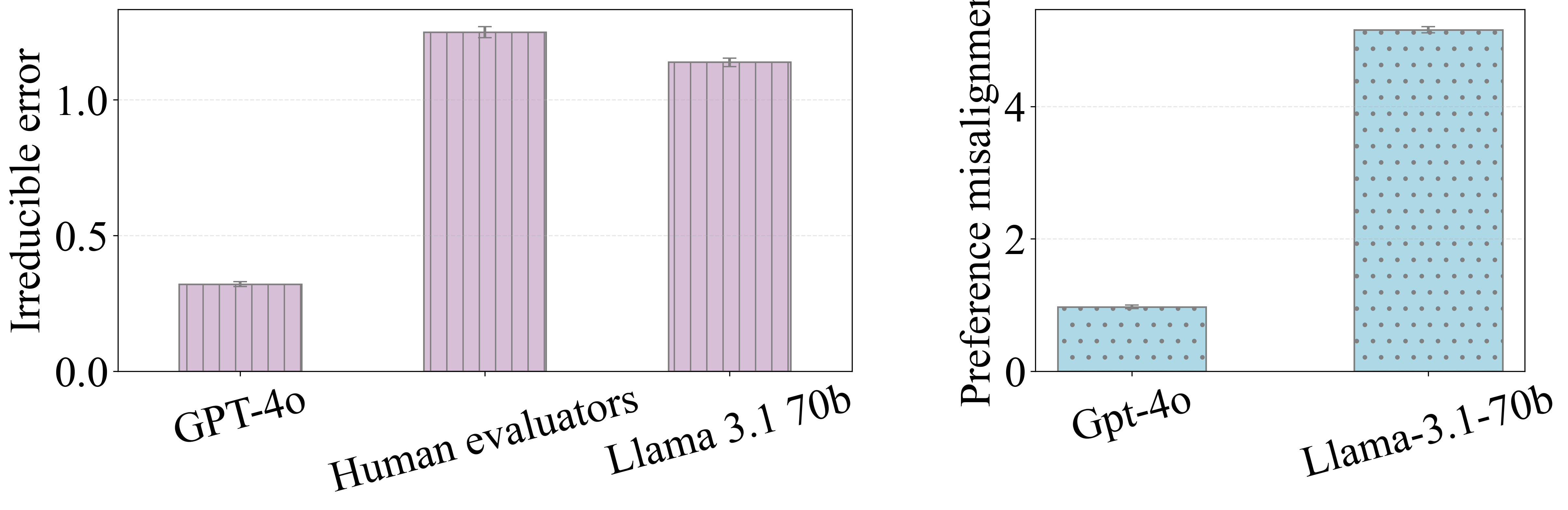}
    \caption{Consistency and preference alignment for LLM (and human) reviewers on the ICLR dataset. Error bars for irreducible error reflect 95\% bootstrap confidence intervals (percentile method); error bars for preference alignment are computed via the standard error of the mean.}\label{fig:iclr_results}
\end{figure}

\begin{table}[tbp]
\caption{Level of preference consistency and alignment by evaluator type. Standard errors are in parentheses for preference misalignment. 95\% bootstrap CI are in brackets for irreducible error.}
\label{tbl:results_noise}
\centering
\begin{tabular}{lccc}
\toprule
 & GPT-4o & Llama 3.1 70b & Human evaluators \\
\midrule
Irreducible error & 0.321 [0.311,0.330] & 1.139 [1.122,1.153] & 1.251 [1.228,1.267] \\
Preference misalignment & 0.971 (0.014) & 5.157 (0.024) & --- \\
\bottomrule
\end{tabular}
\end{table}

Finally, to understand ways in which human and LLM preferences differ, we plot the estimated preference function fixing the two criterion scores (soundness and contribution) and varying the third criterion (presentation). As shown in \Cref{fig:varying_criteria_LLM}, we find that the preference function for GPT is closer to the human one, consistent with our findings for preference misalignment. Additionally, we find that the preferences of GPT (and to a lesser extent, Llama) are concave when the other criterion scores are 2, and convex when they are 3. In contrast, the human preference function is close to linear in both plots. Concavity signal diminishing returns to an increase in the criterion (i.e., an increase in a score can only help the paper so much) while convexity signals increasing returns (i.e., a very good score leads to a larger benefit). While it is possible to know for certain, given the truncated criterion score range for GPT, it is possible that the importance of presentation is qualitatively different for GPT than for human reviewers. Though the change from concavity to convexity is especially pronounced when varying the contribution criteria, a similar pattern holds when varying the other two criterion (see \Cref{app:criteria_plots}). 

\begin{figure}
    \centering
     \includegraphics[width=.7\linewidth]{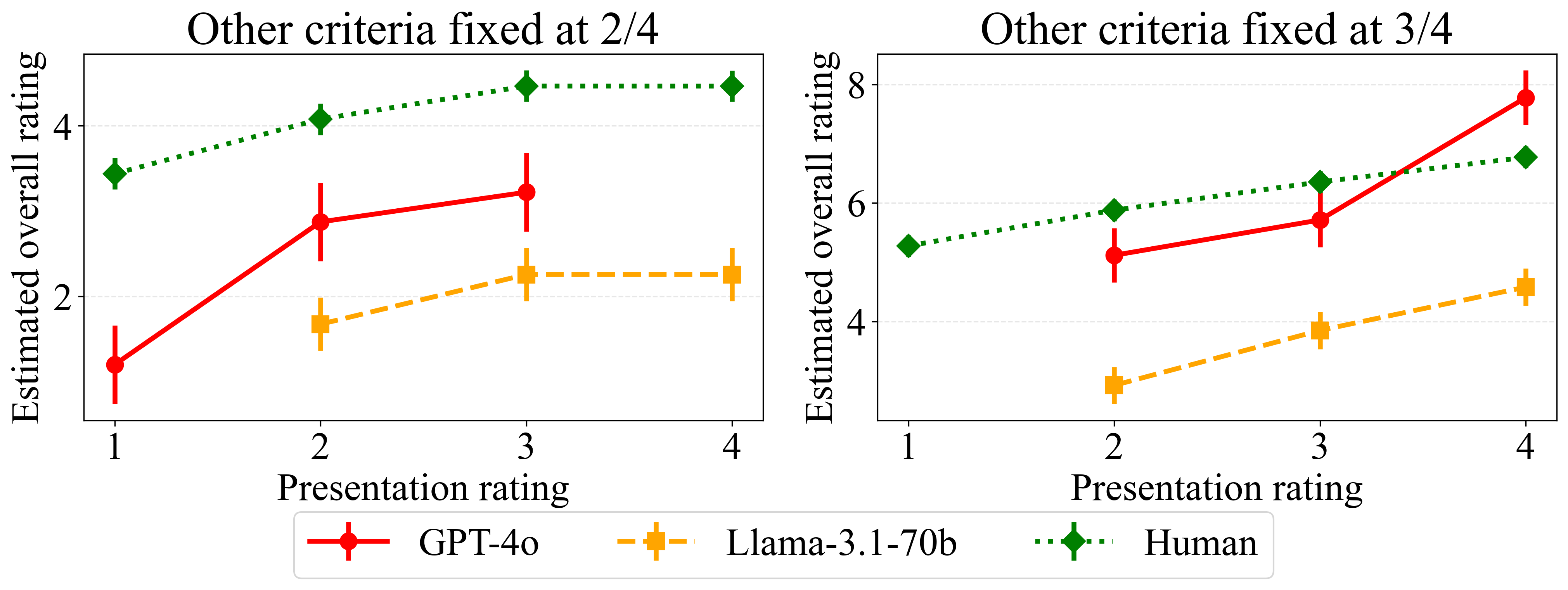}
    \caption{Impact of increasing presentation on the overall paper rating. We omit estimates for all criterion scores not seen in training. Error bars are computed based on the standard error of prediction error (see \Cref{app:criteria_plots}). }\label{fig:varying_criteria_LLM}
\end{figure}

\section{Peer review case study}\label{sec:iclr_case_study}
In social decision-making contexts, it is often many individuals who provide evaluations rather than a single person or entity (e.g., an algorithm or LLM). In medicine, there are multiple clinicians, each assessing different patients; in the judicial system, there are multiple judges, each tasked with determining sentences for different defendants; in peer review, many different evaluators assess conference papers. In these settings, we call the learned mapping the ``community preference function'' to underscore that it captures a group of individuals with heterogeneous beliefs. The purpose of this section is two-fold: (1) to illustrate how we can use the community preference function in practice to understand what the group of evaluators care about, and (2) to alleviate unfairness due to preference variability across evaluators and contexts. We will do so by considering the specific application of academic peer review, using data from the ICLR 2026 review process; however, the questions we ask and methods we use apply to the other aforementioned settings (e.g., judicial and medical decision-making) as well. 

\subsection{Dataset and calibration}\label{sec:ICLR_data_info}
We analyze the way ICLR reviewers aggregate different criteria into overall scores. In total, there were 75,859 reviews across 19,494 submissions for the ICLR 2026 review process. We also did this same analysis for ICLR 2025 with similar results. Each review contains a score for three criteria––presentation, soundness, and contribution––in addition to an overall rating. The criteria are rated on a scale from 1 (poor) to 4 (excellent), and the overall ratings take on values in the $\{0,2,4,6,8,10\}$ scale, with higher scores reflecting better papers.~\\

\noindent \textit{A note on calibration}. Before discussing our methods in more detail, we will provide some context for the ratings distribution and what constitutes a large change in the overall rating. In general, ratings crowd around the acceptance threshold. As such, even small changes in ratings can be quite consequential. To make this concrete, we model acceptance as taking the submissions with the highest mean overall rating, with the cutoff set to reproduce ICLR 2026's acceptance rate of 29\% among the 18,789 submissions that ultimately received an accept or reject decision (this excludes desk-rejected submissions). Under this model, the decision for about 30\% of submissions would flip if a single randomly chosen reviewer changed their overall score rating by one rating level (2 points on the 0-10 scale), and 38\% sit within one reviewer's single-level change of the cut-off.

\subsection{Understanding community preferences}
We will begin with the task of using the learned community preference function to provide qualitative information on what matters to reviewers. For settings with more than two criteria, we cannot visualize the preference function. Furthermore, as our function class is non-parametric, one cannot look to specific parameter values for insights. To address this, we present a few common qualitative questions and methods for answering them.~\\

\subsubsection{Which criteria matter most to reviewers?}

To quantify how much a given criterion matters, we compute the Shapley value of each criterion alongside another method that provides similar results. The Shapley value tells us how informative the criterion score is for predicting the overall score. In detail, to compute the value $v$ of a subset $S \subseteq [\numcriteria]$, we fit a function of the subset of criteria on the training data. The value of the subset, denoted by $v(S)$, is the negative prediction error of that function on the test set. Given a set of $\numcriteria$ criterion, the Shapley value for criterion, $s_i$, $i$ is
\[s_i = \sum_{S \subset [d]\setminus \{i\}}\frac{|S|!(d-|S|-1)!}{d!}\left[v(S \cup \{i\})- v(S)\right].\]
A high Shapley value indicates that a given criterion is highly informative for predicting a paper's overall score. Conversely, a lower value indicates that the given criterion is less informative, for instance, if that criterion is highly correlated with many other criteria. The sum of the Shapley values is the total predictive power the criteria have for overall paper ratings. Namely, it is the difference in error between always predicting the mean rating and using a function of the criteria scores. Thus, we can interpret the fraction $s_i/\sum_{i \in [\numcriteria]}s_i$ as the percent of predictive power that comes from the $i^{th}$ criterion.~\\ 

\noindent \textit{Results}. As shown in \Cref{fig:criteria_importance}, we find that contribution is the largest driver of the overall rating, followed by soundness. Presentation consistently comes last. \textbf{Indeed, nearly half (47\%) of the predictive power of the criteria scores comes from contribution alone.} 

\begin{figure}[tbp]
    \centering
    \includegraphics[width=0.45\linewidth]{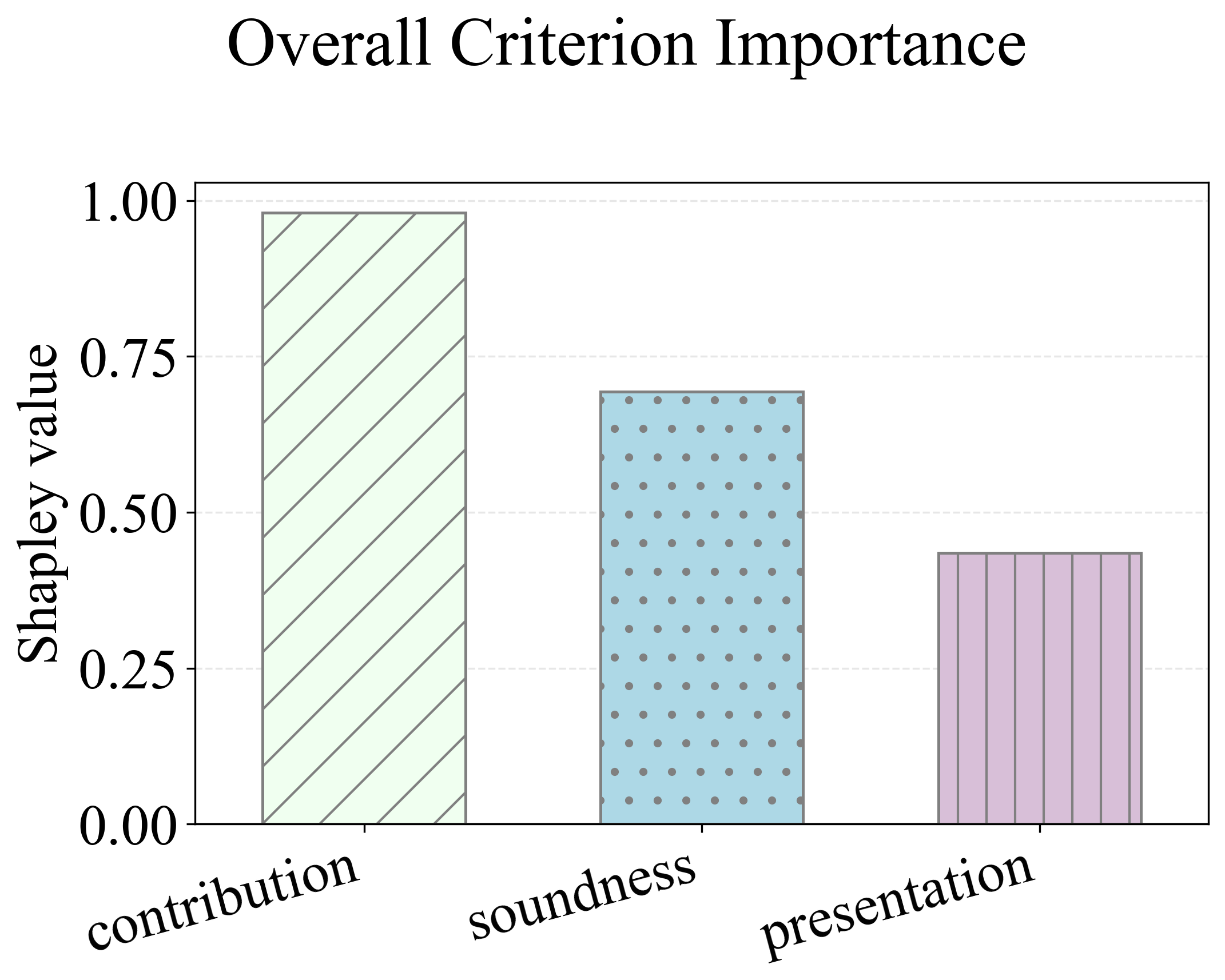}
    \caption{Criterion importance on the full ICLR 2026 data as measured by the Shapley value of out-of-sample predictive contribution. We find similar results for ICLR 2025.}
    \label{fig:criteria_importance}
\end{figure}

\subsubsection{Does the value of improving one criterion depend on the other criteria?}\label{sec:criteria_interaction}
Instead of computing an overall marginal impact weight, we ask how the marginal effect of a criterion varies based on the values of the other criteria. In the graph below, we measure the average effect of a $+1$ increase in the criterion from $1$ to $2$, $2$ to $3$, and $3$ to $4$ while keeping the values of the remaining criteria fixed.~\\

\noindent \textit{Results}. Our results are shown in \Cref{fig:criterion_marginal_effect_heatmap}. We find that reviewers grade on a curve that gets steeper as submissions get better. In the learned community preference function, when the other criteria are rated 1 or 2, a ``+1'' improvement on one criterion increases the overall rating by only about 0.4 points. By contrast, when the other criteria are rated 3 or 4, a ``+1'' improvement increases the overall rating by triple the amount, about 1.2 points. Hence, our analysis indicates that effects of criteria interact with one another: the benefit of a higher criterion score is magnified when other aspects of the paper are higher-quality. 

\begin{figure}[tbp]
    \centering
    \includegraphics[width=0.5\linewidth]{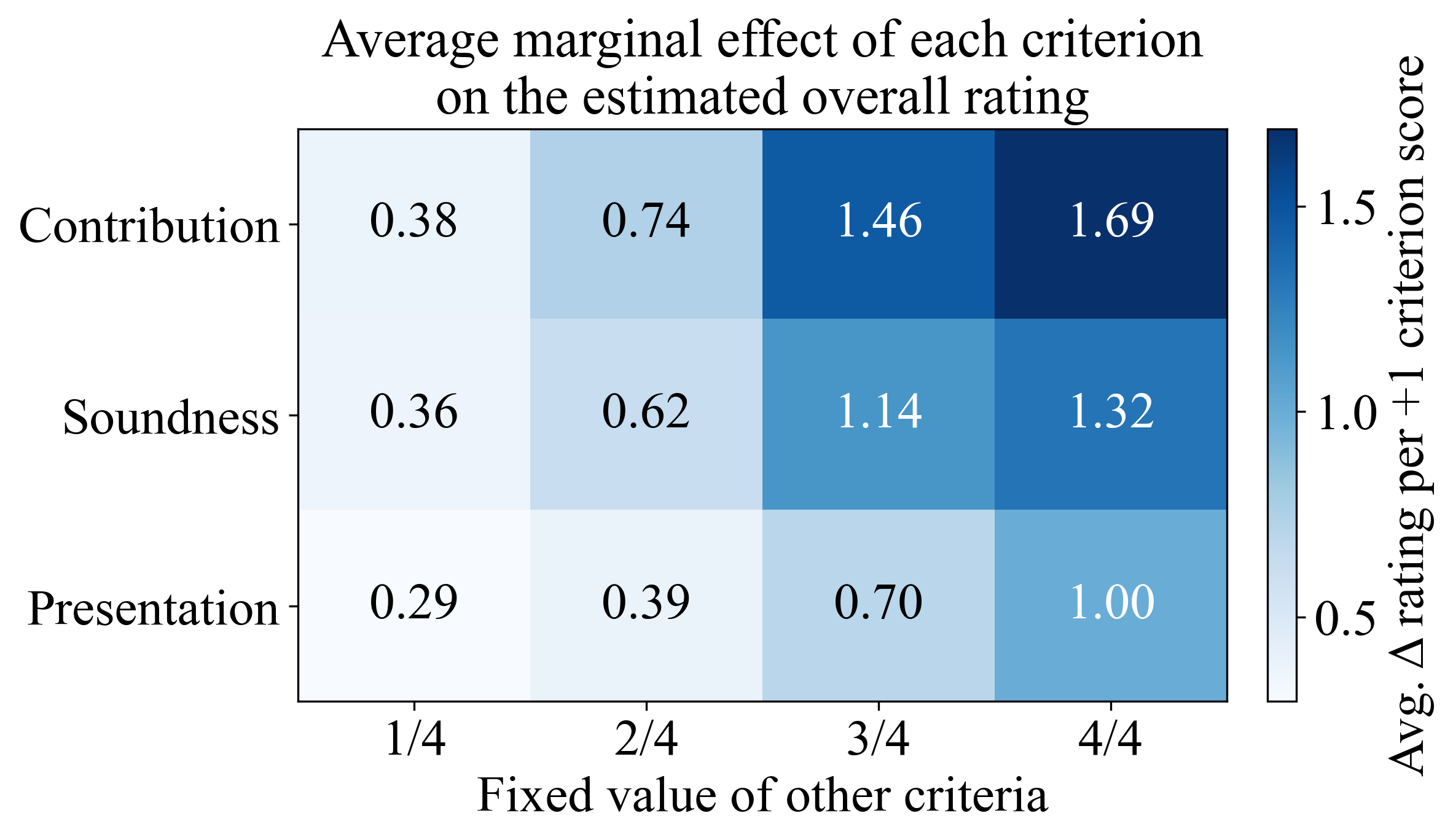}
    \caption{Average change in the overall rating per $+1$ on each criterion, as a function of where the other criteria sit. The same improvement is worth about $3$ times more on an already-strong submission.}
    \label{fig:criterion_marginal_effect_heatmap}
\end{figure}

\subsubsection{Is the aggregation just a weighted sum of individual criteria?}
A weighted sum would correspond to a linear (affine) mapping from criteria scores to the overall rating. Recall that our algorithm does not fit the isotonic function on its own: its objective adds a penalty that pulls the function toward the best-fitting linear model, controlled by a weight (a regularization parameter). A weight of zero lets the function be fully flexible, subject only to monotonicity; an infinite weight forces it to be exactly linear; intermediate weights interpolate between the two. If the model is truly linear, we expect an infinite weight to do (weakly) better than the others, and the weight of zero to do the worst, as any nonlinearities in the estimator do not exist in the actual reviewer preferences.~\\ 

\noindent \textit{Results}. Cross-validation selects a weight of zero (the most flexible, least-linear end). Moreover, the fully linear model (infinite weight) has the worst held-out error of every weight tried. These pieces of evidence suggest that reviewers’ aggregation is reliably not a weighted sum. \Cref{sec:criteria_interaction} shows a specific way in which reviewer preferences depart from linearity; the effect of an increase in criterion scores depends (quite substantially) on the other criterion scores.

\subsection{Commensuration bias}
To start, we will go into the problem of preference variability or commensuration bias in peer review in more detail before introducing how we help to alleviate this issue. To understand why this form of bias can be problematic, consider the fact that each paper's evaluation (and thus its accept/reject decision) depends upon the preferences the reviewer applies. If these preferences vary based on the reviewer and context, it can result in different overall evaluations for the paper, and, ultimately, acceptance decisions. \textit{As such, the acceptance of a paper is dependent upon a function of the particular reviewer that they are assigned.} This issue is magnified by the fact that many papers are clustered around the reject/accept border (see \Cref{sec:ICLR_data_info}); as such, very small differences in preferences across reviews can lead to a large number of swaps in paper decisions. Given these issues, in the past, program chairs have contacted area chairs handling reviews with large commensuration biases so that they can keep a closer watch on those reviews. 

Concretely, we define commensuration bias to be the absolute difference between a review’s overall rating and the rating implied by that review’s own criteria scores under the community preference function. The implied rating is a continuous value (it is not restricted to the $\{0,2,4,6,8,10\}$ grid), so this difference can take any value in the range $[0,10]$. Hence for review $i \in [\numsamp]$, criterion scores and overall evaluation $(\xval_i, \yval_i)$, and fitted function $\hatf$ on the full review dataset:
\[\text{Commensuration bias}_i = |\hatf(\xval_i) - \yval_i|. \]
We provide a per-review file with this value for every review, sorted in decreasing order available \href{https://github.com/madelineceli13/evaluator-preference-algorithm/blob/main/data/case_study/debiased_reviews_2026.csv}{here}. 

\subsubsection{How much commensuration bias is there?}
To understand the extent of bias, we look at the distribution of commensuration bias across the set of reviews, presented in \Cref{fig:commensuration_bias}. This helps to give us a sense of the average amount of bias in a typical review, and how much it varies across reviews. Recall that on ICLR's 0-10 scale, one rating level is $2$ points. The median bias is about $0.8$ points on the $0-10$ scale, and the mean is similar ($0.9$), corresponding to about half of a rating level. However, $1\%$ of reviews have a bias of more than 1.5 levels (3 points), and $0.1\%$ differ by more than two levels (4 points). 

\begin{figure}[tbp]
    \centering
    \includegraphics[width=0.5\linewidth]{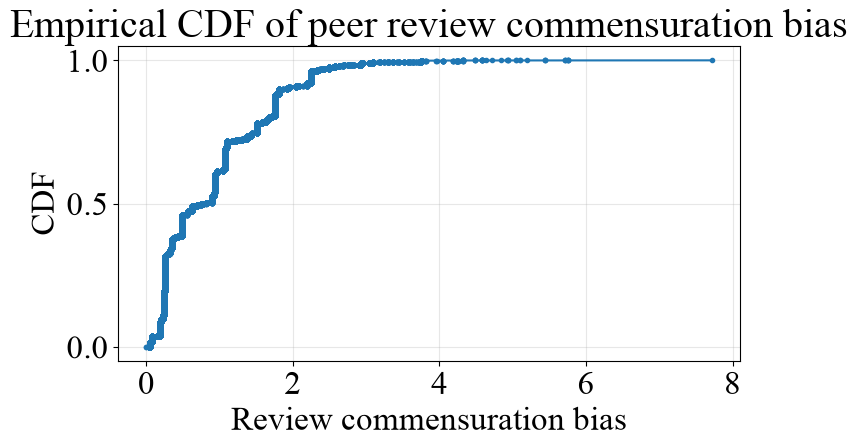}
    \caption{Empirical CDF of commensuration bias across all $75,859$ ICLR 2026 reviews}
    \label{fig:commensuration_bias}
\end{figure}

\subsubsection{To what extent does commensuration bias impact paper decisions?}
Finally, we analyze the potential real-world effects of commensuration bias on which papers are accepted or rejected. To model this, we follow the approach in \citet{noothigattu_loss_2021}: we consider a simplified review process where the papers with the top mean overall rating are accepted. The decision threshold is chosen to reproduce the observed acceptance rate in the ICLR review process (29\%). We compare this simplified review process with an unbiased one; instead of using the provided overall rating, we use the rating implied by that review’s own criteria scores under the community preference function.

\textit{Under this analysis, $16\%$ of paper accept/reject decisions would be altered in the absence of commensuration bias.} Admittedly, real decisions are made by area chairs rather than by a strict mean threshold, so this is a calibrated approximation of the decision boundary, not the exact committee process. Moreover, there can be valid reasons a reviewer deviates from the community preference function, such as a consideration not captured in the decision-making criterion. Nevertheless, the scale of the number of papers for which their outcome is changed underscores the fact that preference variability across reviews is likely to have significant stakes in peer review processes. 

\section{Conclusion}\label{sec:conclusion}
The problem of learning evaluator preferences is important for applications from peer review to business to medicine. Many researchers assume (extensions of) linear preferences, and we prove that when this assumption does not hold, there can be significant consequences for many outcomes practitioners care about: learning preferences, ranking individuals or items, and comparing criteria importance. We focus specifically on one of the most frequent assumptions---linearity---and ask if there is a learning algorithm that can do better. To address this, we develop an algorithm that can learn any type of isotonic preference function and performs provably just as well as linear regression (up to constant factors). To conclude, we remark on a few directions of future research:~\\ 
\begin{itemize}[itemsep=0.5em]
    \item \textit{Improving predictive power by including non-quantitative features}. One of our findings was that there is a substantial gap between prediction and reducible error on our real-world datasets, implying that a large amount of variability cannot be explained by any function of the criteria scores. Therefore, to meaningfully decrease prediction error, we may need to include more information. One source of additional information is non-quantitative inputs (i.e., free-form review summaries). In Tripadvisor ratings, users provide a summary of their travel experience alongside a title describing their say. In peer review, reviewers often provide a paragraph summary of the paper. How can we systematically leverage this and other similar free-form language data to improve predictive accuracy?  
    \item \textit{Interpretability of the preference function}. Even if we learn the best possible (collective) preference function, additional work is needed to make it as useful as possible for practitioners. In many settings, it is unrealistic to provide non-technical practitioners with the preference function and expect them to trust its outputs, especially if it is being used in moral decision-making settings. Thus, it may be advantageous to develop accessible methods that convey properties of the function. Some of the work in \Cref{sec:iclr_case_study} addresses this, for instance, computing Shapley values to assess (relative) criterion significance. Yet many questions remain unanswered. In economics, it is often of interest to understand whether aspects or goods are complements or substitutes, meaning that having more of one increases the value of the other. Is there a way to measure the complementarity of criteria? Furthermore, when the preference function is nonlinear, how can we convey the dominant types of nonlinearity? 
    \item \textit{Using our methods to understand the evaluation preferences of LLMs}. LLMs are being increasingly used in practice to provide preliminary assessments for applications such as medical care, hiring, and peer review. Our empirical work suggests key aspects, such as preference alignment and consistency, may vary widely depending on the specific LLM. However, additional work is needed to understand how robust our results are to the choice of statistic and the determinants of misalignment (i.e., why and how preferences may differ).~\\
\end{itemize}

In sum, we hope that our approach can help researchers and practitioners to learn preferences more reliably and faithfully, improving outcomes in practical settings, and inspiring related directions for future research.

\section*{Acknowledgments} 
This work was funded through grants NSF  1942124 and ONR N000142512346. We thank R. Ravi for very helpful discussions. We thank Jacqueline Lane for thoughtful discussions and for letting us run the algorithms on a dataset of aerospace and robotics design evaluations.

\bibliography{references}

\appendix

~\\~\\\noindent \textbf{\Large Appendices}
\section{Supplementary material for Section \ref{sec:problems_with_linearity}}\label{appendix:problems_with_linearity}
\subsection{Proof of Proposition \ref{prop:bias_genint_functions}}\label{app:estimation_functions}

For our proof, it is convenient to express the preference function $\scoref: [\numvals]^{\numcriteria} \to [0,1]$ by a tensor $\Tst \in [0,1]^{\numvals^{\numcriteria}}$ defined such that for all $(i_1,\dots, i_\numcriteria) \in \calX$, $\Tst_{i_1,\dots, i_\numcriteria}=\scoref(i_1,\dots, i_\numcriteria)$ and then to prove the analogous result for the tensor versions of the estimator and true preference function. Let $\Ciso^\numcriteria$ be the set of isotonic tensors representing all functions in $\Fiso$ and $\Cgenint^\numcriteria$ be the set of tensors representing all functions in $\Fint$. In particular,  
\[\Ciso^\numcriteria := \{\tensor \in [0,1]^{\numvals^\numcriteria}|\tensor_{i_1,\dots, i_\numcriteria} \ge \tensor_{i_1',\dots, i_\numcriteria'} \text{ whenever }i_k \ge i_k' \forall k \in [\numcriteria]\}.\]
As for the definition of $\Fint$, we define $\Cgenint^\numcriteria$ to be the set of tensors that can be induced by a specific equation. For tensors the analog of the equation for functions in $\Fint$, Equation \Cref{eq:genint_class}, replaces the set of non-decreasing transformations over inputs with a set of non-decreasing vectors. 

Let the set of transformations on the outputs be captured by the set of vectors $\{a^{(k)}\}_{k = 1}^{\numcriteria}$ with increasing coordinates. As for the definition of $\Fint$ let the coefficients of the criteria interactions be denoted by $\{\gamma_S\}_{S \in \mathcal{P}([\numcriteria]}$ and $\genlinf$ be a non-decreasing transformation. We say that $\tensor$ is in a member of tensor analog of $\Fint$, denoted as $\Cint^\numcriteria$, if for some $\{a^{(k)}\}_{k = 1}^{\numcriteria}$, $\{\gamma_S\}_{S \in \mathcal{P}([\numcriteria]}$, and $\genlinf$:

\begin{equation}\label{eq:tensor_lin_int}
    \tensor_{i_1, i_2, \dots, i_\numcriteria}=\genlinf \left(\sum_{S \in \mathcal{P}([d])}\gamma_S\prod_{k \in S}a^{(k)}_{i_k}\right) \text{ for all pairs }(i_1,i_2, \dots, i_\numcriteria) \in [\numvals]^{\numcriteria}.
\end{equation}

By construction, for any two functions $\scoref, \hatf$ mapping $[\numvals]^{\numcriteria} \to [0,1]$ with tensor representations $\Tst, \tensor$ where $P$ is the uniform distribution, $\frac{1}{\numvals^\numcriteria}\|\tensor-\Tst\|_F^2 =\|\scoref - \hatf\|_{L_2(P)}^2$ where $P$ is the uniform distribution. Using this fact and our definitions of the relevant classes of tensors are defined, we can note that \Cref{prop:bias_genint_functions} is equivalent to the \Cref{prop:bias_genint_tensors}. 
\begin{proposition}\label{prop:bias_genint_tensors}
    There exists some universal constant $\const$ such that for all $\numvals \ge 4$, $\numcriteria \ge 2$, there exists a tensor $\Tst \in \Ciso^\numcriteria$ such that
    \[\frac{1}{\numvals^\numcriteria}\inf_{\tensor \in \Cgenint^\numcriteria}\|\tensor-\Tst\|_F^2 \ge \const.\]
\end{proposition}
\begin{proof}
We will prove this by constructing an explicit $\Tst$ such that the proposition holds. To do this, we will first solve the case when $\numcriteria = 2$ and then extend the result to cover the case when $\numcriteria \ge 3$. 

When $\numcriteria =2$, $\scoref \in \Fiso$ and $\hatf \in \Fint$ can be expressed as matrices $\Mst \in \Ciso^2$ and $\mat \in \Cgenint^2$, respectively. To construct such a counterexample, we will show that there is a set of inequalities over entries that $\Mst$ can satisfy but $\mat$ cannot. This implies that over that subset of indices, $\mat$ must incur some amount of error. We then construct a matrix such that the set of those inequalities has size of order $\numvals^\numcriteria$ to prove our bias result. The specific ordering and proof of the impossibility for any $\mat \in \Cgenint^2$ is the subject of \Cref{lemma:ordering_interactions} which we will prove at the end of this section.
\begin{lemma}\label{lemma:ordering_interactions}
    Assume that $\mat \in \Cgenint(\genlinf)$ for some non-decreasing $\genlinf$ and $\numvals=4$. Then the following inequalities cannot all hold:
    \begin{align}
        & \mat_{12} < \mat_{21} < \mat_{31} < \mat_{22}\\
        & \mat_{22} < \mat_{13} < \mat_{23} < \mat_{32}\\
        & \mat_{31} < \mat_{22} < \mat_{32} < \mat_{41}\\
        & \mat_{23} < \mat_{32} < \mat_{42} < \mat_{33}
    \end{align}
\end{lemma}
Taking \Cref{lemma:ordering_interactions} as given, we will construct a general $\numvals \times \numvals$ matrix $\Mst$ and showing that it satisfies 
\[\frac{1}{\numvals^2}\inf_{\mat \in \Cgenint^2}\|\mat-\Mst\|_F^2 \ge \const\]
for some universal constant $\const>0$. We will do this by using a $4 \times 4$ matrix, $\Ex$, defined as follows
\[\Ex = \frac{1}{12}\begin{pmatrix} 
1 & 2 & 6 & 12 \\
3 & 5 & 7 & 12 \\
4 & 8 & 11 & 12 \\
9 & 10 & 12 & 12
\end{pmatrix}\]
Note that $\Ex \in \Cbiso$. Additionally, it satisfies the inequalities in \Cref{lemma:ordering_interactions}. These inequalities for our specific choice of $\Ex$ evaluate to  
\begin{align*}
        & \Ex_{12} < \Ex_{21} < \Ex_{31} < \Ex_{22} & \implies 2/12 < 3/12 < 4/12 < 5/12 \\
        & \Ex_{22} < \Ex_{13} < \Ex_{23} < \Ex_{32} & \implies  5/12 < 6/12 < 7/12 < 8/12 \\
        & \Ex_{31} < \Ex_{22} < \Ex_{32} < \Ex_{41} & \implies 4/12 < 5/12 < 8/12 < 9/12 \\
        & \Ex_{23} < \Ex_{32} < \Ex_{42} < \Ex_{33} & \implies 7/12 < 8/12 < 10/12 < 11/12
\end{align*}
Define $\Mst$ to have entries $\Mst_{ij} = \Ex_{\lceil (4i)/\numvals \rceil, \lceil (4j)/\numvals \rceil}$ for all $(i,j) \in [\numvals]^2$ and assume that $\numvals/4$ is an integer (we will later show that this is without loss of generality). Therefore, we can write $\Mst$ as
\[\Mst = \begin{pmatrix}
A_{11} & A_{12} & A_{13} & A_{14}\\
A_{21} & A_{22} & A_{23} & A_{24}\\
A_{31} & A_{32} & A_{33} & A_{34}\\
A_{41} & A_{42} & A_{43} & A_{44}\\
\end{pmatrix}\]
where $A_{kl}$ for $(k,l) \in [4]^2$ are $n/4 \times n/4$ constant matrices where each entry is equal to $\Ex_{kl}$. We can partition $\Mst$ into $\numvals^2/16$ sub-matrices that are all equal to $\Ex$. For instance, let each sub-matrix be defined by having entry $i,j$ of each of the blocks $A_{11}, \dots, A_{44}$ (this means that the sub-matrix has rows $\{i,n/4 +i, 2n/4 +i, 3n/4+i\}$ and columns $\{j,n/4 +j, 2n/4 +j, 3n/4+j\}$). 

Let $\mat \in \Cgenint^2$ be any estimate for $\Mst$. By construction, any sub-matrix of a $\mat$ must also be in $\Cgenint^2$. Let $\Ex'$ be the estimate of $\Ex$ for any of the sub-matrices defined above. By \Cref{lemma:ordering_interactions}, it must be the case that one of the four inequalities does not hold on $\Ex'$. 

This implies that at least one of the ordinal relations between the entries of $\Ex$ must be flipped or hold with equality. Since each relevant entry (those in the first 3 columns) is at least $1/12$ apart, we have that $\|\Ex-\Ex'\|_F^2 \ge 2 \cdot \frac{1}{24^2}=\frac{1}{288}$. This is because squared error is minimized when the change is the same for both entries contained in the ordinal relation. Since the total change is at least $1/12$, they individually must change at least by $1/24$. 

Summing over the $\numvals^2/16$ sub-matrices equal to $\Ex$, we have that $\|\Mst - \mat\|_F^2 \ge \numvals^2/(16\cdot 288)$. As this inequality holds for all $\genlinf,a,b$ we have that for the $\Mst$ defined above 
\[\frac{1}{\numvals^2}\inf_{\mat \in \Cgenint^2}\|\Mst - \mat\|_F^2\ge \numvals^2/(16\cdot 288)\]
Note that the assumption that $\numvals/4$ is an integer is without loss of generality. If $\numvals/4$ was not an integer, we could make the same argument for the first $4\lfloor \numvals/4\rfloor$ rows and columns (corresponding to more than a quarter of all entries) and fill the remaining rows and columns with all $1$. This would give us error at least $\frac{1}{4}\numvals^2/(16\cdot 288)$. Setting $\const = \frac{1}{4}/(16\cdot 288)$, then the above inequality holds when $\numcriteria=2$ and $\numvals \ge 4$. 

In order to extend our result to cases when $\numcriteria>2$, we will use \Cref{lemma:genint_tensors} which tells us that when fixing $\numcriteria-2$ coordinates for a $\tensor \in \Cgenint^\numcriteria$, the resulting matrix must be in $\Cgenint^2$. This allows us to extend the bias result for $\numcriteria=2$ to the case when $\numcriteria>2$. 

\begin{lemma}\label{lemma:genint_tensors}
    Assume that $\tensor \in \Cgenint^\numcriteria$ and $\numcriteria \ge 3$. Then the matrix $\mat$ formed by fixing any $\numcriteria -2$ indices is in $\Cgenint^2$. 
\end{lemma}

Let $\Mst$ be a $\numvals \times \numvals$ matrix such \Cref{prop:bias_genint_tensors} holds (we have just proved the existence of one). Set $\Tst_{i_1,i_2,\dots, i_\numcriteria}= \Mst_{i_1,i_2}$ for any $(i_1,\dots, i_\numcriteria) \in [\numvals]^\numcriteria$ so that it is constant with respect to all but the first two indices. Consider the approximation of $\Tst$ by any $\tensor \in \Cgenint(\genlinf)$ for some non-decreasing $\genlinf$. For simplicity of notation, we will use $\tensor_{.,i_{3:\numcriteria}}$ to denote the matrix form by fixing the last $\numcriteria-2$ indices of the tensor. By the definition of the Frobenius norm and our construction of $\Tst$:

\begin{align*}
    \|\tensor - \Tst\|_F^2 & = \sum_{i_1=1}^{\numvals}\sum_{i_2=1}^{\numvals}\cdots \sum_{i_\numcriteria=1}^{\numvals}(\Tst_{i_1, i_2,\dots, i_\numcriteria} -\tensor_{i_1, i_2,\dots, i_\numcriteria})^2\\
    & = \sum_{i_3=1}^{\numvals}\sum_{i_4=1}^{\numvals}\cdots \sum_{i_\numcriteria=1}^{\numvals}(\Mst -\tensor_{.,i_{3:\numcriteria}})^2
\end{align*}

By our choice of $\Mst$ we have that for some universal constant $\const>0$, $\inf_{\mat \in \Cgenint^2}\|\Mst - \mat\|_F^2 \ge \const \numvals^2$. Thus, multiplying this by all $\numvals^{\numcriteria-2}$ different pairs $(i_3, \dots, i_\numcriteria) \in [\numvals]^{\numcriteria-2}$, we have our claimed result 
\[ \frac{1}{\numvals^\numcriteria}\|\tensor - \Tst\|_F^2 \ge \const.\]
\end{proof}

Now all that is left is to prove the remaining lemmas, \Cref{lemma:ordering_interactions} and \Cref{lemma:genint_tensors}. 
\begin{proof}[Proof of \Cref{lemma:ordering_interactions}]
Our goal is to show that the series of inequalities is not compatible with any set of parameters and non-decreasing function $\genlinf$. In this proof we will leverage the fact that the ordering over the entries has to be determined by the relations of the inputs to $\genlinf$. From the first two inequalities we learn something about the relative values of $(\gamma_1+\gamma_{12}b_1)$, $(\gamma_1+\gamma_{12}b_2)$, $(\gamma_1+\gamma_{12}b_2)$ such both inequalities can hold at the same time. The second two inequalities are constructed to be of the same form, but provide different requirements on the relative values of $\gamma_1+\gamma_{12}b_j$ for $j \in \{1,2,3\}$. We can show that these requirements are not compatible with a single set of parameters, allowing us to reach a contradiction.

Assume for the sake of contradiction that $\mat \in \Cgenint$ for some $\genlinf,a,b,c,\gamma_1, \gamma_2, \gamma_{12}$. Because our argument pertains to the ordinal relationship between entries, we can assume WLOG that $c=0$. 

The first inequality implies that for some $\genlinf,a,b, \gamma_1, \gamma_2, \gamma_{12}$

\begin{align*}
    \genlinf(\gamma_1 a_1+ \gamma_2 b_2+\gamma_{12} a_1b_2) & < \genlinf(\gamma_1 a_2+ \gamma_2 b_1+\gamma_{12} a_2b_1)\\
    & < \genlinf(\gamma_{1}a_3+ \gamma_{2}b_1+\gamma_{12} a_3b_1)\\
    & < \genlinf(\gamma_1 a_2+ \gamma_2 b_2+\gamma_{12} a_2b_2).
\end{align*}
As $\genlinf$ is non-decreasing, the ordinal relationship of the inputs is the same as the outputs. Consequently,
\begin{align*}
    \gamma_1 a_1+ \gamma_2 b_2+\gamma_{12} a_1b_2  & \gamma_1 a_2+ \gamma_2 b_1+\gamma_{12} a_2b_1\\
    & < \gamma_{1}a_3+ \gamma_{2}b_1+\gamma_{12} a_3b_1\\
    & < \gamma_1 a_2+ \gamma_2 b_2+\gamma_{12} a_2b_2.
\end{align*}
Since the difference between the outer two terms is larger than that between the inner two terms we have that
\begin{align*}
    & \left(\gamma_1 a_2+ \gamma_2 b_2+\gamma_{12} a_2b_2\right) - \left(\gamma_1 a_1+ \gamma_2 b_2+\gamma_{12} a_1b_2\right)\\
    & > \left(\gamma_{1}a_3+ \gamma_{2}b_1+\gamma_{12} a_3b_1\right) - \left(\gamma_1 a_2+ \gamma_2 b_1+\gamma_{12} a_2b_1\right).
\end{align*}
Rearranging, we have that 
\[(\gamma_1+\gamma_{12} b_2)(a_2-a_1) > (\gamma_1 +\gamma_{12}b_1)(a_3-a_2).\]
Following the same argument the remaining three inequalities imply:
\begin{align*}
    & (\gamma_{1}+\gamma_{12} b_3)(a_2-a_1)<(\gamma_{1}+\gamma_{12} b_2)(a_3-a_2)\\
    & (\gamma_{1}+\gamma_{12} b_1)(a_4-a_3)>(\gamma_{1}+\gamma_{12} b_2)(a_3-a_2)\\
    & (\gamma_{1}+\gamma_{12} b_2)(a_4-a_3)<(\gamma_{1}+\gamma_{12} b_3)(a_3-a_2)\\
\end{align*}
There are two cases. 

\textbf{Case 1:} Both $(\gamma_1 + \gamma_{12} b_3)$ and $(\gamma_1 + \gamma_{12} b_2)$ are non-zero. This allows us to combine the first two inequalities to get 
\[\frac{\gamma_1 + \gamma_{12} b_1}{\gamma_1 + \gamma_{12} b_2}(a_3-a_2)<a_2-a_1 < \frac{\gamma_1 + \gamma_{12} b_2}{\gamma_1 + \gamma_{12} b_3}(a_3-a_2)\]
Because $a_3 -a_2$ is non-negative by assumption, in order for this to hold, we then require that $\frac{\gamma_1+\gamma_{12} b_2}{\gamma_1+\gamma_{12} b_3}> \frac{\gamma_1+\gamma_{12} b_1}{\gamma_1+\gamma_{12} b_2}$. Likewise, we can derive from the last two inequalities that 
\[\frac{\gamma_1+\gamma_{12} b_2}{\gamma_1+\gamma_{12} b_3}(a_4-a_3)<a_3-a_2 < \frac{\gamma_1+\gamma_{12} b_1}{\gamma_1+\gamma_{12} b_2}(a_4-a_3)\]
which require that $\frac{\gamma_1+\gamma_{12} b_2}{\gamma_1+\gamma_{12} b_3}< \frac{\gamma_1+\gamma_{12} b_1}{\gamma_1+\gamma_{12} b_2}$, a contradiction. 

\textbf{Case 2}: If one of $(\gamma_1 + \gamma_{12} b_3)$ and $(\gamma_1 + \gamma_{12} b_2)$ are zero, we will show this implies many of the entries are tied. If $(\gamma_1 + \gamma_{12} b_j)=0$ for $j \in \{2,3\}$, we will show that the $j^{th}$ column of $\mat$ must be constant. For any $i \in [\numvals]$: 
\begin{align*}
    \mat_{ij} & = \genlinf(\gamma_1 a_i + \gamma_2 b_j + \gamma_{12}a_ib_j) \\
    & =  \genlinf((\gamma_1 +\gamma_{12}b_j)a_i + \gamma_2 b_j) \\
    & = \genlinf(\gamma_2 b_j)
\end{align*}
Since the value of $\mat_{ij}$ does not depend on $i$, this would imply that the jth column is constant. For $j=3$, we would have $\mat_{23} = \mat_{33}$, and for $j=2$ we would have $\mat_{12}=\mat_{22}$. Either case gives us a contradiction since those inequalities are strict in $\mat$. 

Since both cases lead to a contradiction, the four inequalities cannot hold simultaneously for some $\mat \in \Cgenint^2$. 

\begin{proof}[Proof of \texorpdfstring{\Cref{lemma:genint_tensors}}{Lemma \ref{sec:genint_tensors}}]
    By assumption $\tensor \in \Cgenint(\genlinf)$ for some function $\genlinf$. Therefore, $\tensor$ is induced by Equation \Cref{eq:tensor_lin_int} for some scalars $\{\gamma_S\}_{S \in \mathcal{P}([\numcriteria])}$, non-decreasing vectors $a^{(1)}, \dots, a^{(\numcriteria)} \in \R^\numvals$, and non-decreasing function $\genlinf$. 
    
    WLOG we will fix $i_3, \dots, i_{\numcriteria}$ (the same arguments hold for fixing any $\numcriteria-2$ values of the criteria scores). Consider the matrix $\mat = \tensor_{.,i_{3:\numcriteria}} \in \R^{\numvals \times \numvals}$. Observe that we can re-write the input for $\genlinf$ for an arbitrary set of indices $(i_1, i_2) \in [\numvals]^2$ as: 
    
    \begin{align*}
        \sum_{S \in \mathcal{P}([d])}\gamma_S\prod_{k \in S}a^{(k)}_{i_k}  & = \sum_{S \in \mathcal{P}([d]): |S\cap \{1,2\}|=0}\gamma_S\prod_{k \in S}a^{(k)}_{i_k} + \sum_{S \in \mathcal{P}([d]): 1 \in S, 2 \notin S}\gamma_S\prod_{k \in S}a^{(k)}_{i_k} \\
        & + \sum_{S \in \mathcal{P}([d]): 1 \notin S, 2 \in S}\gamma_S\prod_{k \in S}a^{(k)}_{i_k} + \sum_{S \in \mathcal{P}([d]): 1 \in S, 2 \in S}\gamma_S\prod_{k \in S}a^{(k)}_{i_k}\\
        & = \left(\sum_{S \in \mathcal{P}([d]): |S\cap \{1,2\}|=0}\gamma_S\prod_{k \in S}a^{(k)}_{i_k}\right) + a^{(1)}_{i_1}\left(\sum_{S \in \mathcal{P}([d]): 1 \in S, 2 \notin S}\gamma_S\prod_{k \in S\setminus \{1\}}a^{(k)}_{i_k}\right)\\
        & + a^{(2)}_{i_2}\left(\sum_{S \in \mathcal{P}([d]): 1 \notin S, 2 \in S}\gamma_S\prod_{k \in S\setminus \{2\}}a^{(k)}_{i_k}\right)\\
        & + a^{(1)}_{i_1}a^{(2)}_{i_2}\left(\sum_{S \in \mathcal{P}([d]): 1 \in S, 2 \in S}\gamma_S\prod_{k \in S\setminus \{1,2\}}a^{(k)}_{i_k}\right)
    \end{align*}
Note that all components in parenthesis are constant with respect to $i_1$ and $i_2$. For each element $S$ of the power set $\{1,2,(1,2), 0\}$ we see that there is a constant that multiplies $\prod_{k \in S}a_{i_k}^{(k)}$. Set $a = a^{(1)}$, $b = b^{(2)}$ and the constants $c, \gamma_1, \gamma_2, \gamma_{12}$ to be the values in parenthesis 
\begin{align*}
    & c = \sum_{S \in \mathcal{P}([d]): |S\cap \{1,2\}|=0}\gamma_S\prod_{k \in S}a^{(k)}_{i_k} \\
    & \gamma_{1} = \sum_{S \in \mathcal{P}([d]): 1 \in S, 2 \notin S}\gamma_S\prod_{k \in S\setminus \{1\}}a^{(k)}_{i_k} \\
    & \gamma_{2} = \sum_{S \in \mathcal{P}([d]): 1 \notin S, 2 \in S}\gamma_S\prod_{k \in S\setminus \{2\}}a^{(k)}_{i_k}\\
    & \gamma_{12} = \sum_{S \in \mathcal{P}([d]): 1 \in S, 2 \in S}\gamma_S\prod_{k \in S\setminus \{1,2\}}a^{(k)}_{i_k}.
\end{align*}
Then is it easy to see that for all $(i_1, i_2) \in [\numvals]^2$, \[\mat_{i_1, i_2} = \tensor_{i_1,i_2,\dots, i_d} = \genlinf(\gamma_1 a_{i_1} + \gamma_2b_{i_2} + \gamma_{12} a_{i_1}b_{i_2}  + c).\]
Hence $\Mst \in \Cgenint(\genlinf)$ as desired.  
\end{proof}
\end{proof}

\subsection{Proof of Proposition \ref{prop:bias_ranking_genlin}}\label{app:ranking}

Our goal is to show that there exists a preference function $\scoref \in \Fiso$ such that when approximated by any function $\hatf \in \Fgen$, there is a worst-case Kendall-tau error. As in the proof of \Cref{prop:bias_genint_functions} we will use the same tensor representation of the preference function $\scoref$ and prove the analogous result for the tensor representations of any estimator $\hatf \in \Fgen$ and a worst-case function $\scoref \in \Fiso$. Recall that the class $\Ciso^\numcriteria$ represents all functions in $\Fiso$. Further, we can express all functions in $\Fgen$ by the set of tensors. For notational simplicity we will represent indices both as the list of their components (for instance, $(i_1, \dots, i_\numcriteria)$) and simply the vector $i$. We say that a tensor corresponds to an $\scoref \in \Fgen$, denoted by $\tensor \in \Cgen^\numcriteria$, if there exists a set of vectors $\{a^{(k)}\}_{k=1}^{\numcriteria} \subset \R^\numvals$ (encoding the coordinate-wise transformations), a non-decreasing $\genlinf$, and a $c \in \R$ such that for all $i \in \calX$:
\begin{equation}\label{eq:gam_tensor}
    \tensor_{i} = \genlinf\left(\sum_{k \in [\numvals]}a_{i_k}^{(k)}+c\right)
\end{equation}

We will confirm that any $\scoref \in \Fgen$ has a corresponding $\tensor \in \Cgen^\numcriteria$. For any $\scoref \in \Fgen$ expressed by Equation \Cref{eq:genlin}, this is accomplished by setting $\genlinf=\genlinf$, $b=c$, and for all $k \in [\numcriteria$ and $x_k \in \numvals$, $a^{(k)}_{x_k}=a_kh_k(x_k)$. Therefore, the set of permutations $\perm_\scoref$ for $\scoref \in \Fgen$ are contained in those of $\perm_\Tst$ for $\Tst \in \Cgen^\numcriteria$. Let the corresponding tensors of $\scoref$ and $\hatf$ be $\Tst$ and $\tensor$, respectively. Then 
\begin{align*}
    \ktdist(\perm_{\hatf}, \perm_\scoref) & = \ktdist(\perm_\tensor, \perm_\Tst)\\
    & = \frac{1}{\binom{|\calX|}{2}}\sum_{\text{unique }i \neq i' \in \calX}\ktcomp(i,i',\perm_\tensor, \perm_\Tst).
\end{align*}
Therefore, since $\Cgen$ contains the tensor representations of all $\Fgen$, the statement of \Cref{prop:bias_ranking_genlin} is implied by the following proposition 
\begin{proposition}\label{prop:ranking_bias_genlin_tensors}
    There exists some universal constant $\const > 0$ such that for all $\numvals \ge 3$ and $\numcriteria \ge 2$, there exists a $\Tst \in \Ciso^\numcriteria$ such that 
    \[\frac{1}{\binom{|\calX|}{2}}\inf_{\mat \in \Cgen^\numcriteria}\ktdist (\perm_{\Tst}, \perm_{\tensor}) \ge \const.\]
\end{proposition}
\begin{proof}[Proof of \Cref{prop:ranking_bias_genlin_tensors}]
    We will prove this by constructing an explicit $\Tst$ that satisfies the conditions of \Cref{prop:ranking_bias_genlin_tensors}. Let's assume that $\numvals/3$ is an integer and use the approximation of  $\binom{|\calX|}{2}$ by $\numvals^{2\numcriteria}$ (this is WLOG up to constant factors). As in the proof of \Cref{prop:bias_genint_tensors}, we will start with the case when $\numcriteria=2$ and then extend our construction to cases when $\numcriteria>2$. At a high-level, our goal is to show that there are a series of inequalities that cannot be satisfied by any $\tensor \in \Cgen^\numcriteria$, and to show that the number of such inequalities between pairs of entries leads to a worst-case Kendall tau distance. Since there are a number of lemma's involved in this proof, we will leave the proof of each lemma to the end of the section. 

    When $\numcriteria=2$, we can represent $\scoref$ and $\hatf$ by matrices $\scoref \in \Ciso^2$ and $\hatf \in \Cgen^2$, respectively. First we show that such a set of inequalities over entries exists in a $3\times 3$ matrix. 
    \begin{lemma}\label{lemma:ordering_linear}
    Assume that $\mat \in \Cgen(\genlinf)$ for some $\genlinf$ and that $\numvals=3$. Then the following inequalities cannot both hold:
    \[\mat_{12} < \mat_{21}< \mat_{31} < \mat_{22}\]
    and 
    \[\mat_{22} < \mat_{13}< \mat_{23} < \mat_{32}.\]
    \end{lemma}

    Let $\Ex = \frac{1}{10}\begin{pmatrix} 1 & 2 &6\\
    3 & 5 & 7 \\
    4 & 8 & 9
    \end{pmatrix}$ be a $3 \times 3$ bivariate isotonic matrix that satisfies the ordinal relations in  \Cref{lemma:ordering_linear}. We will prove our result for a matrix $\Mst$ which can be written as 
    \[\Mst = \begin{pmatrix}A_{11} & A_{12} & A_{13}\\
    A_{21} & A_{22} & A_{23}\\
    A_{31} & A_{32} & A_{33}
    \end{pmatrix} + \begin{pmatrix}E & E & E\\
    E & E & E\\
   E & E & E
    \end{pmatrix}\]
    where $A_{kl}$ for $(k,l) \in [3]^2$ are $n/3 \times \numvals/3$ constant matrices where each entry is equal to $\Ex_{kl}$. We can partition $\Mst$ into $\numvals^2/9$ sub-matrices that are all equal to $\Ex$ and $E \in \R^{\numvals/3 \times \numvals/3}$ and entries $E_{ij} = \frac{1}{10\numvals}i + \frac{1}{10\numvals^2}j$. Here the additional terms $E$ create ordering within each of the blocks and ensure that each entry is unique. Explicitly, $\Mst_{ij} = \Ex_{\lceil (3i)/\numvals\rceil, \lceil (3j)/\numvals\rceil} + \epsilon_{ij}$ where
    \[\eps_{ij} = \frac{1}{10\numvals^2}\left(i-\frac{n}{3}(\lceil ni/3 \rceil -1)\right) + \frac{1}{10\numvals}\left(j-\frac{n}{3}(\lceil \numvals j/3 \rceil -1)\right).\]
    We will now verify that $\Mst \in \Cbiso$ by confirming that it is isotonic and has entries between $0$ and $1$. First we will show that it is isotonic. Each submatrix $A_{kl}$ is isotonic by construction. Let $(i,j)$ and $(i',j')$ be entries from different sub-matrices such that $i\ge i'$ and $j \ge j'$ where are least one inequality holds strictly. Then 
    \[\Mst_{ij}-\Mst_{i'j'} \ge \frac{1}{10}-\frac{\numvals}{3}\left(\frac{1}{10\numvals^2}+\frac{1}{10\numvals}\right)>\frac{7}{120}.\]
    By construction, $\Mst_{ij}\ge \Mst_{11} \ge 0.1 >0$. Additionally, 
    \[\Mst_{ij} \le\Mst_{nn} = 0.9 + \frac{1}{10\numvals^2}(\numvals/3) + \frac{1}{10\numvals}(\numvals/3) < 0.9 + \frac{7}{120}<1.\]
Hence $\Mst \in \Cbiso$.  We will now lower-bound the Kendall tau distance from $\Mst$ to any matrix $\mat \in \Cgen(\genlinf)$ for some non-decreasing $\genlinf$. 

We can partition $\Mst$ into $L=\numvals^2/9$ sub-matrices with ordering the same as $\Ex$. Let each sub-matrix be defined rows $\{i,\numvals/3 +i, 2\numvals/3 +i\}$ and columns $\{j,\numvals/3 +j, 2\numvals/3 +j\}$. Let $S_1,\dots, S_L$ denote the set of indices that correspond to each sub-matrix. We can rearrange the terms in the Kendall tau distance to get:
\begin{align*}
    \ktdist(\perm_{\Mst}, \perm_{\mat}) & = \sum_{\text{unique }(i,j),(i',j') \in [\numvals]^2}\overline{\ktdist}_{(i,j), (i',j')}(\perm_{\mat}, \perm_\Mst)\\
    & = \sum_{\subindex =1}^{L}\sum_{(i,j)\neq (i',j') \in S_\subindex}\overline{\ktdist}_{(i,j), (i',j')}(\perm_{\mat}, \perm_\Mst)\\
    & + \sum_{\subindex =2}^{L}\sum_{\subindex =1}^{\subindex-1}\sum_{(i,j) \in S_\subindex, (i',j') \in S_{\subindex'}}\overline{\ktdist}_{(i,j), (i',j')}(\perm_{\mat}, \perm_\Mst)
\end{align*}

Define $\ktdist_\subindex$ to be the distance within the sub-matrix $\subindex \in [\numindex]$, and $\ktdist_{\subindex, \subindex'}$ to be the distance between the permutations corresponding to sub-matrices $\subindex$ and $\subindex '  \in [\numindex]$. 
\begin{align*}
    & \ktdist_{\subindex, \subindex'}= \sum_{(i,j) \in S_\subindex, (i',j') \in S_{\subindex'}}\overline{\ktdist}_{(i,j), (i',j')}(\perm_{\mat}, \perm_\Mst)\\
    & \ktdist_{\subindex} = \sum_{(i,j) \neq (i',j') \in S_{\subindex}}\overline{\ktdist}_{(i,j), (i',j')}(\perm_{\mat}, \perm_\Mst)
\end{align*}
The remainder of this proof will be devoted to showing that $\ktdist_{\subindex, \subindex'}\ge 1$ and $\ktdist_\subindex \ge 1$. If this holds when we have 
\[\ktdist(\perm_{\Mst}, \perm_{\mat}) \ge \numindex + \binom{\numindex}{2} = \numvals^2/9 + \frac{\numvals^2/9(\numvals^2/9-1)}{2} \ge \numvals^4/162.\]

\textbf{Case 1}: No ties and incorrect ordering. If there exists some pair of entries $(i,j) \in S_\subindex$ and $(i',j') \in S_{\subindex'}$ where the relative order in $\mat$ is different from in $\Mst$, then we have $\ktdist_{\subindex, \subindex'} \ge 1$. Likewise, if there exists some pair of entries $(i,j) \neq  (i',j') \in S_{\subindex}$ where the relative order in $\mat$ is different from in $\Mst$, then we have $\ktdist_{\subindex} \ge 1$. 

\textbf{Case 2}: If all entries have the same ordering up to ties, we now need to show that there are at least two ties between non-equal entries. First we will consider $\ktdist_{\subindex, \subindex'}$. Each submatrix can be expressed by $\Ex + E_{ij}$ for some $(i,j) \in [\numvals/3]^2$. By construction, the values of $E_{ij}$ are unique. Therefore, $\Mst_{S_\subindex} > \Mst_{S_{\subindex'}}$ or $\Mst_{S_\subindex} < \Mst_{S_{\subindex'}}$ depending on the relative value of $E_{ij}$ for which the submatrices are defined. WLOG assume $\Mst_{S_\subindex} > \Mst_{S_{\subindex'}}$. In order for there to be zero error, it must be true that for the approximation $\mat$: 
    \[\mat_{S_{\subindex'}, 12} \le \mat_{S_{\subindex}, 12} \le \mat_{S_{\subindex'}, 21} \le \mat_{S_{\subindex}, 21}\le  \mat_{S_{\subindex'}, 31} \le \mat_{S_{\subindex}, 31} \le \mat_{S_{\subindex'}, 22} \le \mat_{S_{\subindex}, 22}\]
    and 
    \[\mat_{S_{\subindex'}, 22} \le \mat_{S_{\subindex}, 22} \le \mat_{S_{\subindex'}, 13} \le \mat_{S_{\subindex}, 13} \le \mat_{S_{\subindex'}, 32} \le \mat_{S_{\subindex}, 32}.\]
    
It is easy to show that any sub-matrix of a matrix $\mat \in \Cgen(\genlinf)$ is also in $\Cgen(\genlinf)$. This implies that $\mat_{S_{\subindex'}}, \mat_{S_{\subindex}} \in \Cgen(\genlinf)$ satisfy the ordinal relations in  \Cref{lemma:ordering_linear} with weak inequalities. 
\begin{lemma}\label{lemma:ties}
        If $\mat \in \Cgen(\genlinf)$ for some non-decreasing $\genlinf$ satisfies the ordinal relations in  \Cref{lemma:ordering_linear} with weak inequalities, then $\mat_{12} = \mat_{21}$,  $\mat_{31} = \mat_{22}=\mat_{13}$ and $\mat_{23} = \mat_{32}$.  
\end{lemma}
Taking this lemma as given, we then have that $\mat_{S_{\subindex'},31} = \mat_{S_{\subindex'},22}=\mat_{S_{\subindex'},13}$ and $\mat_{S_{\subindex},31} = \mat_{S_{\subindex},22}=\mat_{S_{\subindex},13}$. This implies that 
\[\mat_{S_{\subindex'}, 31} = \mat_{S_{\subindex}, 31} = \mat_{S_{\subindex'}, 22} =\mat_{S_{\subindex}, 22}=\mat_{S_{\subindex}, 13} = \mat_{S_{\subindex'}, 32}\]
which corresponds to a six non-equal entries being tied including $\mat_{S_{\subindex'}, 31} = \mat_{S_{\subindex}, 31}$ and $\mat_{S_{\subindex'}, 22} =\mat_{S_{\subindex}, 22}$. Hence $\ktdist_{\subindex, \subindex'}\ge 1$. 

Likewise, if $\mat_{S_\subindex}$ has the same ordering as $\Mst_{S_\subindex}\in \Cgen(\genlinf)$ up to ties, it follows that it must satisfy the relations in  \Cref{lemma:ordering_linear} with weak inequalities. This implies that $\mat_{S_{\subindex},31} = \mat_{S_{\subindex},22}$ and $\mat_{S_{\subindex},12} = \mat_{S_{\subindex},21}$ corresponding to two ties between non-equal entries in $\Mst$. Thus $\ktdist_\subindex \ge 1$. Since we have show that $\ktdist_{\subindex, \subindex'}\ge 1$ and $\ktdist_\subindex \ge 1$, this concludes our proof that
\[\ktdist(\perm_{\Mst}, \perm_{\mat}) \ge L+ \binom{L}{2} = \numvals^2/9 + \frac{\numvals^2/9(\numvals^2/9-1)}{2} \ge \numvals^4/162,\]
showing that when $\numcriteria=2$, \Cref{prop:ranking_bias_genlin_tensors} holds. 

Now we will use the $\Mst$ constructed in the proof for the case when $\numcriteria=2$ to create a more general counterexample for the case when $\numcriteria>2$.

In the proof of \Cref{prop:bias_genint_tensors}, we pick $\Tst$ to be constant with respect to the last $\numcriteria-2$ indices. Hence, $\Tst_{.,i_{3:\numcriteria}}=\Mst$ where $\Mst$ is chosen to be a block matrix of the $3 \times 3$ matrix $\Ex$ whose relative ordinal relations cannot be satisfied by a generalized additive model. Let $A_{ij} \in [0,1]^{\numvals/3 \times \numvals/3}$ be the constant matrix taking on the value of $\Ex_{ij}$ for $(i,j) \in [3]^2$. Further, let $\otimes$ denote the tensor product. We constructed
    \[\Tst = \begin{pmatrix}A_{11} & A_{12} & A_{13} \\
    A_{21} & A_{22} & A_{23} \\
    A_{31} & A_{32} & A_{33}
    \end{pmatrix} \otimes 1^{\otimes^{d-2}}.\]
    This worked, because fixing $\numcriteria - 2$ of the last entries, we had $\Omega( \numvals^2)$ Frobenius norm error and, summing over all such pairs, we had $\Omega( \numvals^\numcriteria)$ error. However, we cannot use this same approach for ranking error. If we summed the Kendall tau distance separately over each pair $(i_3, \dots, i_\numcriteria) \in [\numvals]^{\numcriteria-2}$, since the error for each pair is $\mathcal{O}( \numvals^4)$, we would have total error of $\mathcal{O}(  \numvals^{\numcriteria + 2})$ which for $\numcriteria > 2$ is asymptotically less than $\numvals^{2\numcriteria}$ (the order of the Kendall tau distance). 

Therefore, we instead need to compare across different values of $(i_3, \dots, i_\numcriteria) \in [\numvals]^{\numcriteria-2}$. To do this, and to ensure that we know the unique ground truth ranking of $\Tst$ we construct a special $\numvals/3 \times \numvals/ 3 \times \numvals^{\numcriteria-2}$ tensor of perturbations, $E$, to add to each of the nine $\numvals /3 \times \numvals / 3$ blocks. Thus, now we can write 

\[\Tst = \begin{pmatrix}A_{11} & A_{12} & A_{13} \\
    A_{21} & A_{22} & A_{23} \\
    A_{31} & A_{32} & A_{33}
    \end{pmatrix} \otimes 1^{\otimes^{d-2}} + \begin{pmatrix}E & E & E \\
    E & E & E \\
    E & E & E
\end{pmatrix}.\]
    To construct $E$, we will use the following lemma: 
\begin{lemma}\label{lemma:E_tensor}
    There exists a tensor $\widetilde{E} \in [0,1]^{\numvals^\numcriteria}$ that is isotonic and has unique entries. 
\end{lemma}
Taking \Cref{lemma:E_tensor} as given for now, let $E$ be a normalized $\widetilde{E}$ such that it has values between $[0,0.1)$ and contains only the first $\numvals / 3$ values of $i_1$ and $i_2$. In particular, for all $\widetilde{i} \in [\numvals/3]^2 \times [\numvals]^{\numcriteria-2}$:
\[E_{\widetilde{i}} = \frac{1}{11}\widetilde{E}_{\widetilde{i}}.\]
Now, we will re-write this constructed $\Tst$ in an equivalent way which will help us in the proof. To do this, we will define a function $g$ to map entries to the correct block matrix, and a function $h$ to map entries to the component within the block matrix. In particular, we set $g(i) = \lceil 3i/n\rceil$ to partition the $(i_1,i_2)$ pairs into the respective block (i.e. value of $\Ex_{g(i_1), g(i_2)}$) and $h(i) = i - n/3(\lceil 3i/n\rceil-1)$ to map the an index $i$ to the associated entry within the block (i.e., perturbation value $E_{h(i_1), h(i_2), i_3, \dots, i_d}$).  For all $i \in [\numvals]^{\numcriteria}$:
\[\Tst_{i_1, i_2, \dots, i_\numcriteria} = \Ex_{g(i_1), g(i_2)} + E_{h(i_1), h(i_2), i_3, \dots, i_d}.\]
We will confirm that $\Tst$ is in $\Ciso^\numcriteria$.  First, it must have entries in $[0,1]$. $\Ex$ and $E$ are non-negative so $\Tst \ge 0$. Since $\Ex$ has entries between $0.1$ and $0.9$ and $E$ has entries in $[0,0.1)$:
\[\max_{i \in [\numvals]^\numcriteria}\Tst_{i} \le \max_{i \in [\numvals]^\numcriteria}\Ex_{g(i_1), g(i_2)} + \max_{i \in [\numvals]^\numcriteria}E_{h(i_1), h(i_2), i_3, \dots, i_d} < 0.9 + 0.1= 1.\]
Second, $\Tst$ must be isotonic. Consider a non-equal set of indices $i=i_1, i_2, \dots, i_\numcriteria$ and $i'=i_1', i_2', \dots, i_\numcriteria'$ such that $i' \ge i$. If $g(i_1)=g(i_1'), g(i_2)=g(i_2')$, then 
\[\Tst_{i'}-\Tst_{i} =E_{h(i_1'), h(i_2'), i_3', \dots, i_d'} - E_{h(i_1), h(i_2), i_3, \dots, i_d}.\]
Since $E$ is isotonic, and $h$ is increasing in $i_1, i_2$ given $g(i_1)$ and $g(i_2)$, this is strictly positive. The second case is that $g(i_1')\neq g(i_1)$ and/or $g(i_2')\neq g(i_2)$. Since $g$ is weakly increasing, it must be the case that $g(i_1')> g(i_1)$ or $g(i_2')> g(i_2)$. Thus $i$ is mapped to different entry of $\Ex$ which, by construction is at least $0.1$ larger. Using the fact that entries in $E$ are less than $0.1$ apart:
\[\Tst_{i'}-\Tst_{i} > (\Ex_{g(i_1'), g(i_2')}-\Ex_{g(i_1), g(i_2)}) -0.1  =0.\]
Thus, $\Tst \in \Ciso^\numcriteria$ as desired. 

Now, we want to lower-bound the Kendall tau distance between $\Tst$ and its approximation $\tensor \in \Cgen^{\numcriteria}(\genlinf)$ for some non-decreasing $\genlinf$. Consider the partition of $\Tst$ based on the value of $E$ an index maps to (ie the values of $h(i_1), h(i_2), i_3, \dots, i_d$). There are $L = \frac{\numvals^\numcriteria}{9}$ total partitions. We will index these by $\subindex \in [\numindex]$ and use $S_\subindex$ to refer to the set of indices that partition $\subindex$ contains and $\tensor_{S_\subindex} \in \R^{3 \times 3}$ to refer to the sub-matrix formed by just those elements. Note that
\begin{align*}
    \ktdist(\perm_{\Tst}, \perm_{\tensor}) & = \sum_{\text{unique }i,i' \in [\numvals]^{\numcriteria}}\ktcomp(i,i',\perm_{\Tst},\perm_{\tensor})\\
    & = \sum_{\subindex \in [\numindex]}\sum_{i\neq i' \in S_\subindex}\ktcomp(i,i',\perm_{\Tst},\perm_{\tensor})  + \sum_{\subindex \neq \subindex' \in [\numindex]}\sum_{i\neq S_\subindex, i' \in S_{\subindex'}}\ktcomp(i,i',\perm_{\Tst},\perm_{\tensor})
\end{align*}

Let $\ktdist_{\subindex} = \sum_{i\neq i' \in S_\subindex}\ktcomp(i,i',\perm_{\Tst},\perm_{\tensor})$ and $\ktdist_{\subindex, \subindex'} = \sum_{i\neq S_\subindex, i' \in S_{\subindex'}}\ktcomp(i,i',\perm_{\Tst},\perm_{\tensor})$. Note that if for all $\subindex, \subindex' \neq \subindex \in [\numindex]$, $\ktdist_{\subindex} \ge 1$ and $\ktdist_{\subindex, \subindex'} \ge 1$ then we would have 
\[\ktdist(\perm_{\Tst}, \perm_{\tensor}) \ge \numindex + \binom{\numindex}{2} \ge \numindex^2 / 2 = \numvals^{2\numcriteria}/162.\]
Thus, it suffices to show that $\ktdist_{\subindex} \ge 1$ and $\ktdist_{\subindex, \subindex'} \ge 1$. There are two cases. 

\textbf{Case 1}: No ties and an incorrect ordering. First consider what would happen for $\ktdist_\subindex$. This says that there are two indices $i \neq i' \in S_\subindex$ such that the relative order in $\Tst$ is different from in $\tensor$ this implies that $\ktcomp(i,i',\perm_{\Tst}, \perm_{\tensor}) =1$ and hence that $\ktdist_\subindex \ge 1$. The same argument applies for $\ktdist_{\subindex, \subindex'}$: we have some $i \in S_\subindex$ and $i' \in S_{\subindex'}$ where the relative order in  $\Tst$ is different from in $\tensor$ which implies $\ktcomp(i,i',\perm_{\Tst}, \perm_{\tensor}) =1$ and hence that $\ktdist_{\subindex,\subindex'} \ge 1$. Now we are done with case 1. 

\textbf{Case 2}: Same ordering up to ties. First consider $\ktdist_\subindex$. Let $E_\subindex$ be the value of $E$ that this partition corresponds to. Observe that $\Tst_{S_\subindex}=\Ex + E_\subindex$ by construction so that the ranking of entries in $\Tst_{S_\subindex}$ is the same as the ranking over entries in $\Ex$. This implies that $\tensor_{S_\subindex}$ has the same ranking as $\Ex$ up to ties and consequently satisfies that relations in \Cref{lemma:ordering_linear} with weak inequalities. 

Note that all indices $i \in S_\subindex$ have the same values of $i_3,\dots, i_\numcriteria$. Thus $\tensor_{S_\subindex}$ is a sub-matrix of $\tensor_{.i_{3:\numcriteria}}$ for some values $i_3,\dots, i_\numcriteria$ associated with the rows $\{h(i_1), \numvals/3 + h(i_1), 2\numvals/3 + h(i_1)\}$ and columns $\{h(i_2), \numvals/3 + h(i_2), 2\numvals/3 + h(i_2)\}$. Since sub-matrices of matrices in $\Cgen^2$ are also in $\Cgen^2$, $\tensor_{S_\subindex} \in \Cgen^2$ if $\tensor_{.i_{3:\numcriteria}} \in \Cgen^2$. We will now show that this is true. Because $\tensor \in \Cgen^\numcriteria(\genlinf)$, there exists some $a^{(1)},\dots, a^{(\numcriteria)}$ such that for all $i_1,i_2 \in [\numcriteria]^2$ 
\[\tensor_{i_1i_2\dots i_\numcriteria} = \genlinf(\sum_{k \in [d]}a_{i_k}^{(k)}+c) = \genlinf(a_{i_1}^{(1)}+a_{i_2}^{(2)}+\sum_{k =3}^{d}a_{i_k}^{(k)}+c).\]
Setting $b \in \R^n$ to have entries $a^{(2)}_{i_2}+\sum_{k =3}^{d}a_{i_k}^{(k)}+c$ for all $i_2 \in [\numvals]$ and $a^{(1)} = a^{(1)}$ we see that $\tensor_{.i_{3:\numcriteria}} \in \Cgen^2$ for some the same $\genlinf$, $a^{(1)}$ and $a^{(2)}=b$. As our argument does not depend on the value of $\subindex$, $\tensor_{S_\subindex} \in \Cgen^2$ for all $\subindex \in [\numindex]$. 

By \Cref{lemma:ties}, this implies that $\tensor_{S_\subindex, 12} = \tensor_{S_\subindex,21}$, $\tensor_{S_\subindex,31} = \tensor_{S_\subindex,22}=\tensor_{S_\subindex,13}$, and $\tensor_{S_\subindex,23} = \tensor_{S_\subindex,32}$. Since none of these entries are equal in $\Tst$ and each tie in $\tensor$ for non-equal entries in $\Tst$ adds $1/2$ to the Kendall tau distance, it follows that $\ktdist_{\subindex}\ge 1$. 

Now, consider $\ktdist_{\subindex, \subindex'}$. Note that $\Tst_{S_\subindex} = \Ex + E_\subindex$ and $\Tst_{S_{\subindex'}} = \Ex + E_{\subindex'}$. Since the entries of $E$ are unique, it follows that either $\Tst_{S_{\subindex}} > \Tst_{S_{\subindex'}} $ or $\Tst_{S_{\subindex}} < \Tst_{S_{\subindex'}}$. WLOG we will assume that $\Tst_{S_{\subindex}} > \Tst_{S_{\subindex'}}$. If all relations between entries in $S_\subindex$ and $S_{\subindex'}$ are satisfied up to ties, then: 
\[\tensor_{S_{\subindex'}, 12} \le \tensor_{S_{\subindex}, 12} \le \tensor_{S_{\subindex'}, 21} \le \tensor_{S_{\subindex}, 21}\le  \tensor_{S_{\subindex'}, 31} \le \tensor_{S_{\subindex}, 31} \le \tensor_{S_{\subindex'}, 22} \le \tensor_{S_{\subindex}, 22}\]
    and 
\[\tensor_{S_{\subindex'}, 22} \le \tensor_{S_{\subindex}, 22} \le \tensor_{S_{\subindex'}, 13} \le \tensor_{S_{\subindex}, 13} \le \tensor_{S_{\subindex'}, 32} \le \tensor_{S_{\subindex}, 32}.\]
    
This implies that $\tensor_{S_{\subindex'}}, \tensor_{S_{\subindex}} \in \Cgen(\genlinf)$ satisfy the ordinal relations in  \Cref{lemma:ordering_linear} with weak inequalities. By \Cref{lemma:ties},  $\tensor_{S_\subindex,31} = \tensor_{S_\subindex,22}=\tensor_{S_\subindex,13}$ and $\tensor_{S_{\subindex'},31} = \tensor_{S_{\subindex'},22}=\tensor_{S_{\subindex'},13}$. Taken together with the inequalities above:
\[\tensor_{S_{\subindex'},31} = \tensor_{S_\subindex,31} = \tensor_{S_{\subindex'},22} = \tensor_{S_\subindex,22} = \tensor_{S_{\subindex'},13} = \tensor_{S_\subindex,13}\]
which corresponds to six non-equal entries being tied including more than two pairs of entries where one is in $S_\subindex$ and one is in $S_{\subindex'}$. Since each tie adds $1/2$ to $\ktdist_{\subindex, \subindex'}$: $\ktdist_{\subindex, \subindex'} \ge 1$. This concludes the 2nd case. Since we have shown that for all $\subindex, \subindex' \neq \subindex \in [\numindex]$, $\ktdist_{\subindex} \ge 1$ and $\ktdist_{\subindex, \subindex'} \ge 1$, we have 
\[\ktdist(\perm_{\Tst}, \perm_{\tensor}) \ge \numindex + \binom{\numindex}{2} \ge \numindex^2 / 2 = \numvals^{2\numcriteria}/162.\]
Since $\binom{|\calX|}{2}$ is order $\numvals^{2\numcriteria}$ we have shown that the statement of \Cref{prop:ranking_bias_genlin_tensors} holds when $\numcriteria=2$ and $\numcriteria>2$, we have now completed the proof of \Cref{prop:ranking_bias_genlin_tensors}. All that is left is to prove the auxiliary lemmas. 

\begin{proof}[Proof of \Cref{lemma:ordering_linear}]
    Assume for the sake of contradiction that $\mat \in \Cgen(\genlinf)$ for some $\genlinf, a, b,c$. Without loss of generality, we can assume that $c=0$. This is because we can set $\widetilde{b} = b_j +c$ such that for all $(i,j)\in [\numvals]^2$, $f(a_i+\widetilde{b}_j)=f(a_i+b_j+c)$. Therefore, $\mat \in \Cgen(\genlinf)$ for some $\genlinf, a, b,c$ implies that $\mat \in \Cgen(\genlinf)$ for some $\genlinf, a, b$. 
    
    Since $\mat_{ij}=\genlinf(a_i + b_j)$, the first inequality implies:
    \[\genlinf(a_1 + b_2) < \genlinf(a_2 + b_1) < \genlinf(a_3 + b_1) < \genlinf(a_2 + b_2).\]
    Since $\genlinf$ is non-decreasing, the ordering of the entries is the same as the ordering of the underlying linear model. Therefore, the inequalities in terms of the outputs of $\genlinf$ hold for the inputs. Hence, 
    \[a_1 + b_2 < a_2 + b_1 < a_3 + b_1 < a_2 + b_2.\]
    Lastly, since the difference between the first and last terms ($a_1 + b_2$ and $a_2 + b_2$) is greater than the difference between the middle two terms  ($a_2 + b_1$ and $a_3 + b_1$) we have that
    \[a_2 - a_1 > a_3 -a_2.\]
    We will now make the same argument for the second inequality. The second inequality implies that 
    \[\genlinf(a_2 + b_2) < \genlinf(a_1 + b_3) < \genlinf(a_2 + b_3) < \genlinf(a_3 + b_2).\]
    Following the same argument, we get that
    \[a_2 - a_1 < a_3 -a_2\]
    which is a contradiction. Therefore, both inequalities cannot hold at the same time. 
\end{proof}

\begin{proof}[Proof of \Cref{lemma:ties}]
        By assumption:
    \[\mat_{12} \le \mat_{21}\le  \mat_{31} \le \mat_{22}\]
    and 
    \[\mat_{22} \le \mat_{13} \le \mat_{23} \le \mat_{32}.\]
    Since $\mat_{ij}=\genlinf(a_i + b_j)$, the first inequality implies:
    \[\genlinf(a_1 + b_2) \le \genlinf(a_2 + b_1) \le \genlinf(a_3 + b_1) \le \genlinf(a_2 + b_2).\]
    Since $\genlinf$ is non-decreasing, the ordering of the entries is the same as the ordering of the underlying linear model. Therefore, the inequalities in terms of the outputs of $\genlinf$ hold for the inputs. Hence, 
    \[a_1 + b_2 \le a_2 + b_1 \le a_3 + b_1 \le a_2 + b_2.\]
    Lastly, since the difference between the first and last terms ($a_1 + b_2$ and $a_2 + b_2$) is weakly greater than the difference between the middle two terms  ($a_2 + b_1$ and $a_3 + b_1$) we have that
    \[a_2 - a_1 \ge a_3 -a_2.\]
    Following the same argument for the second inequality, we get that
    \[a_2 - a_1 \le a_3 -a_2\]
    In order for both to hold and be consistent with an underlying linear model we require that $(a_2 - a_1) = (a_3 - a_2)$. If the first equation holds with equality, then we have that the difference between ($a_1 + b_2$ and $a_2 + b_2$) is equal to the difference between ($a_2 + b_1$ and $a_3 + b_1$). This is only possible if $a_1 + b_2 = a_2 + b_1$ and $a_3 + b_1 = a_2 + b_2$ which implies that $\mat_{12} = \mat_{21}$ and $\mat_{31} = \mat_{22}$. Following the same argument, if second equation holds with equality this implies that $\mat_{22} = \mat_{13}$ and $\mat_{23} = \mat_{32}$. Consequently, we have that $\mat_{12} = \mat_{21}$,  $\mat_{31} = \mat_{22}=\mat_{13}$ and $\mat_{23} = \mat_{32}$. 
\end{proof}

\begin{proof}[Proof of \Cref{lemma:E_tensor}]
   First, we will construct a tensor that has unique values. For all $i \in [\numvals]^\numcriteria$, let 
   \[\tensor_{i} = \sum_{k=1}^{\numcriteria}\numvals^{k-1}i_k = i_1 + \numvals i_2 + \dots + \numvals^{\numcriteria-1}i_\numcriteria \]
   Since all indices $i_k$ are between $1$ and $\numvals$, each entry of $\tensor$ can be represented by a unique value (base $\numvals$) and thus is unique. Since the value of $\tensor_{i}$ is increasing in the value of each index, $\tensor$ is isotonic. Further, we have that for all $i$, $\tensor_{i}\le \numcriteria \numvals^{\numcriteria}$. Dividing each entry by $\numcriteria \numvals^{\numcriteria}$, we have the desired claim. 
\end{proof}
\end{proof}

\subsection{Measuring criteria importance}
\subsection{Proof of Proposition \ref{prop:criteria_importance}}\label{app:criteria_importance}

Our proof relies on the following result. 
\begin{lemma}\label{lemma:ols_simple}
Consider the setting of a function $\scoref: \calX \to \R$ fit by a linear model of the form, $\scoref(\xval_1,\xval_2) = a_1\xval_1 + a_2\xval_2$ for some $(a_1,a_2) \in \R^2$. Given samples $\{\xval^{(j)}\}_{\subindex=1}^{\numsamp}$ consider the least squares minimization problem:
\[\min_{(a_1,a_2)}\sum_{j = 1}^{\numsamp}(\scoref(\xval^{(j)})-a_1\xval_1^{(j)} - a_2\xval_2^{(j)})^2.\]
Let $\numsamp_{\xval_1,\xval_2}$ for $\xval \in \{0,1\}^2$ be the number of observations corresponding to each input pair. Let $\alpha_{\xval_1,\xval_2}=\scoref(\xval)$ the function value for those inputs. If at least two of $\numsamp_{1,1}, \numsamp_{0,1}, \numsamp_{1,0}$ are strictly positive, then the solutions are:

\begin{align*}
    & a_1 = \frac{(\numsamp_{0,1} + \numsamp_{1,1})(\numsamp_{1,0}\alpha_{1,0} + \numsamp_{1,1}\alpha_{1,1}) -\numsamp_{1,1}(\numsamp_{0,1}\alpha_{0,1} + \numsamp_{1,1}\alpha_{1,1})}{\numsamp_{1,1}(\numsamp_{0,1} + \numsamp_{1,0}) +\numsamp_{1,0}\numsamp_{0,1}} \\
    & a_2 = \frac{-\numsamp_{1,1}(\numsamp_{1,0}\alpha_{1,0} + \numsamp_{1,1}\alpha_{1,1}) +(\numsamp_{1,1} + \numsamp_{1,0})(\numsamp_{0,1}\alpha_{0,1} + \numsamp_{1,1}\alpha_{1,1})}{\numsamp_{1,1}(\numsamp_{0,1} + \numsamp_{1,0}) +\numsamp_{1,0}\numsamp_{0,1}}
\end{align*}
\end{lemma}
Let us take this lemma as given for now (we will prove it at the end of this section). 

\paragraph{Proof of part (a)}: Consider the case when input pairs for all but $(0,1)$ are observed. Set $\alpha_{1,0}=\frac{1}{4}$, $\alpha_{0,1}=0$, and $\alpha_{1,1}=1$.The denominator of the expressions for $a_1$ and $a_2$ are equal and strictly positive. Thus, it suffices to show that the difference of the numerators is negative. Then taking the difference of the numerators we get that
\[a_1 < a_2 \iff \numsamp_{1,0}\numsamp_{1,1}(2\alpha_{1,0}-\alpha_{1,1})<0.\]
Plugging in $\alpha_{1,0}=\frac{1}{4}$ and $\alpha_{1,1}=1$, one can see that this inequality is satisfied.

\paragraph{Proof of part (b)}: Assume that there is at least one sample of each criteria pair $\xval \in \{0,1\}^2$. Consider the case when $\alpha_{0,1}=\alpha_{1,0}=\alpha \in [0,\frac{1}{2})$, and $\numsamp_{1,0}>\numsamp_{0,1}$. Additionally, assume that $\alpha_{0,0}=0$ and $\alpha_{1,1}=1$. As for the proof of part (a), the denominator of the expressions for $a_1$ and $a_2$ are equal and strictly positive, so it suffices to show that the difference of the numerators is negative. Substituting in the values for $\alpha_{1,1}$ we get that 
\[a_1 < a_2 \iff \numsamp_{1,1}(1-2\alpha)(\numsamp_{0,1}-\numsamp_{1,0})<0.\]
Since $\alpha < \frac{1}{2}$, $(1-2\alpha)>0$. As $\numsamp_{1,0}> \numsamp_{0,1}$, $(\numsamp_{0,1}-\numsamp_{1,0})<0$. Hence the product is negative and $a_1 < a_2$ as claimed.

All that is left is to prove \Cref{lemma:ols_simple}. 
\begin{proof}[Proof of \Cref{lemma:ols_simple}]
    Let $\sqloss(a_1,a_2)$ be the total squared error for any $a_1,a_2$ pair. Plugging in the values of $\scoref$ for the specific inputs we have:
    \begin{align*}
        \sqloss(a_1,a_2) &= \sum_{j \in [\numsamp]}\left(\scoref(\xval_1^{(j)}, \xval_2^{(j)})-a_1\xval_1^{(j)} - a_2\xval_2^{(j)}\right)^2\\
        & = \numsamp_{0,0}\alpha_{0,0}^2 + \numsamp_{1,0}\left(\alpha_{1,0}-a_1\right)^2 + \numsamp_{0,1}\left(\alpha_{0,1} - a_2\right)^2 + \numsamp_{1,1}\left(\alpha_{1,1}-a_1-a_2\right)^2 
    \end{align*}
The partial derivatives of $\sqloss$ with respect to each choice variable are
\begin{align*}
    \frac{\partial \sqloss}{\partial a_1} & = -2\numsamp_{1,0}(\alpha_{1,0}-a_1) - 2\numsamp_{1,1}(\alpha_{1,1}-a_1-a_2) \\
    \frac{\partial \sqloss}{\partial a_2} & = -2\numsamp_{0,1}(\alpha_{0,1}-a_2) - 2\numsamp_{1,1}(\alpha_{1,1}-a_1-a_2)
\end{align*}
Since $\sqloss$ is convex, it is minimized when both partials are zero. With some algebraic manipulation, we have that this is equivalent to the following two systems of equations:
\begin{align*}
    & (\numsamp_{1,0}+\numsamp_{1,1}) a_1 + \numsamp_{1,1}a_2 = \numsamp_{1,0}\alpha_{1,0} +\numsamp_{1,1}\alpha_{1,1}\\
    & \numsamp_{1,1} a_1 + (\numsamp_{0,1}+\numsamp_{1,1})a_2 = \numsamp_{0,1}\alpha_{0,1} +\numsamp_{1,1}\alpha_{1,1}\\
\end{align*}
In matrix form, we can write this as 
\[\begin{pmatrix} \numsamp_{1,0} + \numsamp_{1,1} & \numsamp_{1,1} \\ \numsamp_{1,1} & \numsamp_{1,1} + \numsamp_{0,1} \end{pmatrix}\begin{pmatrix}a_1 \\ a_2 \end{pmatrix} = \begin{pmatrix} \numsamp_{1,0}\alpha_{1,0} + \numsamp_{1,1}\alpha_{1,1} \\ \numsamp_{0,1}\alpha_{0,1} + \numsamp_{1,1}\alpha_{1,1} \end{pmatrix}.\]
Using the formula for the inverse of a $2 \times 2$ matrix, 
\begin{align*}
    \begin{pmatrix}a_1 \\ a_2 \end{pmatrix} & =\begin{pmatrix} \numsamp_{1,0} + \numsamp_{1,1} & \numsamp_{1,1} \\ \numsamp_{1,1} & \numsamp_{1,1} + \numsamp_{0,1} \end{pmatrix}^{-1} \begin{pmatrix} \numsamp_{1,0}\alpha_{1,0} + \numsamp_{1,1}\alpha_{1,1} \\ \numsamp_{0,1}\alpha_{0,1} + \numsamp_{1,1}\alpha_{1,1} \end{pmatrix}.\\
\end{align*}
Let $\widetilde{\numsamp}=\numsamp_{1,1}(\numsamp_{0,1} + \numsamp_{1,0}) +\numsamp_{1,0}\numsamp_{0,1}$. Note that the matrix inverse is well-defined as its determinant is \[\numsamp_{1,1}(\numsamp_{0,1} + \numsamp_{1,0}) +\numsamp_{1,0}\numsamp_{0,1}>\min\{\numsamp_{1,1}\numsamp_{0,1}, \numsamp_{1,1}\numsamp_{1,0}, \numsamp_{1,0}\numsamp_{0,1}\}>0\]
by our assumption that at least two of $\numsamp_{1,1},\numsamp_{0,1}, \numsamp_{1,0}$ are strictly positive. 
With a bit of algebra, we get that the estimated coefficients are
\begin{align*}
    & a_1  = \frac{(\numsamp_{0,1}+\numsamp_{1,1})(\numsamp_{1,0}\alpha_{1,0} + \numsamp_{1,1}\alpha_{1,1}) -\numsamp_{1,1}(\numsamp_{0,1}\alpha_{0,1} + \numsamp_{1,1}\alpha_{1,1})}{\widetilde{\numsamp}}\\
    & a_2 = \frac{-\numsamp_{1,1}(\numsamp_{1,0}\alpha_{1,0} + \numsamp_{1,1}\alpha_{1,1}) +(\numsamp_{1,1} + \numsamp_{1,0})(\numsamp_{0,1}\alpha_{0,1} + \numsamp_{1,1}\alpha_{1,1})}{\widetilde{\numsamp}}.
\end{align*}

\end{proof}

\section{Supplementary material for Section \ref{sec:our_solution}}\label{app:our_solution}

\subsection{Proof of Theorem \ref{thm:CV}}\label{appendix:thm_CV}

We will prove (a) and (b) of \Cref{thm:CV} separately. 
\subsubsection{Proof of Theorem \ref{thm:CV} (a)}
Recall that our goal is to show that there exists a universal constant $\const>0$ such that for any distribution $P$ over $\calX$ and any $\scoref \in \Fiso$
    \[R(\fcv, \scoref)\le \const \log |\Lambda|R(\hatf, \scoref).\]
    We will do this by proving a $\Omega(1/\numsamp)$ lower bound on the risk of $\flin$, and show that, up to a constant factor, the error of $\fcv$ exceeds that of $\flin$ by an additive term of $O(\log|\Lambda|/ \numsamp)$. 

    First, we will focus on the lower bound. It is well known that the risk of linear regression is $\Omega(1/\numsamp)$. However, for completeness, we will include a short proof. Define $\funspace_{\texttt{CONST}}$ to be the set of constant functions mapping $\calX$ to $[0,1]$. Since $\funspace_{\texttt{CONST}}\subset \Flin$ it follows that
    \[\inf_{\hatf}\sup_{\scoref \in \funspace_{\texttt{CONST}}}R(\hatf, \scoref) \le \inf_{\hatf}\sup_{\scoref \in \Flin}R(\hatf, \scoref). \]
    Therefore, it suffices to prove a lower bound for $\funspace_{\texttt{CONST}}$. 

    \textbf{Lower bound for learning constant functions.}
    Let $P_\scoref$ be the probability distribution over $(\xval, \yval)$ induced by a preference function $\scoref$ and $d_{\texttt{TV}}(P, Q) $ be the total variation distance between two distributions $P$ and $Q$. By Assoud's lemma (e.g., \citep{wu_lecture_2020} Theorem 10.2)
    \[\inf_{\hatf}\sup_{\scoref \in \funspace_{\texttt{CONST}}}R(\hatf, \scoref) \ge \sup_{\scoref_0, \scoref_1 \in \funspace_{\texttt{CONST}}}\frac{R(\scoref_0, \scoref_1)}{4}(1-d_{\texttt{TV}}(P_{\scoref_0}, P_{\scoref_1})).\]
    Let $\scoref_0 = 0$ and $\scoref_1 = \frac{1}{\sqrt{\numsamp}}$. Then $R(\scoref_0, \scoref_1)=\frac{1}{\numsamp}$. Further, by the affine invariance of total variation distance, and the fact that the sample mean is a sufficient statistic for normal random variables when the variance is known 
    \[d_{\texttt{TV}}\left(P_{\scoref_0}, P_{\scoref_1}\right) = d_{\texttt{TV}}\left(N(0,\frac{1}{\numsamp}\right), N\left(\frac{1}{\sqrt{\numsamp}},\frac{1}{\numsamp})\right) = d_{\texttt{TV}}\left(N(1,1),N(0,1)\right).\]
    Note that $(1-d_{\texttt{TV}}\left(N(1,1),N(0,1))\right)\approx 0.7\ge \frac{1}{2}$. This implies 
    \[\inf_{\hatf}\sup_{\scoref \in \Flin}R(\hatf, \scoref) \ge \frac{1}{8\numsamp}.\]

    \textbf{Upper bound for our algorithm.}
    To prove an upper bound for our algorithm, we will show that the risk of $\fcv$ can be bounded by the risk of $\scoref_{\infty}$. Then we will show that the risk of $\scoref_{\infty}$ is weakly less than that of non-negative least squares. To bound $R(\fcv,\scoref)$ in terms of $R(\scoref_\infty,\scoref)$ we will use the following result: 
    \begin{corollary}[Corollary 4.2 of \citet{vaart_oracle_2006}]\label{corr:oracle_inequality}
        Consider the setting where we have a set of functions $\scoref \in \funspace$ bounded by a constant $M\ge 1$ and are given access to data $(\xval^{(j)}, \yval^{(j)})_{j \in [\numsamp]}$ of independent and identically distributed samples where for all $j \in [\numsamp]$ 
        \[\yval^{(j)}=\scoref(\xval^{(j)}) +w^{(j)}\]
        with $w^{(j)}$ independent sub-exponential noise with $\E[w^{(j)}|\xval^{(j)}]=0$. Assume we divide the data into two random subsets: a dataset $S_1$ of size at least $\alpha \numsamp$ for validation and a dataset $S_0$ for training, where $\alpha >0$. Suppose we fit a set of estimators $\scoref_1,\dots, \scoref_K$ on data in $S_0$, calculate the mean squared error of each estimator on $S_1$, $\sum_{i \in S_1}(\scoref_k(\xval_i)-\yval_i)^2$ for $k \in [K]$, and return $\scoref_{\widetilde{k}}$ that minimizes the error on $S_1$. Then, there exists some constant $\const_1 >0$ such that for any $\delta \in (0,1)$:
        \[R(\scoref_{\widetilde{k}}, \scoref_0)\le (1+2\delta)\inf_{k \in [K]}R(\scoref_{k}, \scoref_0)+\const_1 /\numsamp \log(1+K)\frac{M^2}{\delta}.\]
    \end{corollary}
    This theorem gives us a bound on a cross-validation estimator that takes as input a set of $K$ functions fitted by regression on the data in terms of the risk of the best method in hindsight. 
    A brief remark is that in the Corollary presented in \citet{vaart_oracle_2006}, they don't mention the dependencies of $\const_1$. However, these dependencies can be attained by referring to Theorem 2.3 and Lemma 4.1 of the same paper from which the Corollary is derived (the proof is trivial). The constant factor depends only on the distribution of $\omega$ and linearly on $\frac{1}{\alpha}$, neither of which is relevant to our asymptotic analysis. 

    In order to apply this result, we need to check that our cross-validation setting is equivalent to the setting of \Cref{corr:oracle_inequality} for a specific choice of $\alpha, K, M$. Recall that we observe outputs from an unknown function $\scoref \in \Fiso$ with additive noise that is sub-exponential and with conditional mean zero (in our case, it is specifically $N(0,1)$) and impose that $\|\scoref\|_\infty \le 1$ (so we can set $M=1$). Further, we test $|\Lambda|$ potential regression functions on a $\frac{1}{4}$ fraction of the data, satisfying the condition that the size of the validation set is at least a constant fraction of the total data ($\alpha = \frac{1}{4}>0$). Lastly, we return $\fcv$ which minimizes the loss of $\{\flam\}_{\lambda \in \Lambda}$ on the validation dataset. Thus, the setting of \Cref{thm:CV} satisfies the conditions of \Cref{corr:oracle_inequality} with $K=|\Lambda|$, $M=1$ and $\alpha = \frac{1}{4}$. Letting $\delta = \frac{1}{2}$ and substituting in for the values of $K$, $M$ and $\alpha$ the result of \Cref{thm:CV} implies that
    \begin{equation}\label{eq:oracle_inequality}
        R(\fcv, \scoref)\le 2\inf_{\lambda \in \Lambda}R(\scoref_{\lambda}, \scoref)+\const_1 /\numsamp \log(1+|\Lambda|)
    \end{equation}
    As $\infty \in \Lambda$, it follows that 
    \[\inf_{\lambda \in \Lambda}R(\scoref_{\lambda}, \scoref) \le R(\scoref_{\infty}, \scoref).\] 
    When $\lambda = \infty$, we impose linearity on the regression function (prior to truncation). As the true $\scoref$ takes on values in $[0,1]$, $R(\scoref_\infty, \scoref) \le R(\flin, \scoref)$. This implies that 
    \[R(\fcv, \scoref) \le 2 R(\flin, \scoref)+\const_1 /\numsamp \log(1+|\Lambda|).\]
    Finally, as $\sup_{\scoref \in \Flin}R(\flin, \scoref)\ge \frac{1}{8\numsamp}$ by our lower bound on $\flin$, 
    \[\const_1 /\numsamp \log(1+|\Lambda|) \le \const_1 \log(1+|\Lambda|)\sup_{\scoref \in \Flin}R(\flin, \scoref).\]
    
    Using the fact that $\log(1+|\Lambda|) \le \log 2 \log(|\Lambda|)$, it follows that there exists some constant $\const >0$ such that 
    \[\sup_{\scoref \in \Flin}R(\fcv, \scoref)\le \const \log(|\Lambda|)\sup_{\scoref \in \Flin}R(\flin, \scoref).\]
    This concludes our proof of \Cref{thm:CV}
\subsubsection{Proof of Theorem \ref{thm:CV} (b)}
Recall that our goal is to prove that 
\[\sup_{\scoref \in \Fiso}R(\fcv, \scoref) \le \const_{\numcriteria, \minval}\numsamp^{-1/\numcriteria}\log^{\gamma_\numcriteria}\numsamp.\]
We will do this by combining the oracle inequality from Equation \Cref{eq:oracle_inequality} with an upper bound on $R(\scoref_0, \scoref)$. Since we have already proved the oracle inequality in the proof of part (a) of this theorem, we will proceed to prove an upper bound on $R(\scoref_0, \scoref)$. We will do this by slightly modifying a result from \citet{han_isotonic_2019} on the risk of isotonic functions mapping $[0,1]^\numcriteria \to [-1,1]$. 
\begin{theorem}[Theorem 4 of \citet{han_isotonic_2019}]\label{thm:isotonic}
    Consider the setting of \Cref{thm:isotonic} where instead of the preference function being in $\Fiso$, it belongs to the set of isotonic functions from $\ccalX =[0,1]^\numcriteria$ to $[-1,1]$ (which we denote by the set $\cFiso$). Let the distribution over $\ccalX$ be denoted by $\cP$. 
    If $\cP$ satisfies 
    \[0 < m_0 \le \inf_{\xval \in \calX}\cP(\xval) \le \sup_{\xval \in \ccalX}\cP(\xval)\le M_0 \le \infty,\]
    then there exists an estimator $\hatf$ such that for $\numcriteria\ge 2$, these exists a $\const_{\numcriteria, m_0, M_0}$ depending only on $\numcriteria, m_0, M_0$ such that 
    \[\sup_{\cscoref \in \cFiso}R(\hatf, \scoref) \le \const_{\numcriteria, m_0, M_0}
\numsamp^{-1/\numcriteria}\log^{\gamma_\numcriteria}\numsamp.\]
\end{theorem}
Our goal is to show that, given a different set of preference functions and our algorithm (which is slightly different from theirs), the same upper bound still holds. In \citet{han_isotonic_2019}, they pick estimator $\hatf$ as the function in $\cFiso$ minimizing squared error on the sample data. With a small amount of algebra, one can confirm that this is equivalent to our approach of finding the minimizing function with outputs in $\R$ and truncating the outputs to $[-1,1]$. For values not seen during training, they interpolate using. 
\[\hatf(\xval):=\min\left(\{\hatf(\xval^{(j)}): 1\le i \le \numcriteria, \xval_{i}^{(j)} \ge \xval_i\}\cup \{\max_{j \in [\numsamp]}\hatf(\xval^{(j)})\}\right).\]
In order to transfer the results of \Cref{thm:isotonic} to our setting, we have to ensure that (extensions of) their arguments apply given differences in the domain, co-domain, and interpolation method. We will show later that the upper bounds apply, given most reasonable interpolation methods, including ours. Further, since we do an ex-post truncation to the true output space $[0,1]$ instead of $[-1,1]$, this step will weakly lower the risk of the estimator. Therefore, we will focus on adapting their arguments to apply to distributions on $\calX$ satisfying the conditions of \Cref{thm:CV} (b). 

In order to do this, we will use the fact that the learning problem of isotonic functions defined on $\calX$ is equivalent to the learning problem when the function is defined only on the lattice, 
\[\mathbb{L}_{\numcriteria, \numvals}:=\prod_{i \in [\numcriteria]}\{1/\numvals, 2/\numvals, \dots, 1\}.\]

In order to bound the risk $R(\hatf, \cscoref)$, the authors define the event that the infinity norm of the function error exceeds some threshold and consider the risk when this event occurs and when it does not. Specifically, for some value of $a > 0$, they define  
$A = \{\|\hatf-\scoref\|_\infty \le a\}$ and indicators $\mathbb{I}_{A}$ and $\mathbb{I}_{A^C}$ to write the following decomposition of risk 
\[R(\hatf, \cscoref) = \E [\|\hatf-\cscoref\|_{L_2(\cP)}\mathbb{I}_{A}] + \E [\|\hatf-\cscoref\|_{L_2(\cP)}\mathbb{I}_{A^C}].\]
They prove upper bounds on each of these terms separately. We will show that, in both cases, the upper bounds still apply when the marginal distribution is defined on the lattice for a suitable choice of $m_0, M_0$. First, consider the risk given bounded error, $\E [\|\hatf-\cscoref\|_{L_2(\cP)}\mathbb{I}_{A^C}]$. The authors use arguments not depending on $\cP$ to show that if certain upper bounds on certain complexity measures for subsets of isotonic functions on $\ccalX$ hold, then when $a=6\log^{\frac{1}{2}}\numsamp$ one can upper-bound the value of $\E [\|\hatf-\cscoref\|_{L_2(\cP)}\mathbb{I}_{A^C}]$ by $\const_{m_0, M_0}\numsamp^{-1/\numcriteria}\log^{\gamma_\numcriteria}\numsamp$. While in the key result bounding the risk conditional on $A$ (Proposition 7), depend on a Rademacher and Gaussian complexity term, they show via further analysis (specifically Lemma 5 and 6, neither of which depending on properties of $\cP$) that it suffices to bound the Rademacher complexity of the set of isotonic functions in the $L_2(\cP)$ ball around the zero function with outputs in $[-1,1]$. 

Now we will more specifically discuss the complexity measure that they use. Let $\xi^{(1)}, \dots, \xi^{(\numsamp)}$ be a sequence of random signs. Define $B(r,\cP)$ to be the ball of radius $2$ in the $L_2(\cP)$. Let $\mathcal{G}_{\cxval}(0,r,1) = \{\cscoref \in \cFiso: \cscoref \in B_2(r,\cP) \text{ and } \|\cscoref\|_\infty \le 1\}$ be the set of isotonic functions over $\ccalX$ in the $L_1(\cP)$ ball of radius $r$ around the zero function. The authors bound 
\[\E \left[\sup_{\cscoref \in \mathcal{G}_{\cxval}(0,r,1)}\left|\frac{1}{\numsamp^{1/2}}\sum_{j=1}^{\numsamp}\xi^{(j)}\cscoref(\cxval^{(j)})\right|\right].\]
We can define the respective quantities for the probability measures restricted to the lattice (or equivalently over $\calX$) in terms by $\mathcal{G}_\xval(0,r,1) = \{\scoref \in \Fiso(\R): \scoref \in B_2(r, P) \text{ and } \|\scoref\|_\infty \le 1\}$ for $B_2(r, P)$ the $L_2(P)$ ball around the zero function. The relevant complexity measure on this restricted domain is 
\[\E \left[\sup_{\scoref \in \mathcal{G}_{\xval}(0,r,1)}\left|\frac{1}{\numsamp^{1/2}}\sum_{j=1}^{\numsamp}\xi^{(j)}\scoref(\xval^{(j)})\right|\right].\]

In the following proposition we show that the Rademacher complexity of $\mathcal{G}_{\xval}(0,r,1)$ given a distribution $P$ satisfying $\min_{\xval \in \calX}P(\xval)|\calX| \ge \minval$ can be upper bounded by the complexity of $\mathcal{G}_{\cxval}(0,r,1)$ given distributions over $\ccalX$ satisfying the conditions of \Cref{thm:isotonic} for $m_0=\minval, M_0=1-\minval$. 
\begin{proposition}\label{prop:complexity_UB}
    For any distribution $P$ satisfying $\min_{\xval \in \calX}P(\xval)|\calX| \ge \minval$, there exists a distribution $\cP$ over $\ccalX$ satisfying the assumptions of Proposition 8 with constants $m_0=\minval,M_0=1-\minval$ such that
    \[\E \left[\sup_{\scoref \in \mathcal{G}_\xval(0,r,1)}\left|\sum_{j=1}^{\numsamp}\xi^{(j)}\scoref(\xval^{(j)})\right|\right] \le \E \left[\sup_{\scoref \in \mathcal{G}_{\cxval}(0,r,1)}\left|\sum_{j=1}^{\numsamp}\xi^{(j)}\cscoref(\cxval^{(j)})\right|\right].\]
\end{proposition}

The result of this proposition (which we prove later) is an inequality in the complexities of function classes. This enables us to directly apply the distribution-dependent upper bounds. Let $ \lesssim_{\minval}$ denote an inequality up to a constant factor depending only on $\minval$. By \Cref{prop:complexity_UB}, $\E [\|\hatf-\scoref\|_{L_2(P)}\mathbb{I}_{A}]\lesssim_{\minval}\numsamp^{-1/\numcriteria}\log^{\gamma_\numcriteria}\numsamp$. Thus we have a bound of the first term. 

To bound the second term, $\E [\|\hatf-\scoref\|_{L_2(P)}\mathbb{I}_{A^C}]$, because $\scoref, \hatf$ map to points in $[0,1]$, we have the following inequality
\[\E [\|\hatf-\scoref\|_{L_2(P)}\mathbb{I}_{A^C}]\le P(A^C).\]
In order to bound $P(A^C)$, the authors prove that $P(A^C) \lesssim \numsamp^{-1}$ for their specific choice of $r$ (see \citep{han_isotonic_2019} Lemma 10). Specifically, they use the fact that $\|\hatf -\cscoref \|_\infty \le \|\hatf\|_\infty + \|\cscoref\|_\infty \le 2 + \max_{j \in [\numsamp]}|w^{(j)}|$ by the constraint that $\|\cscoref\|_\infty \le 1$ and $\|\hatf\|_\infty \le \max_{j \in [\numsamp]}\cscoref(\xval^{(j)})+w^{(j)}$. This argument is the only one dependent on the interpolation method and relies only on the fact that the value of unseen $\xval$ is within the range of observed outputs. Our interpolation method 
is an average of the function values for observed outputs, and thus their argument still applies to attain a $O(\numsamp^{-1})$ upper bound on $\E [\|\hatf-\scoref\|_{L_2(P)}\mathbb{I}_{A^C}]$. Because the second term is asymptotically smaller, the risk is dominated by $\E [\|\hatf-\scoref\|_{L_2(P)}\mathbb{I}_{A}]$. Thus
\begin{align*}
    R(\scoref_0, \scoref) \lesssim_{\numcriteria, \minval}\numsamp^{-1/\numcriteria}\log^{\gamma_\numcriteria} \numsamp + \numsamp^{-1}
\end{align*}
Applying our oracle inequality on the risk of $\fcv$:
\begin{align*}
    R(\fcv, \scoref)& \le 2\inf_{\lambda \in \Lambda}R(\scoref_{\lambda}, \scoref)+\const /\numsamp \log(1+|\Lambda|)\\
    & \lesssim_{\numcriteria, \minval}\numsamp^{-1/\numcriteria}\log ^{\gamma_\numcriteria} \numsamp.
\end{align*}
Therefore, all that is left is to prove \Cref{prop:complexity_UB}. 
\begin{proof}[Proof of \Cref{prop:complexity_UB}.]
    We will prove this showing that there exists a joint distribution over $\xval$ and $\cxval$ such that for all functions $\scoref \in \mathcal{G}_\xval(0,r,1)$ there exists a function $\cscoref \in \mathcal{G}_{\cxval}(0,r,1)$ such that $\scoref(\xval) = \cscoref(\cxval)$ with probability $1$. If such a construction exists, it implies that given any realized of the $\xi^{(j)}$, any supremum taken over all $\mathcal{G}_{\xval}(0,r,1)$ (including the Rademacher complexity) is at most the supremum taken over $\mathcal{G}_{\cxval}(0,r,1)$. As the inequality holds with probability $1$ and does not depend on the joint distribution of $\xi$ and $(\xval,\cxval)$, it also holds for the expected Rademacher complexities when $\xi$ is independent of $(\xval,\cxval)$, proving the Lemma. 

    Hence, all that is left is to show that such a joint distribution and set of functions that are equal with probability $1$ exists. To define this coupling between $\calX$ and $\ccalX$, we first will partition $[0,1]^\numcriteria$ into hypercubes of equal size. For all $\xval \in [\numvals]^\numcriteria$, let $A_x = \prod_{i =1}^{\numcriteria}R_{x_i}$ where $R_{\subindex}=(\frac{\subindex-1}{\numvals}, \frac{\subindex}{\numvals}]$ for $\subindex >1$ and $R_\subindex = [\frac{\subindex-1}{\numvals}, \frac{\subindex}{\numvals}]$ for $\subindex=1$. Let $\mathcal{A}=\{A_x:x \in [\numvals]^\numcriteria\}$. By construction, all elements of $\mathcal{A}$ have volume $\frac{1}{\numvals^\numcriteria}$. Let $\cxval \sim \mathrm{Unif}(A_\xval)$, that is, $\cxval$ takes on all values in the hypercube $A_{\cxval}$ with equal probability. Now we will confirm that the distribution of $\cxval$, denoted by $\cP$, satisfies the condition of Proposition 8, namely that
    \[m_0 \le \inf_{\cxval \in \ccalX}\cP(\cxval) \le \sup_{\cxval \in \ccalX}P_{\cxval}(\cxval) \le M_0.\]
    Let $\xval(\cxval)$ be the element of $\calX$ corresponding to the hypercube that $\cxval$ is contained in. By construction $\cP(\cxval) = P(\xval(\cxval))\numvals^\numcriteria$. Since $P(\xval)|\calX| \in [\minval, 1-\minval]$ for all $x \in [\numvals]^\numcriteria$, it follows that $P(\xval(\cxval))|\calX| \in [m_0, M_0]$ for all $\cxval \in \ccalX$. Given this joint distribution, now we need to show that for all functions $\scoref \in \mathcal{G}_\xval(0,r,1)$ there exists a function $\cscoref \in \mathcal{G}_{\cxval}(0,r,1)$ such that $\scoref(\xval) = \cscoref(\cxval)$. For any  $\scoref \in \mathcal{G}_{\xval}(0,r,1)$ consider the function $\cscoref(\cxval) = \scoref(\xval(\cxval))$. Since $\xval(\cxval)$ and $\scoref$ are isotonic, and the composition of isotonic functions is isotonic, so is $\cscoref$. Further, by construction $\E\|\cscoref\|_2=\E\|\scoref\|_2\le r$ and $\|\cscoref\|_\infty = \|\scoref\|_\infty \le 1$ which implies that $\cscoref \in B_2(r,\cP) \cap B_\infty(1)$ and thus $\cscoref \in \mathcal{G}_{\cxval}(0,r,1)$. 

    This concludes our proof that there exists a joint distribution over $\xval$ and $\cxval$ such that for all functions $\scoref \in \mathcal{G}_\xval(0,r,1)$ there exists a function $\cscoref \in \mathcal{G}_{\cxval}(0,r,1)$ such that $\scoref(\xval) = \cscoref(\cxval)$ with probability $1$. Since such a distribution exists and is a sufficient condition for the statement of the Lemma, we are done.   
\end{proof}

\section{Supplementary material for Section \ref{sec:empirical_results}}\label{app:empirical_results}
\paragraph{Necessary compute resources.} All experiments were run locally on a personal computer. The computational requirements are minimal---all experiments complete in under a few hours on standard consumer hardware (we used the ``caffeinate'' command in terminal) and do not require GPUs or specialized infrastructure.

\paragraph{Code and data availability.} Our code is publicly available online at our \href{https://github.com/madelineceli13/evaluator-preference-algorithm}{Github repository}.

The datasets for Tripadvisor and ICLR LLM/Human peer review are too large to store on GitHub. They can be accessed at the following links: \href{https://github.com/diegoantognini/HotelRec}{Tripadvisor dataset link}, and \href{https://huggingface.co/datasets/IntelLabs/AI-Peer-Review-Detection-Benchmark}{ICLR LLM/Human peer review dataset link}. 

\subsection{Comparison of our algorithm with alternative machine learning methods}\label{app:ml_comparisons}
In addition to comparing our algorithm with non-negative least squares, we test two additional methods: a monotone-constrained neural network \citep{runje_constrained_2023} and a monotone-constrained gradient-boosted tree ensemble. A monotone-constrained neural network (henceforth monotone NN) enforces monotonicity in each variable. We use the methods detailed in \citet{runje_constrained_2023}. A monotone-constrained gradient-boosted tree ensemble (henceforth monotone GBM for gradient-boosting machine) follows the structure of a typical GBM. Starting from a constant prediction, the model fits an estimator, in this case a regression tree, to the residuals (the errors the ensemble makes on each training point) and then adds that tree (scaled by a learning rate) to the current prediction. Then the model recomputes the new residuals and repeats, with each new tree correcting the errors from the previous ones, until a stopping criterion is reached (e.g., a maximum number of trees, minimum leaf size or marginal error reduction). The final estimator is the sum of all fitted trees. In a monotone GBM, the only additional component is that the estimators are constrained to generate isotonic outputs. As such, any weighted combination of the individual estimators (with non-negative weights) will output an isotonic function. Each method---monotone NNs and GBMs---is tuned by cross-validation on the same train/validation/test split and clipping range as our method, so all reported estimation and reducible errors are directly comparable. 

\subsubsection{Synthetic simulations}
We use synthetic simulations to compare the relative performance of our algorithm with linear regression and machine learning baselines given a diversity of preference types and sample sizes. In particular, we test each algorithm on linear, Cobb-Douglas, and Leontief preference functions using the same setup as in \Cref{sec:synthetic_simulations}. As shown in \Cref{fig:nn_gbm}, our algorithm outperforms GBM on all preference types. Moreover, while the monotone NN has a similar estimation error to our algorithm for Linear and Cobb-Douglas preferences, it does significantly worse for Leontief preferences. 

\begin{figure}[htb]
    \centering
   \includegraphics[width = \linewidth]{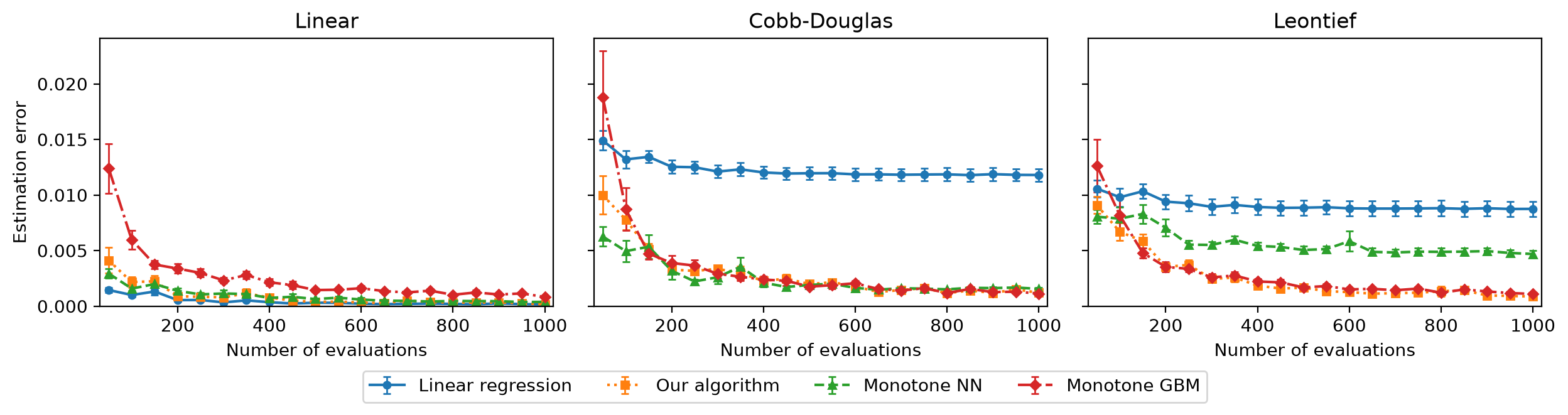}
    \caption{Simulation results for synthetic preferences including monotone NN and GBM baselines. Error bars are computed using the standard error of the mean estimation error.}\label{fig:nn_gbm}
\end{figure}

\subsubsection{Results on Tripadvisor data}
We then test our algorithm on the Tripadvisor datasets described in \Cref{sec:trip_advisor}. We follow the same procedure and report the reducible error of each method (a proxy for estimation error). \Cref{tab:reducible-error} shows the results. While the differences in reducible error are not necessarily statistically significant, it is notable that our algorithm has lower error than both monotone NN and GBM for all of the datasets. Altogether with the synthetic simulation results, this evidence suggests that our algorithm does not just outperform linear models but potentially SOTA machine learning algorithms as well.

\begin{table}[htb]
\centering
\begin{tabular}{l r r r}
\toprule
Dataset    & \multicolumn{1}{c}{Monotone NN} & \multicolumn{1}{c}{Monotone GBM} & \multicolumn{1}{c}{Our algo} \\
\midrule
All - 2019 & 0.014 (0.003) & 0.023 (0.004) & \textbf{0.013 (0.003)} \\
Business   & 0.015 (0.003) & 0.018 (0.003) & \textbf{0.014 (0.003)} \\
Couples    & 0.013 (0.002) & 0.016 (0.002) & \textbf{0.012 (0.002)} \\
Family     & 0.013 (0.002) & 0.019 (0.002) & \textbf{0.011 (0.002)} \\
\bottomrule
\end{tabular}
\caption{Reducible error by dataset, with the lowest error in each row bolded. Standard errors (in parentheses) are computed as the standard error of the mean for prediction error.}
\label{tab:reducible-error}
\end{table}

\subsection{Plots for understanding preferences}\label{app:criteria_plots}

\subsubsection{Construction of error bars}
To construct error bars for the criteria plots, we use an upper bound on the standard error of reducible error. Prediction error is the sum of two independent components: estimation error and noise. Therefore, the expected standard error of prediction error is weakly larger than that of estimation error. As reducible error is our proxy for estimation error, this is approximately true for reducible error as well. We therefore bound the standard error of reducible error by that of prediction error. Lastly, we use the delta method to convert this to an upper confidence bound on $L_1$ error. Therefore, the width of the error bar is $\frac{1.96 \times \widehat{\mathrm{SE}}(\text{prediction error})}{2\sqrt{\text{reducible error}}}$.

\subsubsection{Additional plots for understanding human and LLM preferences}

Below we plot the values of our learned preference functions on the ICLR LLM and human peer review dataset. Recall that the criteria are scored from 1 to 4, and that 2 and 3 are the most common scores. We fix all but one criterion at a given score (2 on the left, 3 on the right). Then we vary the remaining criterion from 1-4. As for presentation (see \Cref{sec:ICLR}) GPT are concave when the other criterion scores are 2, and convex when they are 3. In contrast, the human preference functions functions are closer to linear in all of the plots. 
\begin{figure}
    \centering
   \includegraphics[width = \linewidth]{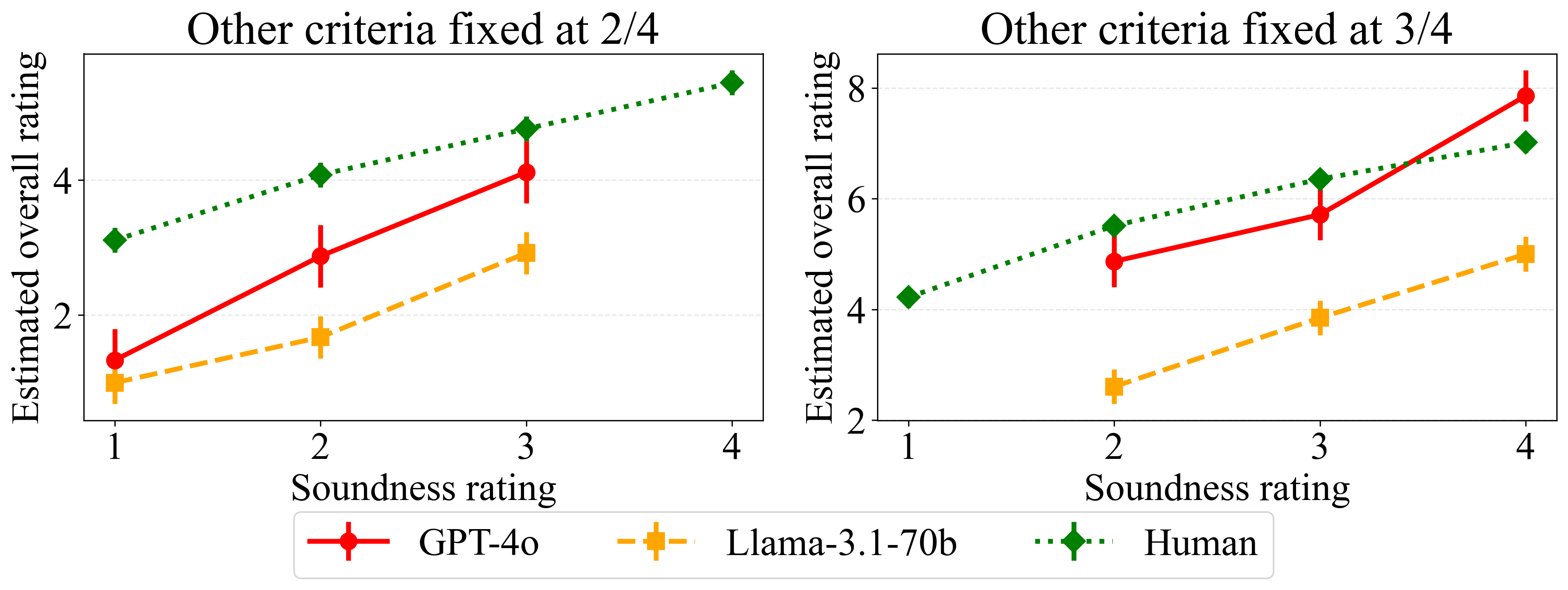}
   \includegraphics[width = \linewidth]{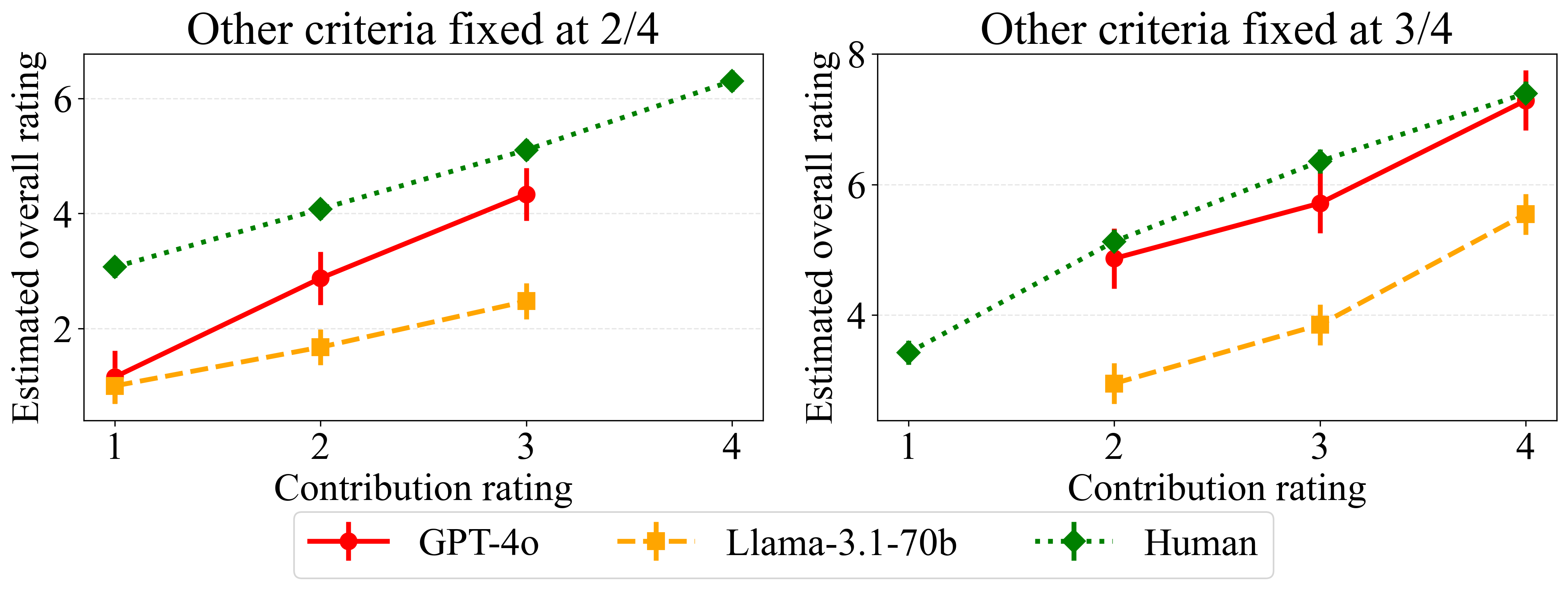}
    \caption{Impact of increasing specific criterion scores on the overall paper rating. We omit estimates for all criterion scores not seen in training. Error bars are computed based on the standard error of prediction error.}\label{fig:LLM_additional_criteria_plots}
\end{figure}

\subsection{Results on additional human preference datasets}\label{appendix:additional_human_evaluator}
In addition to Tripadvisor, we test our algorithm relative to linear regression on two other datasets of human preferences. However, their sample size was insufficient to generate statistically significant findings. 

\paragraph{ICLR Peer Review data}. The first dataset is of criteria scores and overall scores or recommendations provided by peer reviewers for the International Conference on Learning Representations (ICLR) from 2023 to 2025. Since all of the reviews are publicly available, we obtained them via web scraping. Additionally, each year used different criteria and/or rating scales. Every year has 3 criteria, which are rated on a scale of 1-4, and an overall recommendation or rating. In 2023, the criteria were correctness, technical novelty and significance, and empirical novelty and significance. In 2024 and 2025, the criteria were soundness, presentation, and contribution. For the years 2023 and 2025, overall ratings took on values in $\{1,3,5,6,8,10\}$. In 2024, the possible rating values were $\{1,3,5,6,8\}$. In this context, the non-equal intervals are meant to show the relative difference in each level of recommendation. In 2024, for instance, 3 meant ``reject, not good enough'' and 6 meant ``marginally above the acceptance threshold.'' Since the criteria and/or rating scale varied based on the year, we analyzed each separately. The data was cleaned to contain only reviews for which all criteria scores were filled in and obtained via web-scraping. An example of an evaluation from 2024 is presented in \Cref{fig:tripadvisor_iclr} (b). 

\paragraph{Astrobee Challenges Series evaluations}
The last dataset we tested contains evaluations of engineering robotics solutions in the aerospace domain. The solutions being evaluated came from a competition called the Astrobee Challenges Series run by NASA and other researchers. They asked people to propose solutions to $17$ different challenges relating to the creation of an attachment and orientation arm for Astrobee, a free-flying robotic system used by NASA. The winning solution got money and was incorporated into new designs. After the competition ended, Lane and co-authors did a study to see how the type of evaluator impacted their specific preferences \citep{lane_architectural_2023}. Evaluators were asked to provide criteria scores for the feasibility and novelty of the solution as well as an overall quality score. Each score took on integer values between 1 and 7. We use data from that study, which contains $3,850$ evaluations from $374$ evaluators corresponding to $101$ solutions of $9$ challenges. An example of what the response questions look like is given in \Cref{fig:nasa}.

\begin{figure}[htb]
    \centering
        \includegraphics[width = \textwidth]{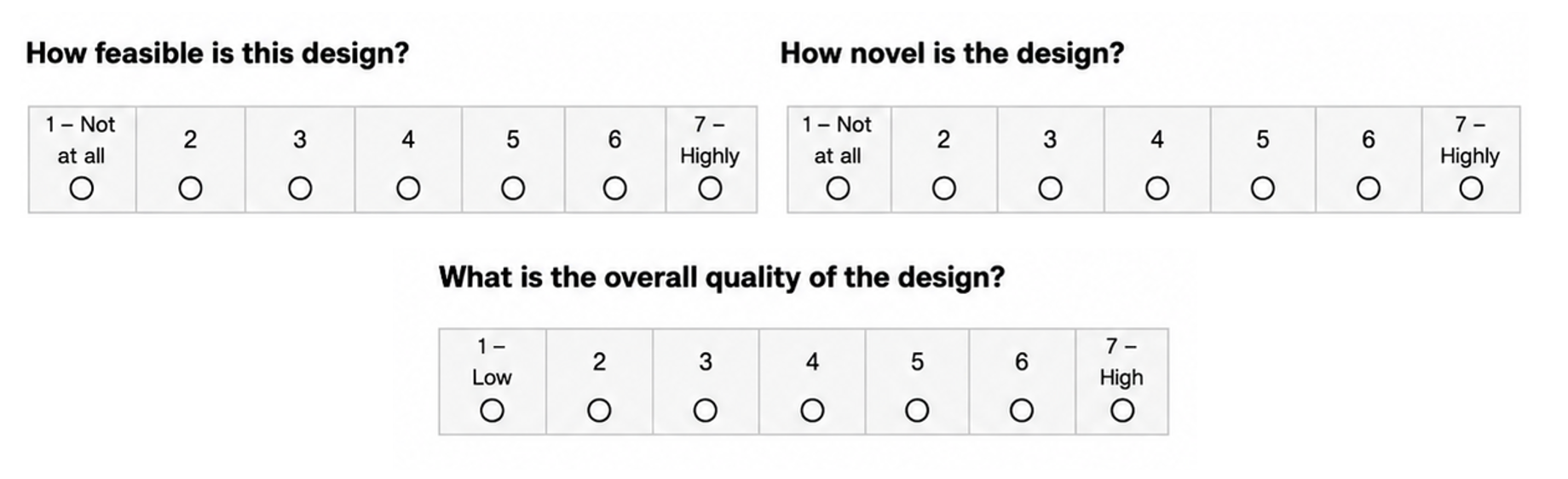}
    \caption{Criteria and overall score input template for Astrobee evaluations. Obtained from \citet{lane_architectural_2023}. Each criterion and overall score are provided on a scale from 1 to 7.}
    \label{fig:nasa}
\end{figure}

\begin{table}[t!]
    \centering
    \begin{tabular}{lcccc}
        \toprule
        Dataset & \# criteria ($\numcriteria$) & \# values ($\numvals$) & score range & \# samples ($\numsamp$) \\
        \midrule
        \multicolumn{5}{l}{\textit{ICLR peer review data}} \\
        \quad 2023 & \multirow{3}{*}{3} & \multirow{3}{*}{4} & 1--10 & 13{,}212 \\
        \quad 2024 &                    &                    & 1--8  & 28{,}028 \\
        \quad 2025 &                    &                    & 1--10 & 46{,}748 \\
        \midrule
        \multicolumn{5}{l}{\textit{Astrobee Challenges Series evaluations}} \\
        \quad All & \multirow{4}{*}{2} & \multirow{4}{*}{7} & \multirow{4}{*}{1-7} & 3{,}850 \\
        \quad Multivalent & &  &  & 89 \\
        \quad Univalent & &  &  & 719 \\
        \quad Outside expertise & &  &  & 3{,}042 \\
        \bottomrule
    \end{tabular}
    \caption{Summary of dataset attributes for ICLR and Astrobee evaluations.}\label{tbl:data_summary_appendix}
\end{table}
\paragraph{Results} Our results are shown in \Cref{tbl:results_appendix}. In all cases, our algorithm does weakly better. For ICLR, some of the improvements in estimation error are very large (as much as 81.6\% less error). However, none of the results are statistically significant. We suspect that the lack of significance is a function of insufficient sample size. The largest of the datasets is ICLR 2025, with a sample size of 47k. For comparison, the smallest Tripadvisor dataset is over four times this size, 197k. The smallest dataset is the Astrobee evaluations with 3.8k samples, nearly one hundredth the size of the largest Tripadvisor dataset (332k samples). 

\begin{table}[t!]
\caption{Results for linear regression and our algorithm on Tripadvisor dataset. Standard errors are computed with the standard error of the mean prediction error. $\Delta$ (\%) denotes the relative error reduction of our algorithm over linear regression.}
\label{tbl:results_appendix}
\centering
\begin{tabular}{lccrccrr}
\toprule
 & \multicolumn{3}{c}{Prediction Error} & \multicolumn{3}{c}{Reducible Error} \\
\cmidrule(lr){2-4} \cmidrule(lr){5-7}
 & Lin. Reg. & Our algo & $\Delta (\%)$. & Lin. Reg. & Our algo & $\Delta (\%)$  \\
\midrule
ICLR 2023 & 1.518 (0.043) & 1.491 (0.043) & 1.8\% & 0.053 (0.043) & 0.027 (0.043) & 50.2\% \\
ICLR 2024 & 1.531 (0.051) & 1.462 (0.042) & 4.5\% & 0.084 (0.051) & 0.015 (0.042) & 81.6\% \\
ICLR 2025 & 1.280 (0.019) & 1.257 (0.019) & 1.8\% & 0.03 (0.019) & 0.007 (0.019) & 75.3\% \\
Astrobee & 0.893 (0.081) & 0.893 (0.081) & 0.0\% & 0.079 (0.081) & 0.079 (0.081) & 0.5\% \\
\bottomrule
\end{tabular}
\end{table}

\subsubsection{Additional analysis on Astrobee dataset}\label{app:nasa}
The final empirical analysis we perform is on the Astrobee dataset.~\\ 

\textit{\textbf{Question}: Do evaluator preferences depend on their type of expertise?} 

One of the questions in the corresponding paper \citep{lane_architectural_2023} is: how does the evaluator's expertise impact their evaluation preferences? In particular, the authors group evaluators by whether they have expertise in the relevant domains of robotics or aerospace. Univalent evaluators have expertise in either robotics or aerospace, multivalent evaluators have expertise in both domains, and evaluators with outside expertise have experience in neither. The authors find that the preferences of multivalent and univalent evaluators differ qualitatively. If this is the case, we would expect (all else constant) to have less noise in the dataset of multivalent and univalent evaluators than the dataset of all of the evaluators. 

To test our hypothesis that evaluator preferences differ based on their expertise, we partition the evaluation-solution pairs into three subsets based on the expertise of the evaluator: multivalent, univalent, and outside expertise. We measure preference consistency by computing irreducible error. Under our hypothesis, we would expect to have a lower irreducible error on a dataset of only evaluators in one of the three categories---multivalent, univalent, or outside expertise---relative to the dataset of all evaluators. To compute an estimate and a confidence interval for the irreducible error, we use bootstrapping. We do $1000$ resamples with replacement, where each resampled dataset has the size of the input dataset. We report the mean and a 95\% percentile confidence interval ($2.5$ to $97.5$ quantiles). 

Additionally, we run our algorithm on the data for each subset of evaluators and all evaluators. We report the optimal regularization parameter and compute the prediction error according to Equation \Cref{eq:prediction_error}. These statistics tell us (loosely speaking) how well we can learn evaluator preferences given the available data. 

\paragraph{Results}
First, we report the amount of irreducible error to quantify how consistent the preference function is across evaluations. We find that the datasets of evaluator types have a weakly lower irreducible error than the dataset of all evaluators. For some evaluator types, this difference is substantial: as shown in \Cref{tbl:nasa}, the level of irreducible error for multivalent evaluators is less than half of that for all evaluators. A stronger version of our hypothesis predicts that groups with more overlapping expertise should show greater preference consistency. Since multivalent evaluators share two domains of expertise, they are likely to have the most similar work background to one another, then univalent, and then those with neither (i.e., outside expertise). Our results track this: irreducible error for multivalent evaluators is lowest, then univalent, outside expertise, and all evaluators (though the latter three relations are not statistically significant).

Another way to test preference consistency is to analyze the prediction error of our algorithm. The less noise in the data, the easier it is to learn the preference function. Thus, all else constant, we expect there to be lower estimation and prediction error for less noisy datasets. Additionally, the more data points we have, the better we can estimate the function. Thus, ceteris paribus, we expect prediction error to decrease with the number of samples. The optimal regularization parameters support this argument, and are larger for datasets with smaller samples, indicating that purely isotonic regression may be overfitting due to insufficient sample size (see \Cref{tbl:data_summary_appendix}). These effects work in opposite directions: since the datasets with the lowest amount of noise have the smallest number of samples, it is unclear a priori whether they would have lower or higher prediction error. We find that prediction error is higher for larger datasets, even though the largest dataset is over 30 times the size of the smallest one. Multivalent evaluators ($\numsamp = 89$) have the lowest prediction error, then univalent evaluators ($\numsamp = 719$) outside expertise ($\numsamp = 3{,}042$), and all evaluators ($\numsamp = 3{,}850$). \textit{Together, these results illustrate the intuitive idea: the more similar the expertise of different evaluators is, the more similar their preferences are, and, consequently, the easier those preferences are to learn. }

\begin{table}[tbp]
\caption{Level of preference consistency and our algorithm performance by evaluator expertise. Standard errors are in parentheses for the prediction error. 95\% bootstrap CI are in brackets for irreducible error.}
\label{tbl:nasa}
\centering
\begin{tabular}{lcccc}
\toprule
 & All & Multivalent & Univalent & Outside expertise  \\
\midrule
Irreducible error & 0.85 [0.785,0.918] & 0.406 [0.243,0.595] & 0.653 [0.538,0.791] & 0.85 [0.771,0.924]\\
Prediction error & 0.893 (0.081) & 0.533 (0.153) & 0.587 (0.094) & 0.864 (0.071) \\
$\lambda^\star$ & 8.0 & $\infty$ & 0.5 & 0\\
\bottomrule
\end{tabular}
\end{table}

\subsection{Chosen regularization parameters}\label{app:reg_param}
In \Cref{tab:lambda_star} we present the values of $\lambda^\star$ chosen by our algorithm for each of the real-world datasets we run our algorithm on. We will now make a few comments about what these values illustrate about our algorithm and the data. Observe that the selected values for each of the Tripadvisor datasets (which have hundreds of thousands of samples) are quite small. Thus, the relative improvement for our algorithm presented in \Cref{sec:trip_advisor} is a result of the flexibility provided by isotonic regression. Additionally, it is notable that for the ICLR datasets from 2024-2026, each picks purely isotonic regression. This fact is perhaps indicative of the nonlinearity of peer review preferences, as illustrated in \Cref{sec:iclr_case_study}. Lastly, the only datasets for which $\lambda>1$ is selected are for the Astrobee evaluations. These datasets have significantly fewer samples. Indeed, on the dataset with $\lambda^\star = \infty$ (multivalent evaluators), there are only $89$ samples!  

Furthermore, in \Cref{tab:synthetic_lambda} we include the chosen regularization parameters for our synthetic simulations in \Cref{sec:synthetic_simulations}. Interestingly, most simulations of non-linear preferences pick an intermediate value of $\lambda$, and this is especially pronounced for the Cobb-Douglas function, as intermediate $\lambda$ values are picked in $64\%$ of the simulations. Furthermore, Cobb-Douglas and Leontief choose $\lambda = \infty$ in less than $3\%$ of simulations. \textbf{Put together, these results illustrate that cross-validation generally does not choose very large $\lambda$ unless we explicitly simulate a linear model or have very small sample sizes, and, oftentimes, that it finds intermediate $\lambda$ values to be the best.}

\begin{table}[htbp]
\centering
\begin{tabular}{lcc}
\toprule
Dataset & $\lambda^{\star}$ & $\numsamp$\\
\midrule
\textit{Tripadvisor} & &\\
\quad All - 2019 & $2^{-6}$ & $332{,}349$\\
\quad Couples & $2^{-8}$ & $302{,}894$ \\
\quad Family & $0$ & $371{,}185$\\
\quad Business & $2^{-9}$ & $197{,}031$\\
\textit{ICLR peer reviews} & &\\
\quad ICLR 2023 & $2^{-4}$ & $13{,}212$\\
\quad ICLR 2024 & $0$ & $28{,}082$\\
\quad ICLR 2025 & $0$ & $46{,}748$\\
\quad ICLR 2026 & $0$ & $75{,}859$\\
\textit{LLM Case Study} & &\\
\quad GPT-4o & $0$ & $27{,}857$\\
\quad Llama 3.1 70b & $0$ & $27{,}834$\\
\quad Human evaluators & $2^{-5}$ & $27{,}857$\\
\textit{Astrobee Evaluations} & &\\
\quad All evaluators & $8$ & $3{,}850$\\
\quad Univalent & $0.5$ & $719$\\
\quad Multivalent & $\infty$ & $89$\\
\quad Outside expertise & $0$ & $3{,}042$\\
\bottomrule
\end{tabular}
\caption{Optimal regularization parameter $\lambda^{\star}$ selected via cross-validation for each dataset.}
\label{tab:lambda_star}
\end{table}

\begin{table}[htbp]
\centering
\begin{tabular}{lcccc}
\toprule
Preferences & $\lambda^\star = 0$ & $0 < \lambda^\star < 1$ & $1 \leq \lambda^\star < \infty$ & $\lambda^\star = \infty$ \\
\midrule
Cobb-Douglas & 33.5\% & 47.0\% & 17.0\% & 2.5\% \\
Linear & 3.2\% & 7.3\% & 40.2\% & 49.3\% \\
Leontief & 46.2\% & 39.5\% & 12.4\% & 1.9\% \\
\bottomrule
\end{tabular}
\caption{Distribution of the cross-validation-selected $\lambda^\star$ across synthetic simulations, averaged over all sample sizes ($N = 50, \dots, 1000$) and trials, by preference structure.}
\label{tab:synthetic_lambda}
\end{table}
\end{document}